\documentclass[preprint, 3p,12pt]{elsarticle}

\usepackage{amssymb}
\usepackage{amsmath}

\usepackage{graphicx}
\usepackage{subfig}
\usepackage{hyperref}
\usepackage{cancel}
\hypersetup{
    colorlinks=true,
    linkcolor=blue,
    filecolor=magenta,      
    urlcolor=cyan,
    pdftitle={Overleaf Example},
    pdfpagemode=FullScreen,
    }
\usepackage{enumitem}
\usepackage[normalem]{ulem}

\makeatletter
\def\namedlabel#1#2{\begingroup
    #2%
    \def\@currentlabel{#2}%
    \phantomsection\label{#1}\endgroup
}
\makeatother

\usepackage{float} 
\usepackage[ruled]{algorithm2e}
\usepackage[noend]{algorithmic}
\usepackage{yhmath}

\def\cF{{\mathcal F}}

\def\sL{{\mathsf{L}}}

\def \st{\mathsf t}
\def \sT{\mathsf T}

\def \kt {\mathfrak t}

\def \cN {\mathcal N}

\def \sm {\mathsf m}

\def \I{\mathbf I}
\def \0{\mathbf 0}

\def \cF{\mathcal F}

\def \cU{\mathcal U}

\def \st{\mathsf{t}}

\def \bP{\mathbb P}

\def \Q{\mathrm Q}

\def\md{{\mathrm d}}

\def\cZ{\mathcal{Z}}

\def\cY{\mathcal{Y}}

\def\Sb{\mathbb S}
\def\b0{\mathbf 0}
\def \bsigma{\boldsymbol{\sigma}}
\def \Rb {\mathbb R}
\def \Eb {\mathbb E}
\def \bR {\mathbb R}
\def \D{\text D}
\def \U{\mathrm U}

\def \bsigma{\boldsymbol{\sigma}}
\def \cP{\mathcal P}

\def \cG{\mathcal G}

\def \Sb{\mathbb S}

\def\Sb{\mathbb{S}}
\def \B{{B}}

\def \Pb{\mathbb P}

\def \I {\mathrm I}

\usepackage{centernot}

\def \sT{\mathsf{T}}
\def \argmin{\mathrm{arg\,min}}

\def \dd{\mathsf{d}}
\def\sJ{{\mathsf{J}}}

\def\cF{{\mathcal F}}

\def\sL{{\mathsf{L}}}
\def \bE{\mathbb E}
\def \diag{\mathrm{diag}}
\def \sW{\mathsf{W}}
\def \sSW{\mathsf{SW}}

\def \F{\mathrm{F}}
\def \h{\mathrm{h}}
\def \b{\mathrm{b}}
\def \inf {\mathrm{inf.}}

\usepackage{xcolor}
\usepackage{bbm}

\definecolor{darkred}{rgb}{.7,0,0}
\definecolor{darkgreen}{rgb}{0,0.6,0}
\definecolor{darkblue}{rgb}{0,0,0.7}
\definecolor{darkmagenta}{rgb}{0.55,0,0.55}
\definecolor{darkteal}{rgb}{0,0.45,0.55}
\definecolor{darkorange}{rgb}{1,0.40,0}

\def\genbox#1#2#3#4#5#6{
    \leavevmode\raise#4bp\hbox to#5bp{\vrule height#5bp depth0bp width0bp
    \pdfliteral{q .5 w \csname #2COLOR\endcsname\space RG
                       \csname #3PDF\endcsname{#5}{#6} S Q
             \ifx1#1 q \csname #2COLOR\endcsname\space rg 
                       \csname #3PDF\endcsname{#5}{#6} f Q\fi}\hss}}

\def\diabox     #1#2{\genbox{#1}{#2}  {dia}      {-.5} {6}    {3}}

\usepackage[noend]{algorithmic}

\usepackage{amsthm}
\newtheorem{theorem}{Theorem}[section]
\newtheorem{lemma}[theorem]{Lemma}
\newtheorem{remark}[theorem]{Remark}
\newtheorem{assumption}[theorem]{Assumption}
\newtheorem{example}[theorem]{Example}

\makeatletter \def\ps@pprintTitle{\let\@oddhead\@empty \let\@evenhead\@empty \def\@oddfoot{} \let\@evenfoot\@oddfoot} \makeatother

\journal{arXiv}

\begin{document}

\begin{frontmatter}

\title{Efficient Deconvolution in Populational Inverse Problems}

\author[1]{Arnaud Vadeboncoeur \corref{cor1}\fnref{fn1}} 
\ead{av537@cam.ac.uk}
\author[1,2]{Mark Girolami \fnref{fn2}} 
\ead{mag92@cam.ac.uk}
\author[3]{Andrew M. Stuart \fnref{fn3}} 
\ead{astuart@caltech.edu}

\cortext[cor1]{Corresponding author}
\fntext[fn1]{AV was supported through the EPSRC ROSEHIPS grant EP/W005816/1.}
\fntext[fn2]{MG was supported by a Royal Academy of Engineering Research Chair and UKRI grants EP/X037770/1, EP/Y028805/1,
EP/W005816/1, EP/V056522/1, EP/V056441/1, EP/T000414/1, EP/R034710/1. }
\fntext[fn3]{AMS was supported by a Department of Defense Vannevar Bush
Faculty Fellowship and by the SciAI Center, funded by the Office of Naval Research (ONR), under grant N00014-23-1-2729.}

\affiliation[1]{organization={Department of Engineering, University of Cambrdige},
            addressline={7a JJ Thomson Ave}, 
            city={Cambridge},
            postcode={CB3 0FA}, 
            country={UK}}
\affiliation[2]{organization={The Alan Turing Institute},
            addressline={96 Euston Rd.}, 
            city={London},
            postcode={NW1 2DB}, 
            country={UK}}
\affiliation[3]{organization={Computing and Mathematical Sciences,  California Institute of Technology},
            addressline={1200 E California Blvd}, 
            city={Pasadena},
            postcode={CA 91125}, 
            state={CA},
            country={US}}

\begin{abstract}
    This work is focussed on the inversion task
    of inferring the distribution over parameters of interest leading to
    multiple sets of observations. The potential to solve such distributional inversion problems is driven by increasing availability of data, but a major roadblock 
    is blind deconvolution, arising when the observational noise distribution is unknown. 
    However, when data originates from collections of physical systems, \textit{a population}, it is possible to leverage this information to perform deconvolution. 
    To this end, we propose a methodology leveraging large data sets of observations, collected from different instantiations of the same physical processes, to simultaneously deconvolve the data corrupting noise distribution, and to identify the distribution over model parameters defining the physical processes.
    A parameter-dependent mathematical model of the physical process is employed. A loss function characterizing the match between the observed data and the output of the mathematical model is defined; it is minimized as a function of the both the parameter inputs to the model of the physics and the parameterized observational noise. 
    This coupled problem is addressed with a modified gradient descent algorithm that leverages specific structure in the noise model. Furthermore, a new active learning scheme is proposed, based on adaptive empirical measures, to train a surrogate model to be accurate in parameter regions of interest; this approach accelerates computation and enables automatic differentiation of black-box, potentially nondifferentiable, code computing parameter-to-solution maps. The proposed methodology is
    demonstrated on porous medium flow, damped elastodynamics, and simplified models of atmospheric dynamics. 
\end{abstract}



\end{frontmatter}




\section{Introduction}
From the conception of the first mathematical model intended to make predictions of nature, the question of how to set model parameters has persisted.
When observations originating from the model are available, we have an opportunity to calibrate our model. This task is often referred to as an \textit{inverse problem}~\cite{engl1996regularization,kaipio2006statistical,stuart2010inverse, aster2018parameter, groetsch1993inverse}.
A prominent difficulty in this arena is the need for precise understanding of the corruption occurring, when collecting data, in the form of measurement error. Furthermore, noise corruptions can exacerbate the ill-posed nature of many inverse problems.
Inverse problems are central to the interfacing of computational models with data from the physical systems they attempt to represent; solving these inverse problems facilitates more accurate predictions when systems are subject to new, unseen conditions.  In this age of data, massive observational data sets are routinely collected not just from single physical systems, but from vast numbers of individual realizations of similar systems. Querying information about these systems as a collective may be referred to as a \textit{distributional inverse problem} or a 
\textit{populational inverse problem}\footnote{To be specific, by ``population'' we mean here a collection of $N$ sets of data arising from distinct physical systems
that may be viewed as arising from different parameter settings in a common
family of systems; this is in contrast to the statistical learning communities typical use of the term ``population'' denoting the $N\rightarrow\infty$ limit of empirical measures.} \cite{gelman1995bayesian, kroese1995distributional}.
The objective of such inverse problems is to characterize the ensemble of unseen model parameters. The need to have knowledge about the nature of the data corruption persists and can be addressed concurrently, by leveraging large  data volumes: we may then seek to concurrently learn the distributions of both the model parameters and the observational noise. Efficient and reliable methods for performing this kind of inference can have important impact in, for example: estimating as-built properties of collections of manufactured assets and quality control~\cite{montgomery2020introduction}; monitoring of collections of deployed assets belonging to a population~\cite{bull2021foundations}; and model calibration for collections of measurements from physical systems~\cite{schneider2017earth}.

The work~\cite{akyildiz2024efficient} introduces a new approach to the emerging field of distributional inversion. However, in that work, the noise distribution is assumed to be fully known; furthermore only a single physical problem (porous medium flow) is addressed. In this paper we tackle
the problem of simultaneously learning the distribution of the noise; and we
illustrate the methodology on a variety of mechanics applications including
the Darcy model of porous medium flow~\cite{darcy1856fontaines}, linear elastodynamics and a general circulation model (GCM) where data is available in the form of 
statistics~\cite{cleary2021calibrate, schneider2017earth}. 

In solving these problems, we find ourselves evaluating solutions and derivatives of numerical models many times for small variations of model parameters. This is prohibitively expensive. Furthermore, practitioners already have numerical code for simulating their systems of interest -- hence, assuming one can use automatic differentiation tools is unrealistic. Finally, there are models, including GCMs for example, where the parameter-to-solution map for the model and observation process may be discontinuous by nature. To address these issues, we propose an efficient active learning surrogate modelling scheme that is coupled to the estimation procedure for the noise and parameter distributions.

\subsection{Contributions and Outline}\label{ssec:contrib_outline}
The main contribution presented in this paper is in developing methodology for concurrent deconvolution, distributional inversion, and surrogate modelling to enable efficient learning of statistical models when we have access to data from large collections of physical systems. The contributions can be broken down further 
as follows:
\begin{enumerate}
    \item[\namedlabel{con:C1}{(C1)}]\label{item:C1} We formulate the problem of deconvolving the noising probability measure and the distribution on the PDE parameters giving rise to the set of  observations, through a regularized probabilistic loss.
    \item[\namedlabel{con:C2}{(C2)}] We propose a modification of the standard 
    gradient-based optimization algorithm that identifies the same global minimizer 
    in the infinite data limit, but is more robust to the data empiricalization needed in practice.
    \item[\namedlabel{con:C3}{(C3)}] We develop a PDE surrogate model training scheme, compatible with black-box PDE simulation code, that concentrates learning efforts on the concurrently estimated underlying parameter distribution.
    \item[\namedlabel{con:C4}{(C4)}] The proposed methodology is tested on the Darcy model of porous medium flow for uncorrelated and correlated observational noise, and we demonstrate the robustness to empircalization of the proposed algorithm.
    \item[\namedlabel{con:C5}{(C5)}]
    We perform the deconvolution and distributional inversion of noise and PDE parameter distributions for a damped elastodynamic problem in three cases: (i) uncorrelated noise (sparse in space dense in time); (ii) correlated noise (sparse in space and time) learned as uncorrelated; (iii) correlated noise (dense in space and time). The inference scheme is carried-out through a concurrently learned surrogate model.
    \item[\namedlabel{con:C6}{(C6)}] The proposed methodology is adapted to chaotic dynamical systems to learn parameter distributions and noise covariances of simplified atomspheric models (L-96 single- and multi- scale) from time-averaged statistics; a surrogate model is used which extracts long-time-averages of dynamics.
\end{enumerate}
The rest of this paper is organized as follows: Section~\ref{ssec:litrev} 
overviews related works in the literature; Section~\ref{sec:background} highlights relevant mathematical background material; Section~\ref{sec:methodology} introduces the proposed methodology, addressing Contributions~\ref{con:C1}-\ref{con:C3}; Section~\ref{sec:darcy} tests the robustness to empiricalization of the proposed algorithm on a porous medium flow problem, corresponding to Contribution~\ref{con:C4}; Section~\ref{sec:elastodyn} demonstrates the scheme on an elastodynamic problem, Contribution~\ref{con:C5}; Section~\ref{sec:l96} adapts the proposed scheme for time-averaged data in the context of models of the atmosphere, which is Contribution~\ref{con:C6}; 
and Section~\ref{sec:conclusion} discusses future work.

\subsection{Related Works}\label{ssec:litrev}
The work we present in this paper is at the intersection of many fields in  computational modelling and statistical inference. 
We organize our literature review around the central topics of deconvolution, distributional inversion and surrogate modelling. 

\subsubsection{Deconvolution}
Central to the objective of this work is the task of statistical deconvolution -- disentangling the distribution of two additive sources~\cite{alexander2009deconvolution, meister2007deconvolving}. This problem has received much attention under various names, such as the blind source separation problem~\cite{moulines1997maximum, attias1998blind}, and independent component analysis~\cite{oja2000independent, hyvarinen1997fast, hyvarinen2002independent, lee1999independent}. Many works assume a
specific form for the noise distribution to be identified. But theoretical results in~\cite{gassiat2022deconvolution} demonstrate settings in which deconvolution is possible even when the noise distribution is not known; this work was continued in ~\cite{rousseau2024wasserstein}. In~\cite{halva2021disentangling} the authors relax some of the assumptions made in previous work, and propose a time-series inference scheme for deconvolution. The papers~\cite{capitao2023deconvolution, capitao2024deconvolution, 10.1214/08-AOS652, comte2010pointwise} highlight a variety of methodologies, tied to specific settings, in which deconvolution is undertaken.

\subsubsection{Distributional Inversion}
In this work, we combine the task of deconvolution {with distributional inverse problems.
In the context of this paper, we understand distributional inversion~\cite{akyildiz2024efficient, li2024stochastic, li2024differential, li2025inverse} to} be the task of recovering a distribution over parameters when observations are generated in terms of the pushforward of a non-Dirac probability distribution on the parameters -- this is in contrast to the classic Bayesian inverse problem where observations are assumed to be generated from a single parameter value~\cite{bernardo2009bayesian, gelman1995bayesian} and the recovered distribution reflects the stochasticity of the measurement process and prior uncertainty. Hence, this work is closely related to hierarchical~\cite{gelman1995bayesian} and empirical~\cite{robbins1964empirical, robbins1992empirical} Bayes approaches.
More recent takes on such inference tasks make use of Wasserstein distance between data and model distributions~\cite{bernton2019parameter, bassetti2006minimum, belili1999estimation}. It is worth mentioning that classical Bayesian views can be cast as pushforward problems~\cite{butler2018combining, butler2020optimal} and other sources of stochasticity may be present in the physical process~\cite{bingham2024inverse} giving rise to distributional inverse problems.

\subsubsection{Surrogate Modelling}
Surrogate modelling is an integral part of the proposed methodology due to the computational cost of repeated evaluations of physical simulations. Constructing efficient, less expensive, approximations of simulation code is a well-trodden field spanning different approach philosophies. These include reduced-order models~\cite{bui2008model, choi2021space} such as subspace techniques~\cite{bai2002krylov, freund2000krylov}, POD~\cite{sampaio2007remarks, fresca2022pod, chaturantabut2010nonlinear}, and reduced-basis methods~\cite{almroth1978automatic, noor1980reduced}. Early machine learning-inspired methodologies leveraged the power of Gaussian processes (GPs)~\cite{krige1951statistical, sacks1989statistical, kennedy2001bayesian, kennedy2000predicting} and
work in this vein continues -- see, for example~\cite{cleary2021calibrate}.
Recently neural network based surrogate modelling has emerged as a potential competitor  and the field of operator learning~\cite{bhattacharya2021model, lu2021learning, kovachki2023neural, li2020fourier} has emerged partially in response
to the specific use-case of surrogate modeling. Approximate models can be adaptively/actively/greedily constructed~\cite{haasdonk2008reduced, dutta2021greedy};
most relevant to our work are methods which adaptively construct surrogate models for inversion~\cite{cleary2021calibrate, yan2021stein, zhang2024bilo, akyildiz2024efficient, vadeboncoeur2023fully, li2023surrogate}.

\section{Background}\label{sec:background}
\subsection{Setup}\label{ssec:setup}

Let $F^\dagger:\cZ\mapsto\cU$ denote the solution operator for a PDE mapping parameters
in $\cZ$ to the solution space $\cU$. Furthermore let $g: \cU \to \bR^{d_y}$ denote
an observation process mapping the solution space of the PDE to a finite dimensional
Euclidean space. Consider, then, the data generating model~\cite{akyildiz2024efficient}, defined by
\begin{subequations}
\label{eq:data_model}\label{eq:data_gen}
\begin{align}
    y^{(n)} &= g\circ F^\dagger(z^{(n)}) + \xi^{(n)},\\
    \xi^{(n)}&\sim\eta^\dagger,\\
    z^{(n)}&\sim\mu^\dagger.
\end{align}
\end{subequations}
Here, for given $n \in \{1, \hdots, N\}$ $y^{(n)}\in\bR^{d_y}$ is an observation, $z^{(n)}\in \cZ$ is a PDE parameter and $\xi^{(n)}\in\Rb^{d_y}$ is a noise variable. The distribution $\mu^\dagger\in\cP(\cZ)$ denotes the unknown data-generating probability measure on the PDE parameters and $\eta^\dagger\in\cP(\bR^{d_y})$ is the unknown noise measure. In this work, we exclusively consider $\eta^\dagger:=\cN(0, \Gamma^\dagger)$.
We are interested in jointly learning $(\alpha^\star, \Gamma^\star)$ to ensure $\mu(\alpha^\star)\approx\mu^\dagger$ and $\eta(\Gamma^\star)\approx\eta^\dagger$.

\subsection{Distributional Generative Model}\label{ssec:emp_measure_pf}
We may summarize the data-generating set-up from the previous subsection using
the succint form
\begin{align}\label{eq:gen_model_data}
    y^{(n)}\sim \eta^\dagger * (g\circ F^\dagger)_\#\mu^\dagger.
\end{align}
We now explain the notation in~\eqref{eq:gen_model_data} in order to show that it indeed comprises a distributional description of~\eqref{eq:data_model}. 
Symbol ``$\#$'' denotes the pushforward operation on measures and ``$*$'' the convolution operation between pairs of measures. To define these two operations,
let the probability triplets $(\Omega_i, \cF_i, \mu_i)$, $i\in\{1,2, 3\}$  be made up of a sample spaces $\Omega_i$, sigma algebras $\cF_i$, and set valued probability functions $\mu_i$. We extend function $f:\Omega_1\mapsto\Omega_2$ to pushforward one probability measure $\mu_1$ into another $\mu_2$ by the notation $\mu_2 = (f)_\#\mu_1$;
this means that a random realization $x_2\sim\mu_2$ can be obtained by setting $x_2=f(x_1)$  where $x_1\sim\mu_1$. More formally the notion of pushforward is captured by the identity $(f)_\#\mu_1(A) = \mu_1(f^{-1}(A))$, for all $A \in \cF_2$ and for $f^{-1}(\cdot)$ computing all preimages of sets $A.$ The convolution of measures, ``$*$'',  can be defined intuitively by stating that $\mu_1 * \mu_2 =q_\#(\mu_1\otimes\mu_2)$ where $q(x_1, x_2)=x_1+x_2$ and ``$\mu_1\otimes\mu_2$'' is the product measure -- the joint distribution on pair $(x_1,x_2)$ defined by assuming independent distributions $\mu_i$ on each $x_i$. The formal definition is
that $(\mu_1 * \mu_2)(A) = \int \mathbbm{1}_A(x_1+x_2)\md\mu_1(x_1)\md\mu_2(x_2)$.

\subsection{Data Distribution} 

We view the collection of observations $\{y^{(n)}\}_{n=1}^N$ as drawn i.i.d.\thinspace from distribution $\nu$. From this data we can construct an \textit{empirical approximation} of $\nu$, denoted by $\nu^N$, and referred to as an
\textit{empirical measure}: a mixture of point masses given by
\begin{align}\label{eq:data_emp_dist}
    \nu^N(A) = \frac{1}{N}\sum_{n=1}^N \delta_{y^{(n)}}(A).
\end{align}
Here the Dirac measure\footnote{Often succintly written  as $\delta_y(A) = \mathbbm{1}_A(y)$} is
\begin{align}
    \delta_y(A)=\begin{cases}
    1,\quad y\in A,\\
    0,\quad y\notin A.
\end{cases}
\end{align}
In what follows we use the empirical measure to
characterize the data set and to learn unknown parameters in the generative model~\eqref{eq:gen_model_data} by asking that the model and the data align as
distributions.

\subsection{Statistical Divergences}\label{ssec:stat_div}
Our learning methodology is based on seeking to align two empirical measures, one defined by the data and the other by the generative model determined by unknown parameters. To measure alignment we will use a divergence\footnote{A metric is a divergence, but a divergence is a weaker concept, foregoing symmetry and the triangle inequality.} between probability measures. Because we work with empirical probability measures\footnote{More formally, we do not have absolute continuity between compared empirical probability measures; $\nu$ being absolutely continuous w.r.t $\mu$,
expressed as $\nu\ll\mu$, is defined by the condition that, for every measurable set $A$ satisfying $\mu(A)=0$, it follows that $\nu(A)=0$. Intuitively, if $z\sim\nu$ and $\nu\ll\mu$, then it is also possible to view $z$ as having been drawn from $\mu$ as well as from $\nu.$} we cannot use the Kullback–Leibler (KL) divergence -- it would either be zero or infinity when comparing two empirical measures.

The Wasserstein class of metrics is suitable for the task of comparing
empirical measures. Indeed they compare measures $\mu, \nu\in\cP(\Rb^d)$ in terms of displacements of probability mass following an optimal transport plan and do not have absolute continuity requirements~\cite{villani2008optimal, santambrogio2015optimal}. 
Let $\cY$ denote a vector space equipped with a norm $\|\cdot\|_B$ and let $\mu, \nu$ denote arbitrary pairs of measures on $\cY.$ Then the squared weighted $2$-Wasserstein metric, defined with respect to this given norm, is
\begin{align}\label{eq:weighted_w}
    \sW_{2,\B}^2(\nu, \mu) = \underset{\pi\in\Pi(\nu, \mu)}{\mathrm{inf.}}{\int_{\cY\times\cY} \|x-y\|^2_\B\md\pi(x,y)};
\end{align}
here the infimum runs over the couplings $\Pi$ of $\nu$ and $\mu$, that is to say, the set of measures $\pi\in\cP(\cY\times\cY)$ with marginals given by $\nu$ and $\mu$.
When the norm $\|\cdot\|$ on $\cY$ is the standard Euclidean norm we simply write
$\sW_{2}^2(\nu, \mu)$ for the $2$-Wasserstein metric.

The Wasserstein metric is readily computed in one dimension, for empirical measures,
using their cumulative distribution functions. When measures $\nu,\mu$ are empirical, evaluation of Wasserstein metrics in dimension greater than 1 requires solution of a linear programming problem to determine the optimal coupling. To avoid this, the sliced-Wasserstein metric~\cite{bonneel2015sliced} may be used. When implemented through
Monte Carlo approximation, this method involves computation of a set of 
1D Wasserstein distances found by pushing forward $\nu,\mu$ under random
maps into $\bR$, maps defining hyperplanes; the closed form 1D Wasserstein distances can then be used. The squared weighted 2-sliced-Wasserstein metric is expressed as
\begin{align}\label{eq:weighted_sw}
    \sSW^2_{2,\B}(\nu, \mu) = \int_{\mathbb{S}^{d-1}}\sW^2_2(P^\theta_{\B\,\#}\nu, P^\theta_{\B\,\#}\mu)\,\md\theta,
\end{align}
where $P^\theta_\B(\,\cdot\,) = \langle \B^{-1/2}\,\cdot\,,\theta\rangle$ and $\Sb^{d-1}$ is the unit sphere in $\bR^{d}$. Approximating the integration over the unit sphere via random sampling of $\theta$ leads to the desired 1D Wasserstein distances and explicit formulae for the (Monte Carlo approximated) sliced-Wasserstein distance. Thus, we now have a suitable statistical divergence for empiricalized measures, which can be efficiently computed. Other classes of divergence would also be suitable for such tasks, including Maximum Mean Discrepancy (MMD)~\cite{gretton2012kernel} and Energy Distances~\cite{szekely2013energy, szekely2003statistics}; however we do not implement these in this work.

\subsection{Distributional Inverse Problem}
In~\cite{akyildiz2024efficient} the authors propose to learn an approximation, $\mu(\alpha)$, of the distribution over model parameters $\mu^\dagger$ (in \eqref{eq:gen_model_data}) from the distribution $\nu^N$ over indirect observations $\{y^{(n)}\}_{n=1}^N$ (in~\eqref{eq:data_emp_dist}) by minimizing
\begin{align}\label{eq:epcid_loss}
    \sJ(\alpha) = \frac{d_y}{2}\sSW^2_{2, \Gamma}(\nu^N, \eta^\dagger * (g\circ F^\dagger)_\#\mu(\alpha)) + h(\alpha);
\end{align}
The authors in~\cite{akyildiz2024efficient} also propose to concurrently learn a surrogate model $F^{\phi}$. Parameter
$\phi$ is learned in a bilevel optimization scheme using PDE residuals~\cite{vadeboncoeur2023fully, rixner2021probabilistic, zhu2019physics, zhang2024bilo}. The inner optimization over $\phi$ results in a choice of $\phi^\star(\cdot)$ that is adapted to the value of $\alpha$ in the outer optimization: $\phi^\star(\alpha).$
The work in this paper makes the significant step of
building on~\cite{akyildiz2024efficient} to allow simultaneous noise deconvolution.
Furthermore we demonstrate the utility of the resulting methodology on a range
of applications from mechanics.

\subsection{Notation}
We denote by
$\langle \cdot, \cdot \rangle_A=\langle \cdot, A^{-1}\cdot \rangle$ the covariance weighted inner product,  
for any positive self-adjoint operator $A$, with induced norm $\|\cdot\|_{A}$.  
In particular, our definition of sliced-Wasserstein distance 
\eqref{eq:weighted_sw} uses this convention
for some matrix, or operator, $\B$. We refer to the PDE domain, a bounded and open set, as $D$; and we denote its boundary by $\partial D$. 
The set $\mathcal{C}^\infty(D)$ is the class of infinitely differentiable functions on $D$. The space $H^1(D)$ is the Sobolev space of once weakly integrable functions, and $H^1_0(D)$ is the subset of $H^1(D)$ in which functions take value zero
on $\partial D$, in a trace sense; this solution space is appropriate for elliptic PDEs
subject to Dirichlet boundary conditions.  
We use the $L^p(D)$ classes of Lebesgue $p$-integrable functions, $1 \le p < \infty$, with $p=\infty$ denoting essentially bounded functions on $D$.
The set of probability measures over measurable space $X$ is denoted as $\cP(X)$. 
A divergence on the space of probability measures is a mapping $\dd:\cP(X)\times\cP(X)\mapsto\Rb_+$ with property that 
$\dd(\mu, \nu)=0\;\mathrm{iff.}\;\mu=\nu$. A metric adds symmetry and the
triangle inequality to the definition.
When $z\in L^\infty(D)$ we construct $\cP(Z)$ on a separable subspace $Z\subset L^\infty(D)$. 
The functions $\exp$ and $\log$ should be understood to be applied element-wise/point-wise when applied to vectors/functions.

\section{Methodology}~\label{sec:methodology}
We now present the main methodological contributions of this work (Contributions~\ref{con:C1}-\ref{con:C3}). We first address \ref{con:C1} and present the proposed loss function for deconvolving the corrupting noise distribution and identifying the parameter distribution underlying the observations. Second, we address \ref{con:C2} by introducing two optimization algorithms to minimize the loss function of noise and input parameter distribution,
one a standard-gradient descent and the other we term {\it{cut-gradient}} descent. 
We show that both have the same global minimzizer in the $N\rightarrow\infty$ setting.  We identify a tractable setting in which the cut-gradient
approach is more robust to empiricalization ($N<\infty$) than is the standard-gradient
descent; later experiments will verify this robustness beyond the confines of this specific example. Third, we address \ref{con:C3} and explain the proposed scheme for concurrently learning a surrogate model for the PDE parameter-to-solution map -- a component which is integral to the efficient application of the deconvolution scheme in practice.

\subsection{Loss Function}

We now present the central objective function for solving the deconvolution and distributional inverse problem (Contribution~\ref{con:C1}). To this end let
$\dd_{\B}$ denote either the squared $B$-weighted 2-Wasserstein~\eqref{eq:weighted_w} or the squared $B$-weighted 2-sliced-Wasserstein metric~\eqref{eq:weighted_sw} 
metric. Then define
\begin{subequations}
\label{eq:O1}
\begin{align}
    \sL_0(\alpha, \Gamma; \Gamma') :=& \dd_{\Gamma'}\Bigl(\nu, \eta(\Gamma) * (g \circ F^\dagger)_\#\mu(\alpha)\Bigr), \tag{{O1a}} \label{eq:O1a}\\
    \sL(\alpha, \Gamma; \Gamma') :=&  \frac{d_y}{2}\sL_0(\alpha, \Gamma; \Gamma')+ h(\alpha) + r(\Gamma), \tag{{O1b}} \label{eq:O1b}
\end{align}
\end{subequations}
where $h(\alpha), r(\Gamma)$ are regularizers.
We have introduced $\Gamma'$ to highlight the fact that the properties
of the unknown distributional noise appear indirectly through the distance
used to define the Wasserstein distance, as well as directly through noise
convolution. To deconvolve noise and parameter distribution, we wish to jointly minimize one of the following loss functions
\begin{subequations}
\label{eq:O2}
\begin{align}
    \sJ_0(\alpha, \Gamma) :=& \sL_0(\alpha, \Gamma;\Gamma), \tag{{O2a}}\label{eq:O2a}\\
    \sJ(\alpha, \Gamma) :=& \sL(\alpha, \Gamma;\Gamma).  \tag{{O2b}}\label{eq:O2b}
\end{align}
\end{subequations}

One can perform minimization of \eqref{eq:O2a} by considering
time-discretization of either of the following two flows:
\begin{align}
    \tag{A1} \frac{\md (\alpha, \Gamma) }{\md t} &= -\nabla_{(\alpha, \Gamma)}\sL_0(\alpha; \Gamma;\Gamma')\Big|_{\Gamma'=\Gamma}\label{eq:A1},\\
    \tag{A2} \frac{\md (\alpha, \Gamma) }{\md t} &= -\nabla_{(\alpha, \Gamma)}\sJ_0(\alpha, \Gamma)\label{eq:A2}.
\end{align}
We tacitly assume in what follows that the vector fields defining these
flows are well-defined, locally Lipschitz functions of $(\alpha,\Gamma)$ so that
the resulting ODEs have, at least locally in time, unique solutions.
We will refer to the right hand side of~\eqref{eq:A1} as the \emph{cut-gradient}  form and \eqref{eq:A2} as the standard \emph{standard-gradient} form.
In what follows we describe the differences in \eqref{eq:A1} and \eqref{eq:A2} in two settings: (S) infinite sample loss setting with $\nu$ ($\nu^N$ as $N\rightarrow\infty$);
and (S$^\prime$) finite sample setting $\nu^N$ with $1\leq N<\infty$  (Contribution~\ref{con:C2}).\footnote{When we implement optimization in practice the gradient flow of setting (S$^\prime$) is not directly implemented; however its mathematical treatment gives key insight into the differences between the cut-gradient and standard-gradient forms that are insightful for the gradient based algorithms that we use in practice.}

We now state a central assumption used in the subsequent results.
\begin{assumption}\label{assum:1}
The data distribution admits representation
\begin{align}\label{eq:H1}
    \nu = \eta^\dagger * (g\circ F^\dagger)_\#\mu(\alpha^\dagger),
\end{align}
where $\eta^\dagger:=\cN(0, \Gamma^\dagger)$; that is, we assume there exist parameters $\alpha^\dagger$, $\Gamma^\dagger$ for which the generative model reproduces the data distribution in the infinite data limit of Setting (S). Furthermore we assume $\Gamma^\dagger \ne 0.$
\end{assumption}

\begin{algorithm}[t]
\caption{Deconvolution and Distributional Inversion with~\eqref{eq:A1}.}
\label{alg:algorithm1}
\begin{algorithmic}[1]
\STATE{Initialize $\alpha_0$, $\Gamma_0$, $\sT$ (\text{number of iterations}), $N_s$ (\text{number of samples for $\sL$}), $F^\dagger$ (\text{forward model}).}
\FOR{$\st=1, \hdots, \sT$}
\FOR{$i=1, \hdots, N_s$}
\STATE{Sample $z^{(i)}\sim\mu({\alpha_{\st-1}})$.}
\STATE{Sample $\varepsilon_\eta^{(i)}\sim\eta(\Gamma_{\st-1})$.}
\STATE{Compute ${y'}^{(i)} = g \circ F^\dagger(z^{(i)}) + \varepsilon_\eta^{(i)}$.}
\STATE{Sample $y^{(i)}\sim\nu^N$.}
\ENDFOR
\STATE{Set $\Gamma_{\st-1}'= \text{CUT-GRADIENT}(\Gamma_{\st-1})$.}
\STATE{Compute $\sL(\alpha_{\st-1}, \Gamma_{\st-1};\Gamma_{\st-1}')$ from the $N_s$ samples $\{y^{(i)}, {y'}^{(i)}\}_{i=1}^{N_s}$.}
\STATE{$(\alpha_\st, \Gamma_\st) = \text{OPTIMISER}( (\alpha_\st, \Gamma_\st), \sL(\alpha_{\st-1}, \Gamma_{\st-1};\Gamma_{\st-1}'))$.}
\ENDFOR 
\RETURN $(\alpha^\star, \Gamma^\star) \leftarrow (\alpha_{\sT}, \Gamma_{\sT})$.
\end{algorithmic}
\end{algorithm}

The value of the following theorem, for us, is that it indicates that gradient
descent algorithms based on the cut-gradient approach will, in the large data limit of setting (S), have limit points which capture the true parameters. Furthermore we will then show that the cut-gradient approach is more robust to empiricalization -- setting (S').

\begin{theorem}
\label{t:cgd}
In setting (S) and under Assumption~\ref{assum:1}, the vector field defining both (A1) and (A2) is zero at the global minimizer $(\alpha^\dagger, \Gamma^\dagger)$ of  $\sJ_0(\alpha, \Gamma).$ 
\end{theorem}
\begin{proof}
First note that, under Assumption~\ref{assum:1}, the global minimizer is zero and is achieved when the parameters $(\alpha,\Gamma)$ equal $(\alpha^\dagger,\Gamma^\dagger)$.
Noting the forms of $\dd_{\Gamma'}$, as either~\eqref{eq:weighted_w} or \eqref{eq:weighted_sw}, and $\cN(0, \Gamma)$, we observe $\nabla_\Gamma\sL_0(\alpha, \Gamma;\Gamma')$ is well-defined as tacitly assumed above. Thus, using the differentiability assumption on $\sL_0$ with respect to $\alpha$, we compute the gradients defining (A1) and (A2) respectively:
\begin{align}
\label{eq:rhs1}
    -\nabla_{(\alpha, \Gamma)}\sJ_0(\alpha, \Gamma) = 
    \left(\begin{array}{ll}
        \nabla_\alpha \sL_0(\alpha, \Gamma; \Gamma)\\
        \nabla_\Gamma \sL_0(\alpha, \Gamma; \Gamma')\Big|_{\Gamma'=\Gamma} + \nabla_{\Gamma'} \sL_0(\alpha, \Gamma; \Gamma')\Big|_{\Gamma'=\Gamma}
    \end{array}\right).
\end{align}
Secondly we note
\begin{align}
\label{eq:rhs2}
    -\nabla_{(\alpha, \Gamma)}\sL_0(\alpha, \Gamma; \Gamma')\Big|_{\Gamma'=\Gamma} = 
    \left(\begin{array}{ll}
            \nabla_\alpha \sL_0(\alpha, \Gamma; \Gamma)\\
            \nabla_\Gamma \sL_0(\alpha, \Gamma; \Gamma')\Big|_{\Gamma'=\Gamma}
    \end{array}\right).
\end{align}
Next we remark that in Setting (S), $\nabla_{\Gamma'} \sL_0(\alpha^\dagger, \Gamma^\dagger; \Gamma')\Big|_{\Gamma'=\Gamma^\star} = 0$. This follows because, in fact, for
all $\Gamma',$
$$\sL_0(\alpha^\dagger, \Gamma^\dagger; \Gamma')=0.$$ 
This implies that, for all $\Gamma',$
$$\nabla_{\Gamma'} \sL_0(\alpha^\dagger, \Gamma^\dagger; \Gamma') = 0.$$ 
This in turn implies that the right-hand sides of \eqref{eq:rhs1} and \eqref{eq:rhs2}
coincide at $(\alpha,\gamma)=(\alpha^\dagger,\gamma^\dagger)$ and completes the proof.
\end{proof}

\begin{example}
\label{ex:ill}
To exemplify and clarify the relevance of Theorem~\ref{t:cgd}, consider noise covariances which are scaled identities, $\Gamma = \gamma^2\I$, and $\Gamma'={\gamma'}^2\I$.  In this simplified setup, we will explicitly confirm that \eqref{eq:A1}-\eqref{eq:A2} have the same stationary points in the infinite sample Setting (S), but show that, for $\nu^N\neq\nu$, these differ in the finite sample Setting (S'). 
We focus on the case where $\dd_{B}$ is the weighted squared 2-Wasserstein distance. For conciseness we define $\Pi_{\alpha, \gamma} := \Pi(\nu, \eta(\gamma^2\I) * (g\circ F^\dagger)_\# \mu(\alpha))$, then
\begin{align*}
    \sL_0(\alpha, \gamma^2\I;{\gamma'}^2\I) &= \sW^2_{2, {\gamma'}^2\I}(\nu, \eta(\gamma^2\I) * (g\circ F^\dagger)_\# \mu(\alpha)) \\
     &= \underset{\pi\in\Pi_{\alpha, \gamma}}{\inf}\int_{\Rb^d\times\Rb^d} \|x-y\|^2_{{\gamma'}^2\I}\, \md\pi(x,y)\\
    &= \underset{\pi\in\Pi_{\alpha, \gamma}}{\inf}\int_{\Rb^d\times\Rb^d} \frac{1}{{\gamma'^2}}\|x-y\|^2\, \md\pi(x,y)\\
    &= \frac{1}{{\gamma'}^2}\sW^2_{2}(\nu, \eta(\gamma^2\I) * (g\circ F^\dagger)_\# \mu(\alpha)) \\
    &= \frac{1}{{\gamma'}^2}\sL_0(\alpha, \gamma^2\I;\I).
\end{align*}
We note that 
\begin{equation}
    \label{eq:chk}
\sL_0\bigl(\alpha^\dagger, (\gamma^\dagger)^2\I;\I\bigr)=0.
\end{equation}
With this, let us look at the stationary points of flows ~\eqref{eq:A1} and \eqref{eq:A2} in the infinite sample Setting (S). 
From \eqref{eq:rhs1} and \eqref{eq:rhs2} we see that it suffices to study the
vector field with respect to $\gamma$ and not $\alpha.$ Under assumption~\ref{assum:1}, using~\eqref{eq:O2a}, and Theorem~\ref{t:cgd}, both flows have zero
vector in $\gamma$ co-ordinate at $(\alpha, \gamma)=(\alpha^\dagger, \gamma^\dagger)$,
by \eqref{eq:chk}:
\begin{align*}
    \quad\frac{\md\gamma}{\md t}&=-\nabla_{\gamma} \sL_0(\alpha, \gamma^{2}\I; {\gamma'}^2\I) = - \frac{1}{{\gamma'}^2}\nabla_\gamma\sL_0(\alpha, \gamma^2\I;\I) = 0,\tag{{{$\gamma$-flow A1}}}\label{eq:gflow1}\\
    \frac{\md\gamma}{\md t}&=-\nabla_{\gamma} \sJ_0(\alpha, \gamma^{2}\I) =   -\frac{1}{{\gamma}^2}\nabla_\gamma\sL_0(\alpha, \gamma^2\I;\I) - \frac{1}{{\gamma}^3}\sL_0(\alpha, \gamma^2\I;\I) = 0,\tag{{{$\gamma$-flow A2}}}\label{eq:gflow2}
\end{align*}
    \diabox0{cblack}
\end{example}

In the preceding example we explicitly use setting (S); the same result cannot be expected in setting (S'). Indeed in
Section~\ref{sec:darcy} we will show empirical evidence indicating that, when implemented with a time-stepping algorithm, and in the finite data setting (S')
where empiricalization  plays a role, the cut-gradient approach is more robust than standard-gradient descent. Thereafter, in Sections~\ref{sec:elastodyn} and~\ref{sec:l96}, we will exclusively use algorithms based
on time-discretization of the cut-gradient methodology.

The flows described in~\eqref{eq:A1}-\eqref{eq:A2} can be discretized using standard explicit time-stepping schemes through discrete pseudo-time $\st=1 \hdots \sT$, using adaptive steps given by an optimizer and using $N_s$ samples to evaluate the distribution matching loss function. This procedure is described in Algorithm~\ref{alg:algorithm1}. The algorithm makes use of the CUT-GRADIENT\footnote{Also commonly refered to as STOP-GRADIENT.} operation, common in automatic differentiation software, to update $\Gamma'$ with $\Gamma$ without propagating gradient computation through $\Gamma'$. In the rest of this work, we will use $\dd_\B$ to be $\sSW^2_{2, \B}$ as defined by~\eqref{eq:weighted_sw}. 
Before moving on to numerical examples, we introduce Contribution~\ref{con:C3}, an active-learning surrogate modelling methodology compatible with the learning of~\eqref{eq:O1a}-\eqref{eq:O1b}.

\subsection{Surrogate Training on Cumulative Empirical Measures}

\begin{algorithm}[t]
\caption{Concurrent Surrogate Learning}
\label{alg:algorithm2}
\begin{algorithmic}[1]
\STATE{Initialize: Algorithm~\ref{alg:algorithm1}, $\sT_a$ ({number of iterations where training pairs are aquired for $\bP^\st_{z,u}$}),  $\sT_{\mathrm{pre.}}$ number of gradient update pre-training steps,  $\sT_{\mathrm{inner}}$ number of gradient update steps on $\phi$ per step on $(\alpha, \Gamma)$, $N_\mathrm{pre.}$ (number of pre-train samples), $N_F$ (the batch size for the surrogate loss).}
\STATE{Sample $z^{(1-N_\mathrm{pre.}:1)}\sim\mu({\alpha_{0}})$.}
\STATE{Compute $u^{(1-N_\mathrm{pre.}:1)}=F^\dagger(z^{(-N_{\mathrm{init.}:1})})$.}
\STATE{Construct $\bP^0_{z,u}$.}
\STATE{Compute $\phi^\star_0$ with $\sT_{\mathrm{pre.}}$ steps and $N_F$ batch size (Pre-Training).}
\FOR{$\st=1 \hdots \sT$ }
\STATE{Increment $(\alpha_{\st-1}, \Gamma_{\st-1})$ following Algorithm~\ref{alg:algorithm1} lines  3-10 replacing $F^\dagger$ with $F^{\phi^\star_{\st-1}}$.}
\IF{$\st < \sT_a$}
\STATE{Sample $z^{(\st)}\sim\mu({\alpha_{\st}})$.}
\STATE{Compute $u^{(\st)}=F^\dagger(z^{(\st)})$.}
\STATE{Update $\bP^\st_{z,u}$.}
\ENDIF
\STATE{Compute $\phi^\star_\st$ initializing $\argmin_\phi$ at $\phi^\star_{\st-1}$ with  $\sT_{\mathrm{inner}}$ steps and $N_F$ batch size.}
\ENDFOR 
\RETURN{$(\alpha^\star, \Gamma^\star) \leftarrow (\alpha_{\sT}, \Gamma_{\sT})$ and $\phi^\star\leftarrow\phi^\star_{\sT}$.}
\end{algorithmic}
\end{algorithm}
%

Numerical PDE solution operators ($F^\dagger$) are challenging to incorporate in inference schemes, especially when one needs to differentiate model solutions with respect to model parameters. These operations are often computationally expensive, and in many cases models may not be differentiable. Additionally, the maps themselves may be discontinuous (as we will see in an example later).

To accelerate computation and enable the seamless use of automatic differentiation we can replace $F^\dagger$ in \eqref{eq:O1a}-\eqref{eq:O1b} with $F^\phi$, a surrogate model with parameters $\phi$  (Contribution~\ref{con:C3}). The question is then how to learn the parameters $\phi$ of this surrogate model. In this work, we want to concurrently learn $\phi$ and $\alpha$ to have high surrogate accuracy under the parameter measure $\mu(\alpha)$, implying a coupled learning scheme is necessary. Thus, given a warm-up sample number $N_\mathrm{pre.}$, we propose to perform~\eqref{eq:A1} and concurrently compute
\begin{subequations}\label{eq:learn_phi}
\begin{align}
     \phi^\star_\st &= \underset{\phi}{\argmin}\;\bE_{z, u\sim\bP^{\st}_{z,u}}\left[\|F^\phi(z) - u\|^2_2\right],\\
     \bP^\st_{z,u} &= \frac{1}{\st + N_\mathrm{pre.} }\sum_{\sm=1-N_\mathrm{pre.}}^{\st} \delta_{(z^{(\sm)}, F^\dagger(z^{(\sm)})},\quad \begin{cases}
         z^{(\sm)}\sim\mu(\alpha_0),\quad \mathrm{for}\; \sm\leq0,\\
         z^{(\sm)}\sim\mu(\alpha_\sm),\quad \mathrm{for}\; 0<\sm\leq \st,\end{cases}\label{eq:cummul_emp_dataset}
\end{align}
\end{subequations}
remembering $\st=1\hdots\sT$ indexes discrete gradient flow pseudo-time.
With this formulation we can replace $F^\dagger$ with $F^{\phi^\star_\st}$ in~\eqref{eq:A1} as we learn $(\alpha, \Gamma)$.

\begin{remark}
    In using~\eqref{eq:learn_phi}, we concentrate the training data-pairs acquisition, and learning efforts, of the surrogate model on regions of the parameter space which are most relevant for the inference task at the current step of the algorithm where we have estimate $\mu(\alpha_\st)$.
    \diabox0{cblack}
\end{remark}

Algorithm~\ref{alg:algorithm2} shows how to implement the proposed scheme and how it relates to learning $(\alpha, \Gamma)$ with \eqref{eq:A1}. We note that one can opt to withhold lines 8-11 until a sufficient batch size is reached in order and take advantage of the speed-up offered by parallelization -- the benefits of this are hardware and problem dependent.

\section{Porous Medium Flow}\label{sec:darcy}
This section addresses Contribution~\ref{con:C4}.
To test the proposed schemes in a simple setting we formulate the deconvolution and distributional inverse problem for the Darcy model of steady porous medium flow, given as
\begin{subequations}\label{eq:darcy}
\begin{align}
    -\nabla\cdot(z\nabla u) &= f, \quad x\in D\\
    u &= 0, \quad x\in\partial D.
\end{align}
\end{subequations}
We assume that the permeability field is constant throughout $D$,
and defined by scalar $z\in\Rb_+$. The spatially varying field $u\in H^1_0(D)$ is the piezometric pressure head, and $f\in L^2(D)$ is a source term. Let $F^\dagger$ denote the map from $z \in \Rb_+$ to $u \in H_0^1(D)$, and $g:H_0^1(D)\mapsto\bR^{d_y}$ be a set of linear functionals collecting observations of $u$. In what follows we choose $D=(0,1)$ and then $F^\dagger$ is available to us in closed form; we have no need  of recourse to a numerical solver or surrogate model. In the numerics shown in this section $f(x)=10$ and we employ $d_y = 50$ pointwise
evaluations of the solution in $D$ to define the linear observation functionals.

In the following subsections we focus on contrasting the numerical behaviors of the cut-gradient~\eqref{eq:A1} and standard-gradient~\eqref{eq:A2} forms.
\begin{itemize}
    \item In Section~\ref{sssec_4_1_1} we compare the losses~\eqref{eq:O1a},~\eqref{eq:O2a}, and in particular the location of their minimizers in relation to the true parameter values, 
    under varying amounts of data; we
    demonstrate the greater accuracy achieved by~\eqref{eq:O1a}, 
    particularly in small data regimes. 
    \item In Section~\ref{ssec:4_1_2} we compare convergence of algorithms based on discretization of the gradient-flows~\eqref{eq:A1},~\eqref{eq:A2}; we show  improved robustness to empiricalization in~\eqref{eq:A1} for estimation of the noise distribution variance.
    \item Section~\ref{ssec:4_2_1} generalizes the conclusions of~\ref{ssec:4_1_2} by considering non-diagonal covariances and estimating noise lengthscale along with parameter $\alpha$ defining the distribution of input $z.$
    \item Section~\ref{ssec:4_2_2} concludes the Darcy experiments by demonstrating the efficacy of the cut-gradient flow~\eqref{eq:A1} under challenging inference conditions: we concurrently estimate not only the noise lengthscale, as in the preceding subsection, but also the noise variance.
\end{itemize}

We omit the use of the regularizers $h(\alpha), r(\Gamma)$ as in~\eqref{eq:O1a} for the following porous medium flow examples.  To ease readability, reproducibility, and minimize information duplication, in this (and subsequent) numerics sections the experimental setup information is developed with increasing specificity as we descend into sub-sections; information pertaining only to individual experiments are listed after the bolded paragraphs labelled as \textbf{Data Specifics}, \textbf{Learning Specifics}, \textbf{Solver Specifics} (when applicable), and \textbf{Results}.

\subsection{Uncorrelated Noise}\label{ssec:4_1} 
\begin{figure}[t]
    \centering
    {\includegraphics[width=0.6\linewidth]{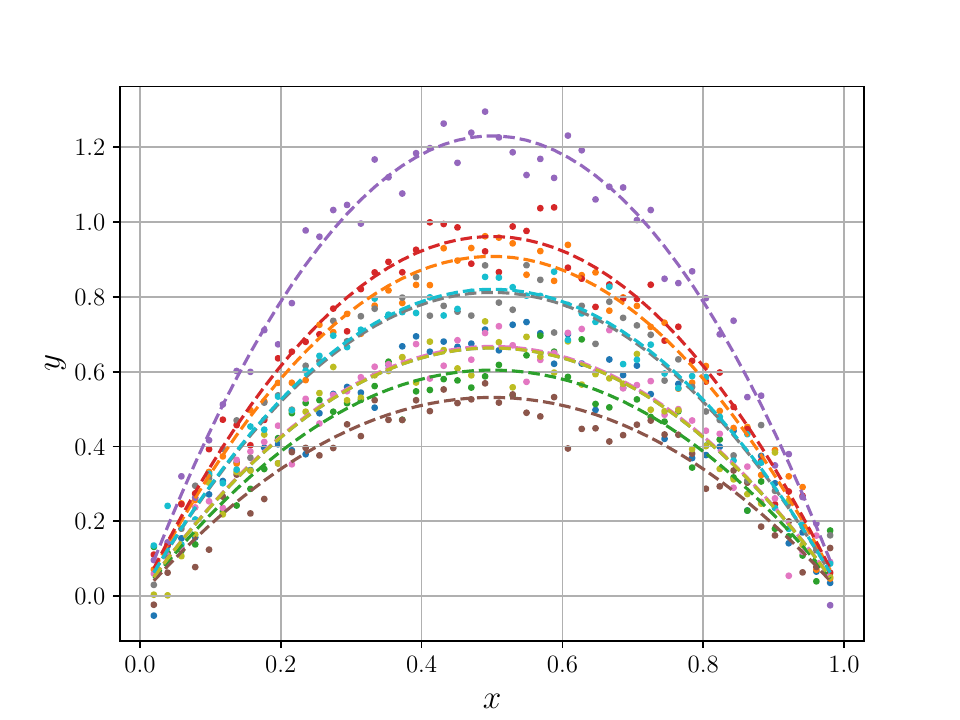}}
    \caption{Five solutions $u$ to the porous medium flow problem (dashed lines) and corresponding observation vectors $y$ with 50 observation locations (scatter points).}\label{fig:Darcy1D_data}
    \label{fig:five_sln_1D_Darcy}
\end{figure}
We consider the case where data is corrupted with uncorrelated Gaussian noise with a scaled identity in the form $\Gamma=\gamma^2\I$. We will first illustrate the effect of removing the $\gamma'$ dependence in the loss function in terms of minimizers of the loss for varying data set sizes, thereby affecting the empiricalizaion error. After that study we then vigorously assess the convergence properties of the two gradient descent algorithms: cut-gradients~\eqref{eq:A1}, and standard-gradients~\eqref{eq:A2}, using loss function~\eqref{eq:O1a} and \eqref{eq:O2a} respectively.  Figure~\ref{fig:five_sln_1D_Darcy} shows five example solutions to~\eqref{eq:darcy} and corresponding sparse observations.
\subsubsection{Loss Function Comparison}\label{sssec_4_1_1}
\begin{figure}[t]
    \centering
    \subfloat[$N=50$]{\includegraphics[width=0.45\linewidth]{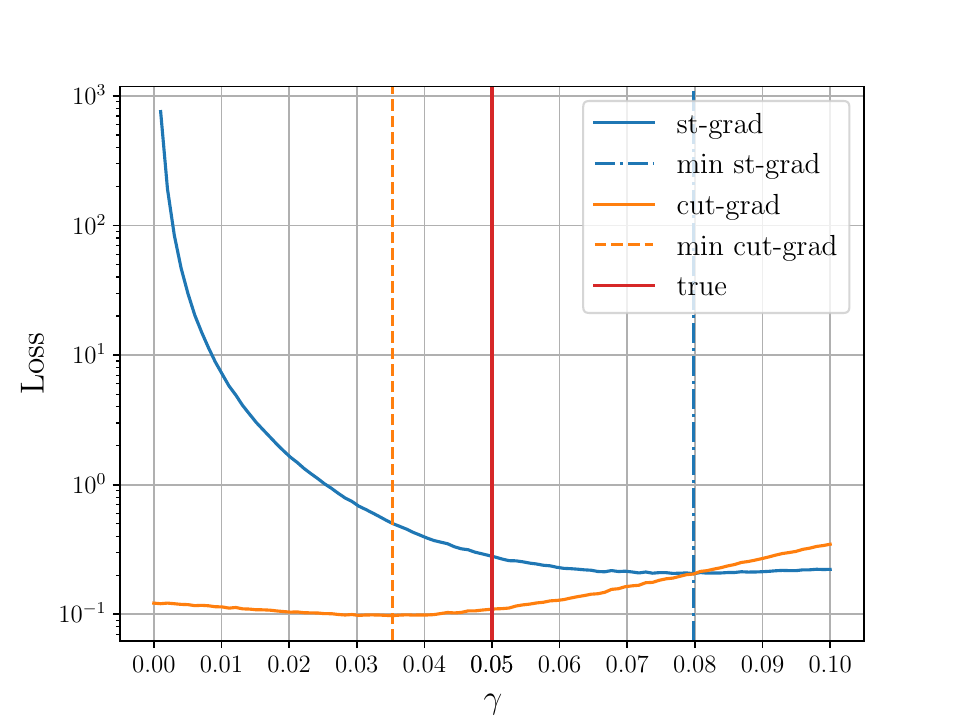}}
    \subfloat[$N=100$]{\includegraphics[width=0.45\linewidth]{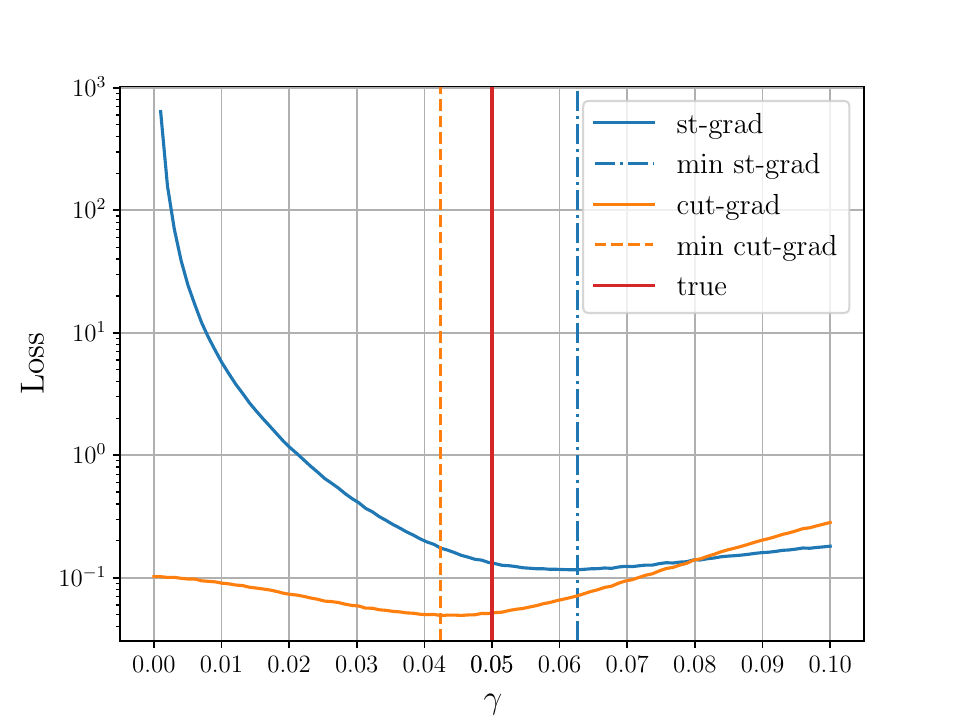}}\\   
    \subfloat[$N=500$]{\includegraphics[width=0.45\linewidth]{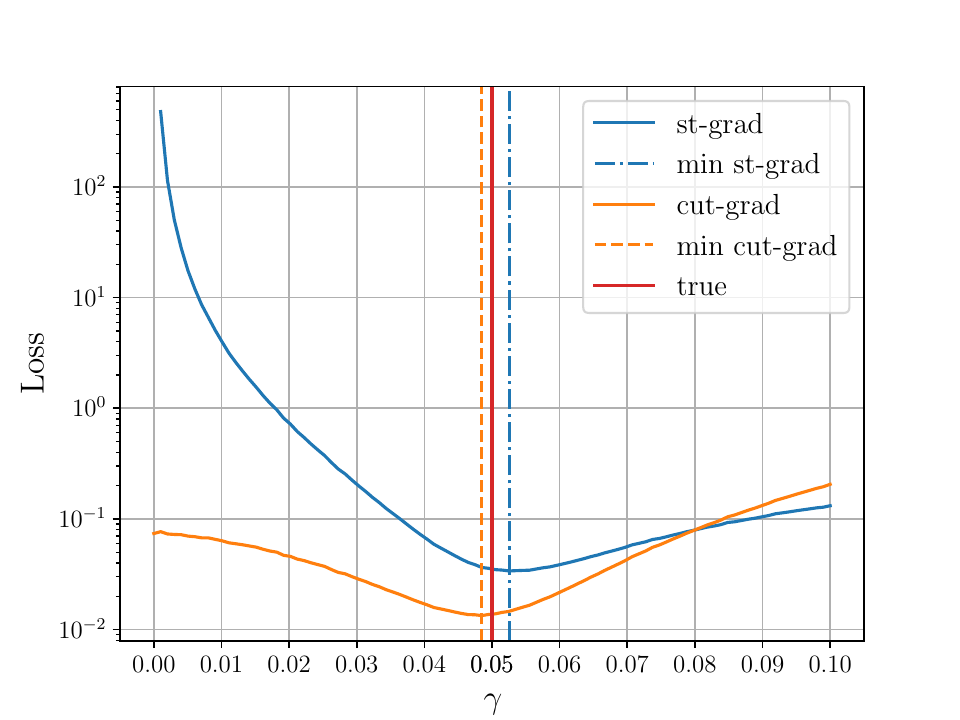}}  
    \subfloat[$N=1000$]{\includegraphics[width=0.45\linewidth]{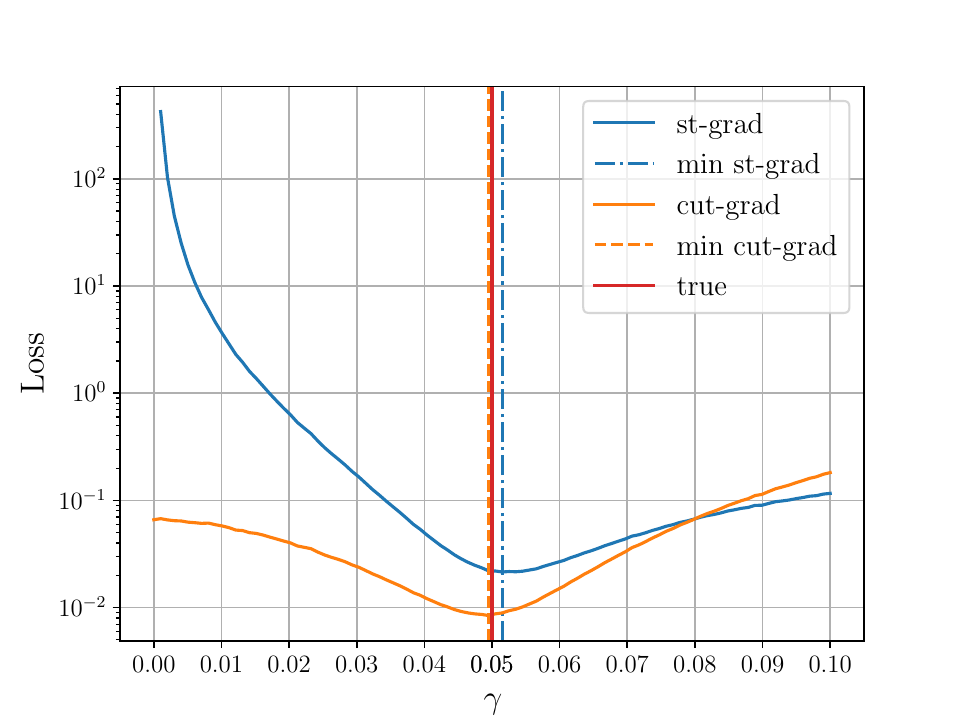}}
    \caption{Loss functions \eqref{eq:O1a} (cut-gradient) with $\gamma'$ fixed at 0.08, and \eqref{eq:O2a} (standard-gradient). We average the loss function over $100$ sets of 100 projection angles $\theta\sim\U(\Sb^{d_y-1})$.}\label{fig:darcy_loss_comparison}
\end{figure}
We will compare the effect of varying the available data between the standard-gradient and cut-gradient approaches.

\noindent\textbf{Data Specifics.} The data for this example is generated using a log-normal distribution on $z\in\Rb_+$ as $\mu(\alpha^\dagger)=\log\cN(m^\dagger, (\sigma^\dagger)^2)$ with values ($m^\dagger, \sigma^\dagger)=(0.5, 0.25)$. In this example the noise distribution $\eta=\cN(0, (\gamma^\dagger)^2\I)$ with $\gamma^\dagger=0.05$.

\noindent\textbf{Learning Specifics.}
To demonstrate the different behaviours of~\eqref{eq:A1} and~\eqref{eq:A2} as the number of available system realizations, $N$, grows, we plot in Figure~\ref{fig:darcy_loss_comparison} the loss function~\eqref{eq:O1a} (cut-gradient) with fixed ${\gamma'}=0.08$, the loss \eqref{eq:O2a} (standard-gradient), and the true data generating parameter $\gamma^\dagger$. We also show the loss minimizers $\gamma^\star$ for both losses.
Of course, gradients are not involved in computing the losses for values of $\gamma$, however both losses nonetheless behave differently from one another with varying $N$ because of the undesirable scaling dependence of \eqref{eq:O2a} on $\gamma$.

\noindent\textbf{Results.}
From Figure~\ref{fig:darcy_loss_comparison} we can see both estimates for $\gamma$ are close to the ground truth when $N$ is large, however the minimizer of~\eqref{eq:O1a} is closest to the ground truth for smaller $N$. Notice also the loss function values intersect at the value $\gamma'=0.08$ (kept fixed in the cut-gradient loss and hence, the curves intersect at this value).
The undesirable scaling dependence on $\gamma$ in~\eqref{eq:O2a} is then reflected in the numerical solution to the optimization problem formulated through the pseudo-time discretized \eqref{eq:A2}, a trait not present in the cut-gradient version~\eqref{eq:A1}, as will be shown in the following subsection.

\subsubsection{Convergence Comparison}\label{ssec:4_1_2}
To establish more strongly the superior robustness to empiricaliztion of the cut-gradient gradient method, in Figure~\ref{fig:darcy_convergenceplots1} we setup a series of learning problems.

\noindent\textbf{Data Specifics.} We vary the amount of data  available from $N$ from $10^1$ to $10^4$. The data is generated with different values of $\gamma$ in the uncorrelated noise covariance; we generate data with $z^{(n)}\sim \mu(\alpha^\dagger)$ as in Subsection~\ref{sssec_4_1_1}. 

\noindent\textbf{Learning Specifics.}  We perform the optimization 50 times using the time discretized flows~\eqref{eq:A1},~\eqref{eq:A2} randomizing initialization and data. We study both approaches, while varying the $N$ from $10^1$ to $10^4$, for various ground truth values of $\gamma^\dagger$. We note that we also learn $\alpha=(m, \sigma)$ during this procedure. We use the Adam oprimizer~\cite{kingma2014adam} with learning rate of $10^{-1}$.

\noindent\textbf{Results.} We plot the mean and standard deviations of the relative error in estimating $\gamma^\dagger$.
\begin{figure}[t]
    \centering
    \subfloat[$\gamma^\dagger=0.01$]{\includegraphics[width=0.32\linewidth]{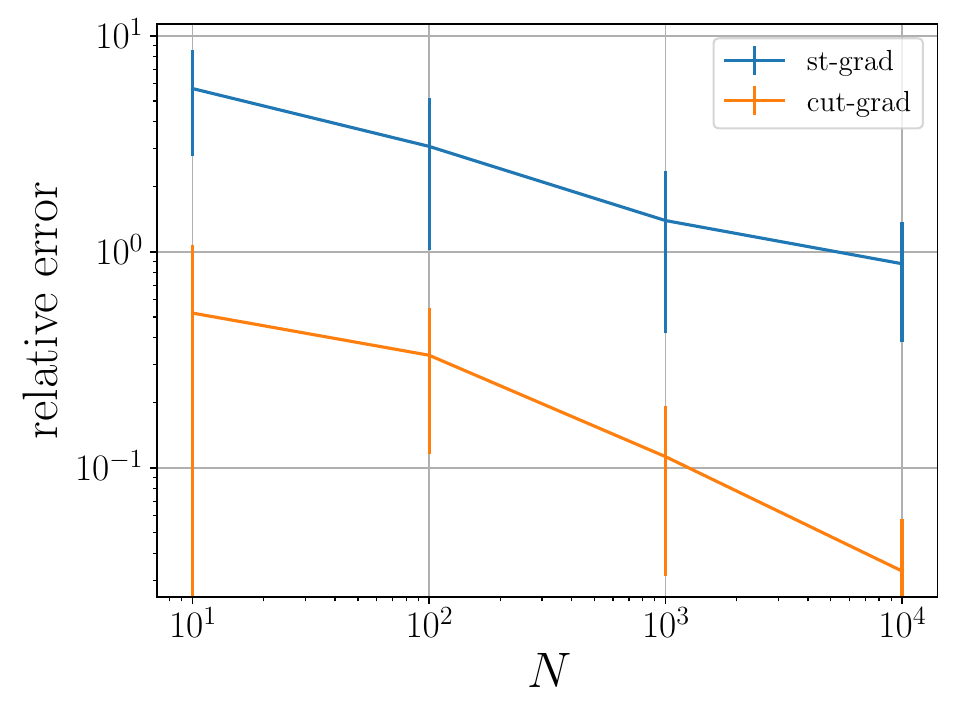}}
    \subfloat[$\gamma^\dagger=0.025$]{\includegraphics[width=0.32\linewidth]{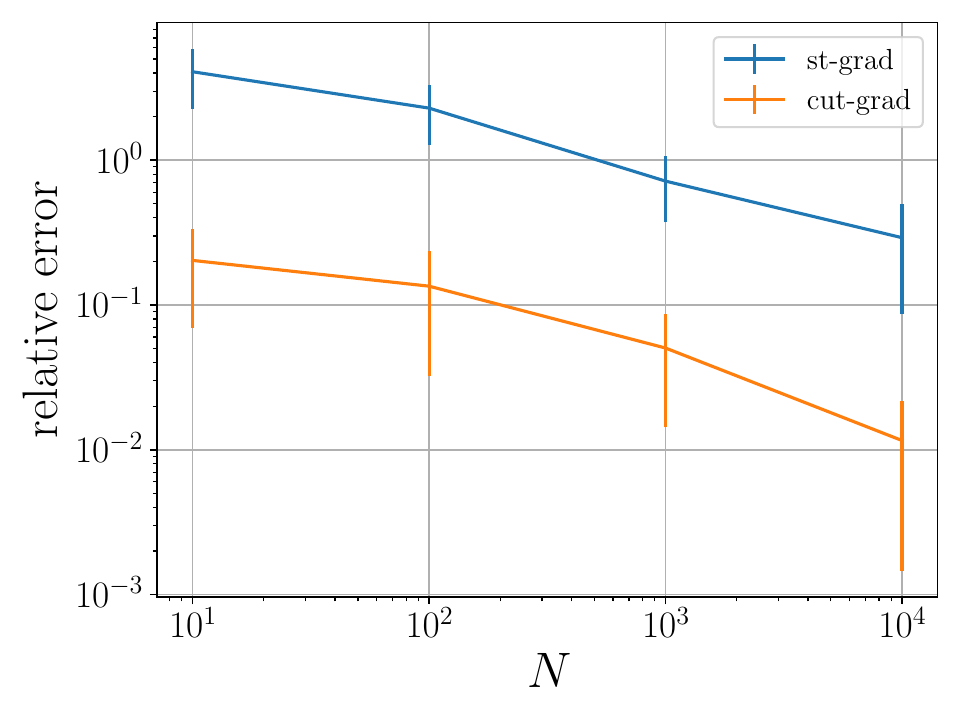}}   
    \subfloat[$\gamma^\dagger=0.063$]{\includegraphics[width=0.32\linewidth]{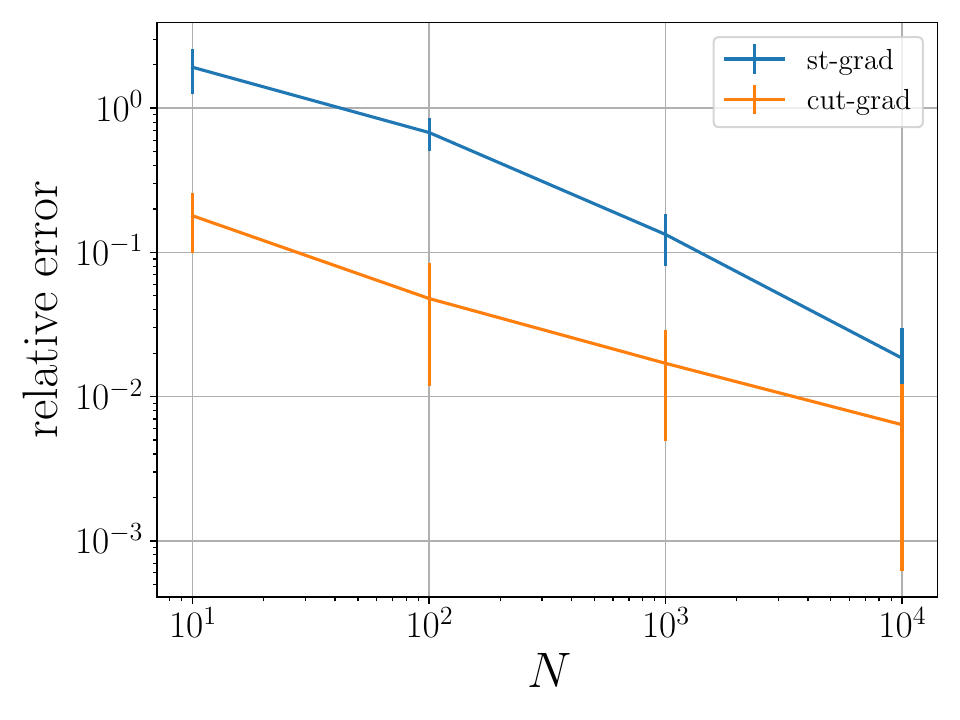}}\\
    \subfloat[$\gamma^\dagger=0.158$]{\includegraphics[width=0.32\linewidth]{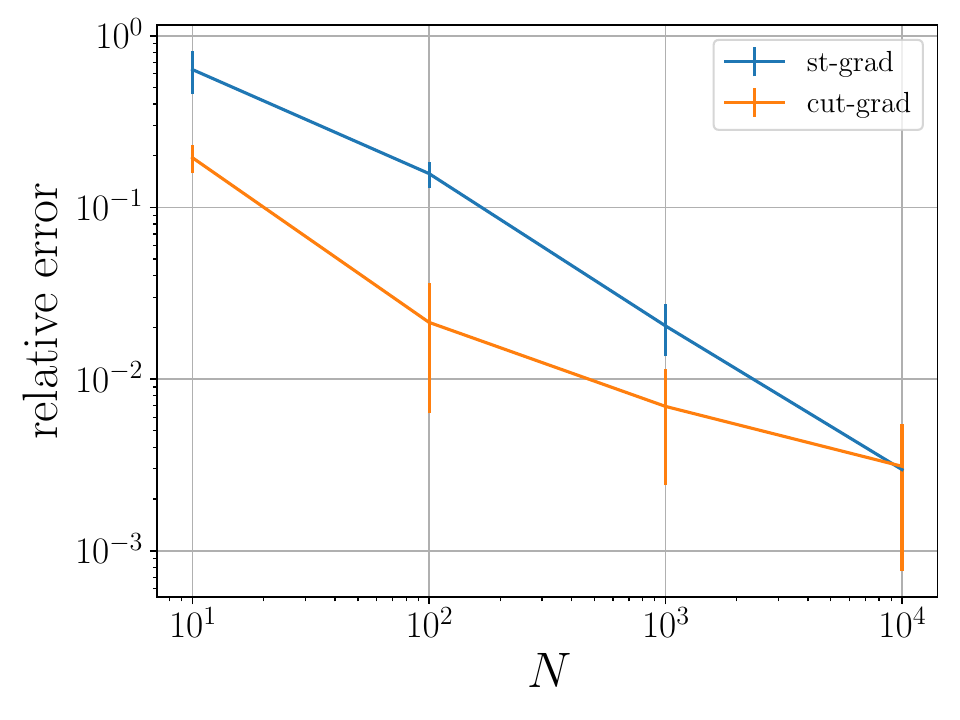}}
    \subfloat[$\gamma^\dagger=0.398$]{\includegraphics[width=0.32\linewidth]{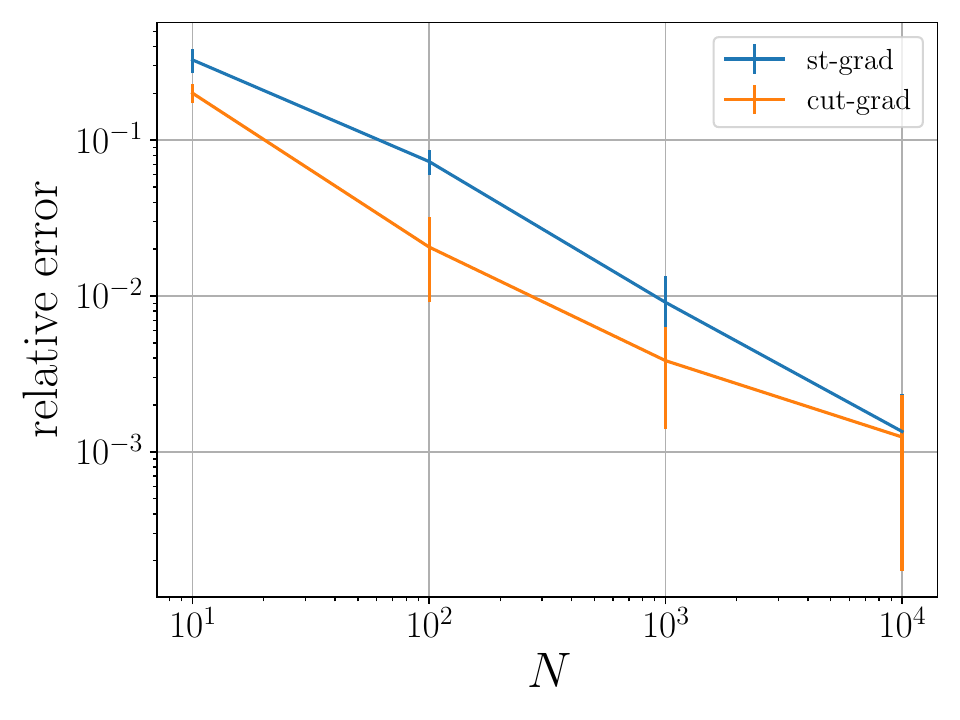}}
    \subfloat[$\gamma^\dagger=1.$]{\includegraphics[width=0.32\linewidth]{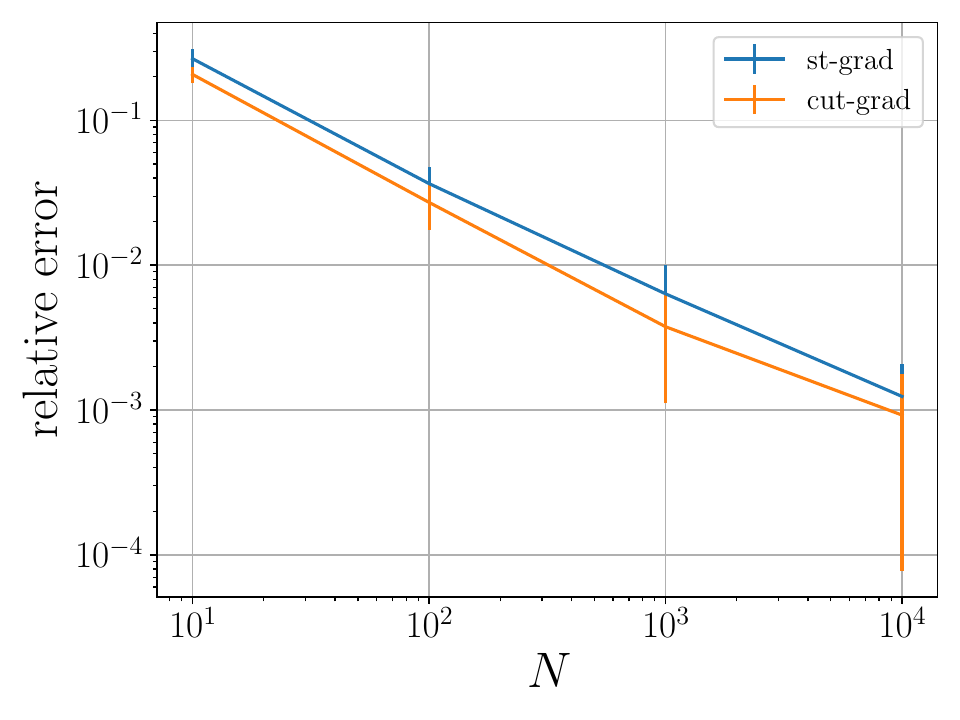}}
    \caption{Scaled diagonal noise estimation for the porous medium flow problem obtained though stochastic optimisation over $(\alpha, \gamma)$ averaged over 50 runs using \eqref{eq:A1} (cut-gradient) and \eqref{eq:A2} (standard-gradient).}\label{fig:darcy_convergenceplots1}
\end{figure}
Having established the improved empiricalization performance of~\eqref{eq:A1} in the scaled diagonal noise case, we move on to more involved covariance specifications.

\subsection{Correlated Noise}
\begin{figure}[t]
\centering
    \subfloat[5 data samples]{\includegraphics[width=0.48\linewidth]{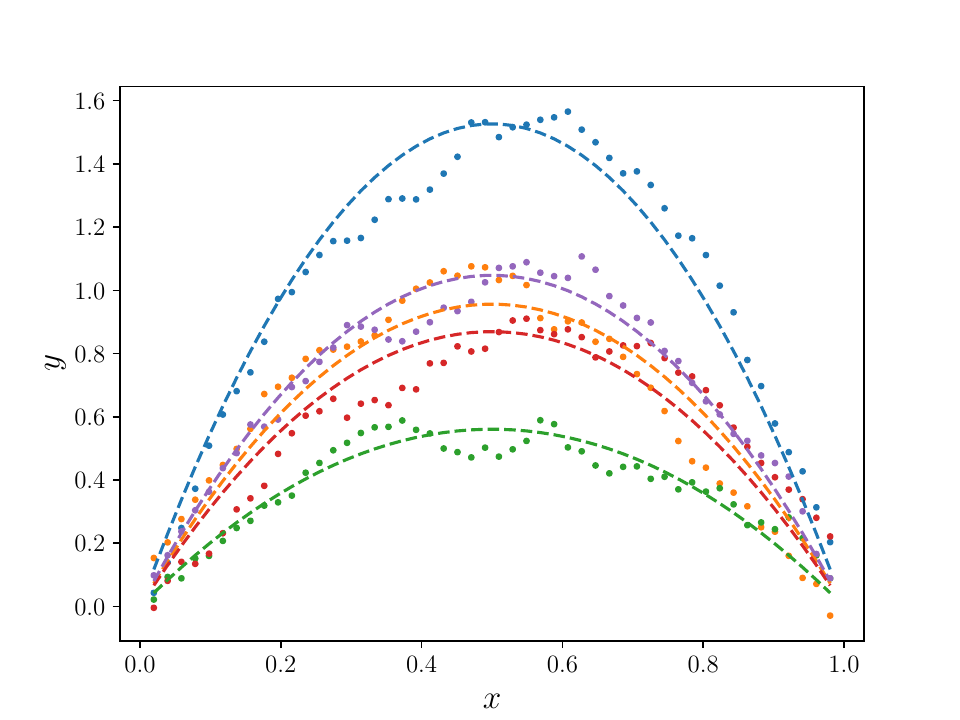}\label{fig:Darcy1D_data_corrlated_A}}
    \subfloat[Covariance]{\includegraphics[width=0.48\linewidth]{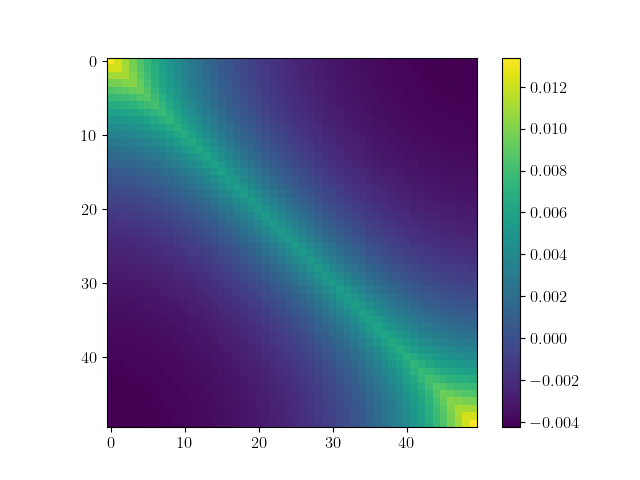}\label{fig:Darcy1D_data_corrlated_B}}
    \caption{Five solutions of the porous medium  flow problem with Wittle-Mat\'ern process noise with $\beta^\dagger$ parameters $\ell^\dagger=0.25$, $\gamma^\dagger=0.1$, $\upsilon^\dagger=0.5$ and $\alpha$ parameters $m=0.5$, $\sigma=0.5$. (a) shows 5 solution along with observations with spatially correlated corruptions, (b) the data generating noise covariance matrix $\Gamma(\beta^\dagger)$.\label{fig:Darcy1D_data_corrlated}}
\end{figure}

To build on the previous subsection we now look at a similar setting as in Subsection~\ref{ssec:4_1} but with a correlated noise distribution.
To do this, we parametrize $\xi^{(n)}\sim\eta(\Gamma):=\cN(0, \Gamma(\beta))$ with $\beta=(\gamma, \ell, \upsilon)$ representing amplitude, lengthscale, and roughness respectively. We specify $\Gamma(\beta)$ to be a Whittle-Mat\'ern process~\cite{LassiRoininen2014InverseProblemsandImaging, dunlop2017hierarchical} 
\begin{subequations}\label{eq:wm_covariance}
\begin{align}
    C_{\gamma, \ell, \upsilon} &= \sigma\ell^d(I-\ell^2\Delta)^{-\upsilon-\frac{d}{2}},\\
    \sigma &= \gamma^2\frac{2^d\pi^{d/2}\sf{Gamma}(\upsilon+d/2)}{{\sf{Gamma}}(\upsilon)}.
\end{align}
Here $\Delta$ is the Laplacian,  applied to functions satisfying homogenous Neumann boundary conditions
and spatial mean zero. Notation $\sf{Gamma}$ is for the standard Gamma function~\cite{abramowitz1965handbook} -- written in this way to avoid confusion with our covariance matrix $\Gamma$.
\end{subequations}
The covariance operator is diagonalized as $\Gamma(\beta)=Q\Lambda Q^{-1}$, where $\Lambda$ is a diagonal matrix of eigenvalues and $Q$ is a matrix of eigenfunctions $\psi(x)$ specified by the Helmholtz problem~\cite{gottlieb1985eigenvalues}
\begin{subequations}
\begin{align}
    \Delta\psi + \lambda\psi &= 0, \quad x\in D, \\
    \frac{\partial\psi}{\partial \hat{n}} &= 0, \quad x\in\partial D,
\end{align}
\end{subequations}
where $\hat{n}$ is the normal to $\partial D$. For $D=(0,1)$ we have solutions $\lambda_j=j^2\pi^2$ and $\psi_j(x)=\cos(j\pi x)$; the resulting eigenvalues of 
$C_{\gamma, \ell, \upsilon}$ are $\sigma\ell^d(\ell^2j^2\pi^2+1)^{-\upsilon-d/2}$. 
We write the expansion for $\xi^{(n)}\sim\eta(\Gamma)$ as a pushforward of a standard normal through independent identically distributed (i.i.d) draws $\xi^{(n)}_{1,j}$ as a Karhunen-Lo\`eve expansion
\begin{align}\label{eq:wm_kl_expansion}
    \xi^{(n)} = \sum_{j=1}^{J+1} \sqrt{\sigma\ell^d(\ell^2j^2\pi^2+1)^{-\upsilon-d/2}}\psi(x) \xi^{(n)}_{1,j};\quad \xi^{(n)}_{1,j}\sim \cN(0,1).
\end{align}
\subsubsection{Convergence Comparison}\label{ssec:4_2_1}
\begin{figure}[t]
    \centering
    \subfloat[$\ell^\dagger=0.01$]{\includegraphics[width=0.32\linewidth]{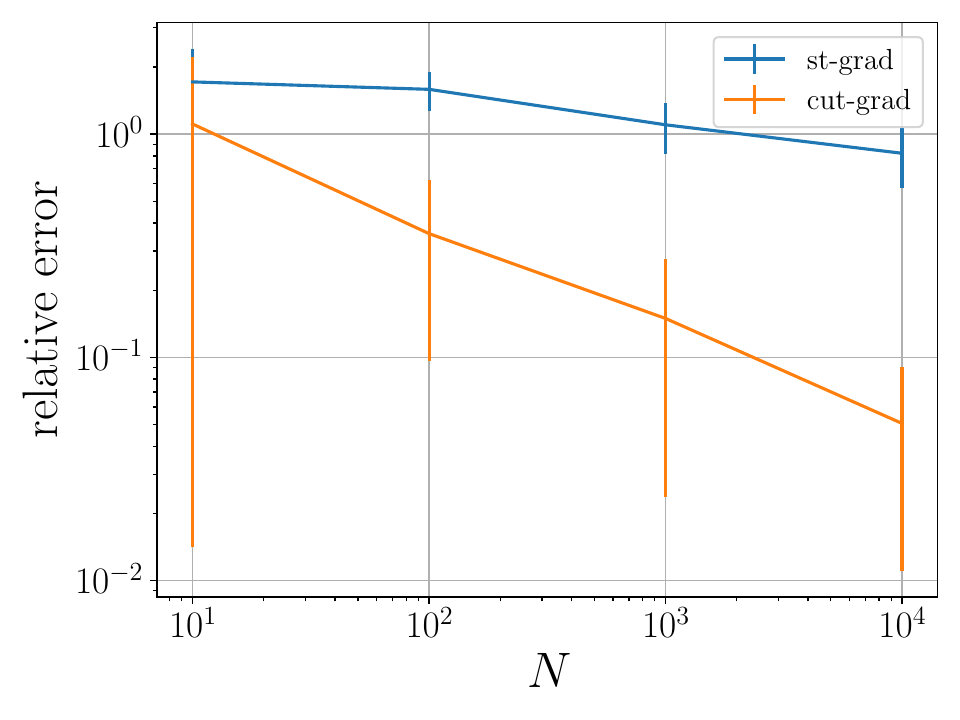}}
    \subfloat[$\ell^\dagger=0.035$]{\includegraphics[width=0.32\linewidth]{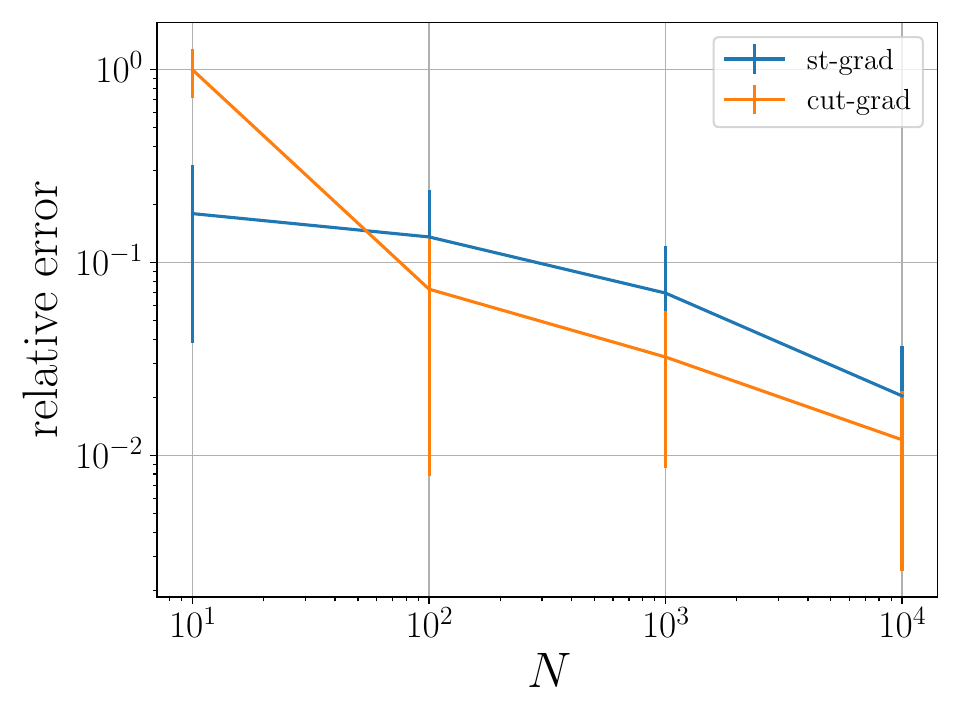}}   
    \subfloat[$\ell^\dagger=0.120$]{\includegraphics[width=0.32\linewidth]{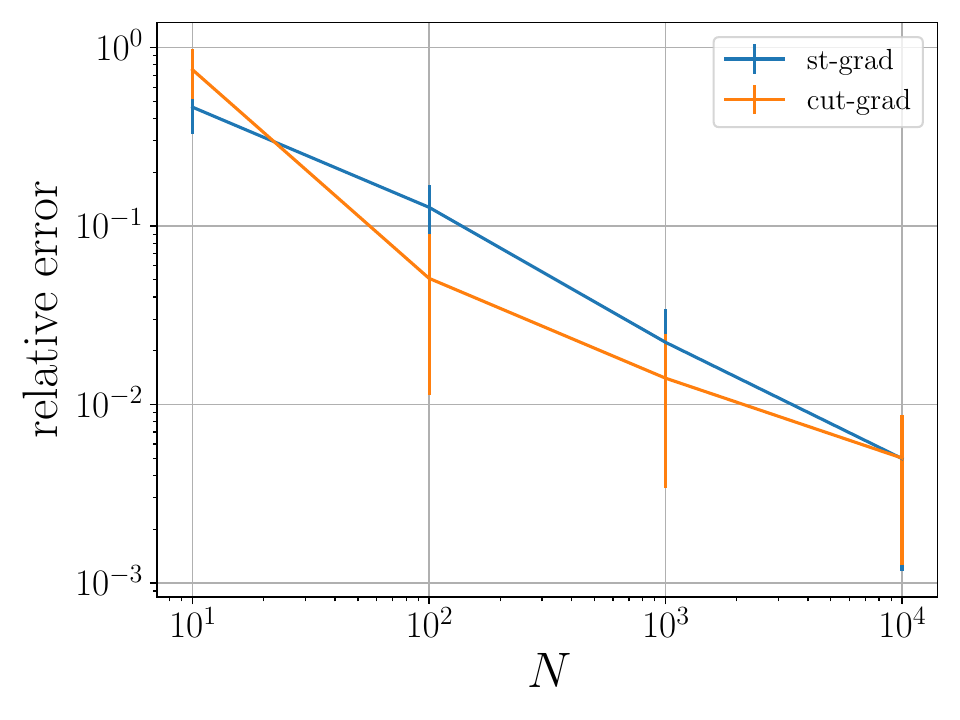}}\\
    \subfloat[$\ell^\dagger=0.416$]{\includegraphics[width=0.32\linewidth]{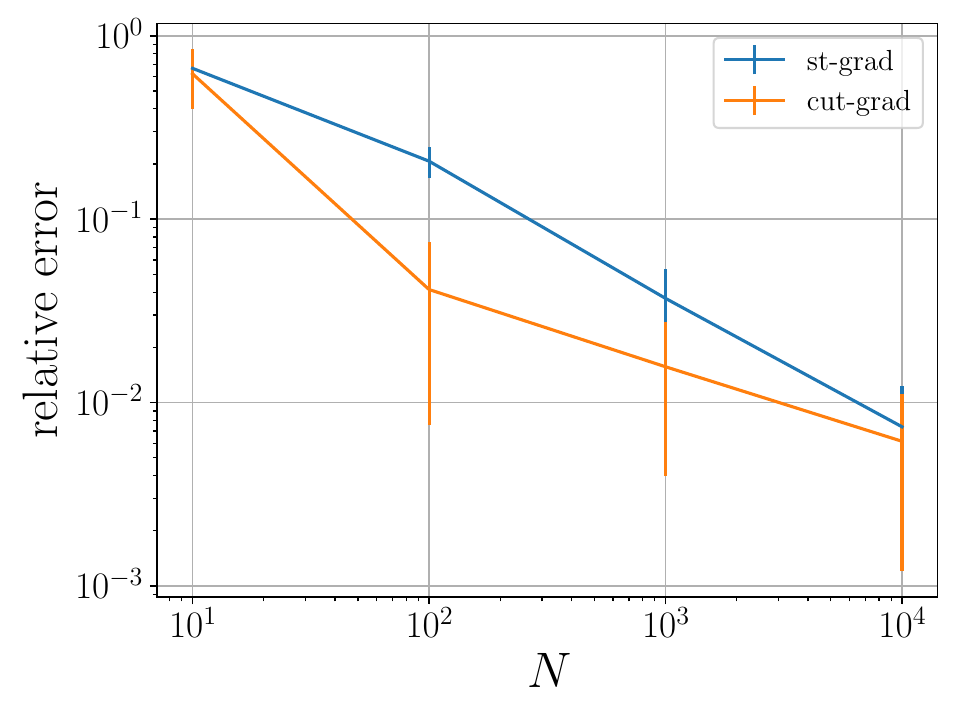}}
    \subfloat[$\ell^\dagger=1.443$]{\includegraphics[width=0.32\linewidth]{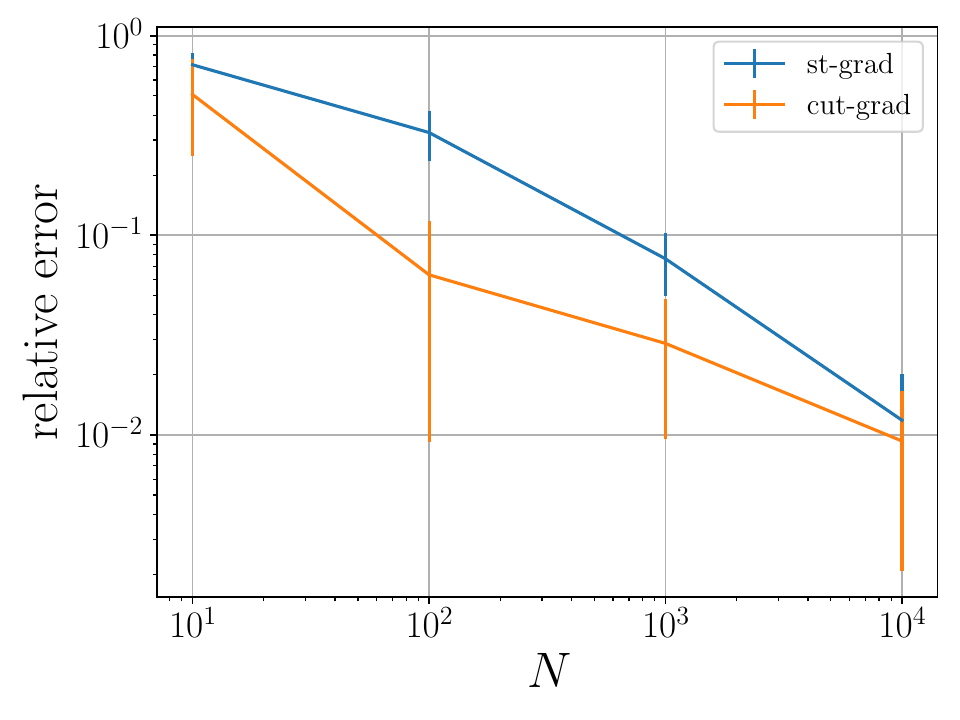}}
    \subfloat[$\ell^\dagger=5.0$]{\includegraphics[width=0.32\linewidth]{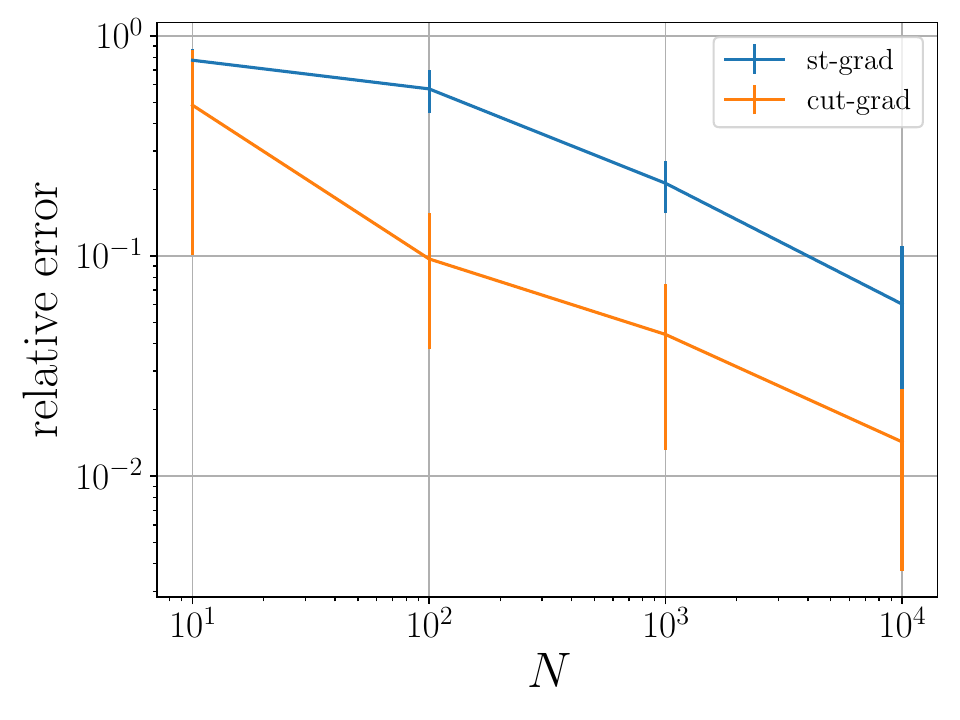}}
    \caption{Whittle-Mat\'ern noise estimation for the porous medium flow problem obtained though stochastic optimisation over $(\alpha, \beta)$ averaged over 100 runs with $\beta=( \gamma^\dagger, \ell, \upsilon^\dagger)$ ($\gamma^\dagger, \upsilon^\dagger$ fixed) using \eqref{eq:A1} (cut-gradient) and \ref{eq:A2} (standard-gradient) for varying $\ell^\dagger$ and $N$.\label{fig:Darcy_learn_ell}}
\end{figure}
We now make use of the aforementioned correlated covariance parametrization for the noise deconvolution and populational inversion problem on the porous medium flow problem. As in Subsection~\ref{ssec:4_1_2} we run a series of repeated optimizations and plot predictive errors for the true data generating parameters.

\noindent\textbf{Data Specifics.}
We specify $\mu(\alpha)$ similarly to Subsection~\ref{sssec_4_1_1}, however we now set $\Gamma(\beta)$ to be the correlated covariance in~\eqref{eq:wm_covariance} with $\beta=(\gamma, \ell, \upsilon)$ for different values of $\ell$.
In Figure~\ref{fig:Darcy1D_data_corrlated_A} we show 5 example sampled solutions $u$ along with the observations with \textit{correlated} noise. In Figure~\ref{fig:Darcy1D_data_corrlated_B} we show an example of the data-generating covariance matrix $\Gamma(\beta^\dagger)$. 

\noindent\textbf{Learning Specifics.}
To facilitate the computation of the weighted norm in the sliced-Wasserstein term needed to compute $\Gamma'^{-1/2}$, we use the normalized diagonalized form of the covariance operator
$\Gamma'(\beta)^{-1/2} = \Q\Lambda_\beta^{-1/2}\Q^\top$ where
$\Q_{ij} =\sqrt{2/d_y} \cos(j \pi x_i)$. 
One could also consider using a modified discrete cosine transform, or explicitly construct $\Gamma(\beta)$~\cite{solin2020hilbert} in order to compute the square root and solve against $y$.
We run the convergence study  100 times randomizing over the data and initialization. 

\noindent\textbf{Results.}
We solve the series of deconvolution and populational inversion problems by simulating~\eqref{eq:A1} (cut-gradient) and~\eqref{eq:A2} (standard-gradient) over $(\alpha, \beta)$, where $\alpha=(m, \sigma)$ is specified in~\ref{ssec:4_1} and $\beta=(\gamma^\dagger, \ell, \upsilon^\dagger)$ with $\gamma^\dagger=0.1, \upsilon^\dagger=0.5$ fixed, for various $N$ from $10^1$ to $10^4$. The relative errors for both approaches for varying $N$ and data generating lengthscale values $\ell^\dagger$ are shown in Figure~\ref{fig:Darcy_learn_ell}. 
\begin{figure}[t]
    \centering
    \subfloat[Loss]{\includegraphics[width=0.32\linewidth]{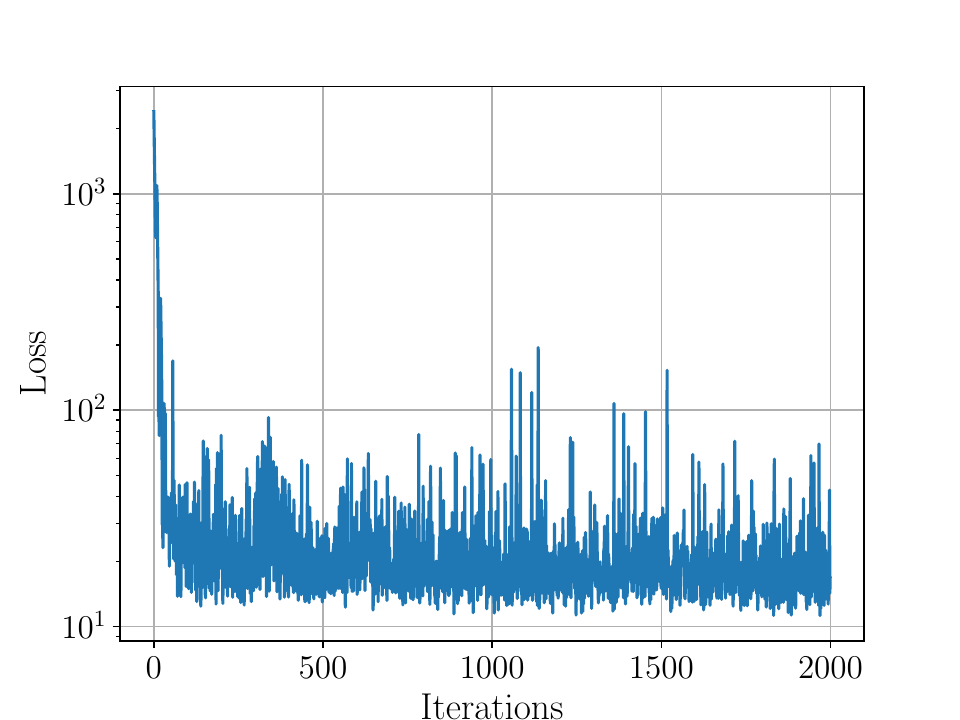}}
    \subfloat[Convergence $\beta$]{\includegraphics[width=0.32\linewidth]{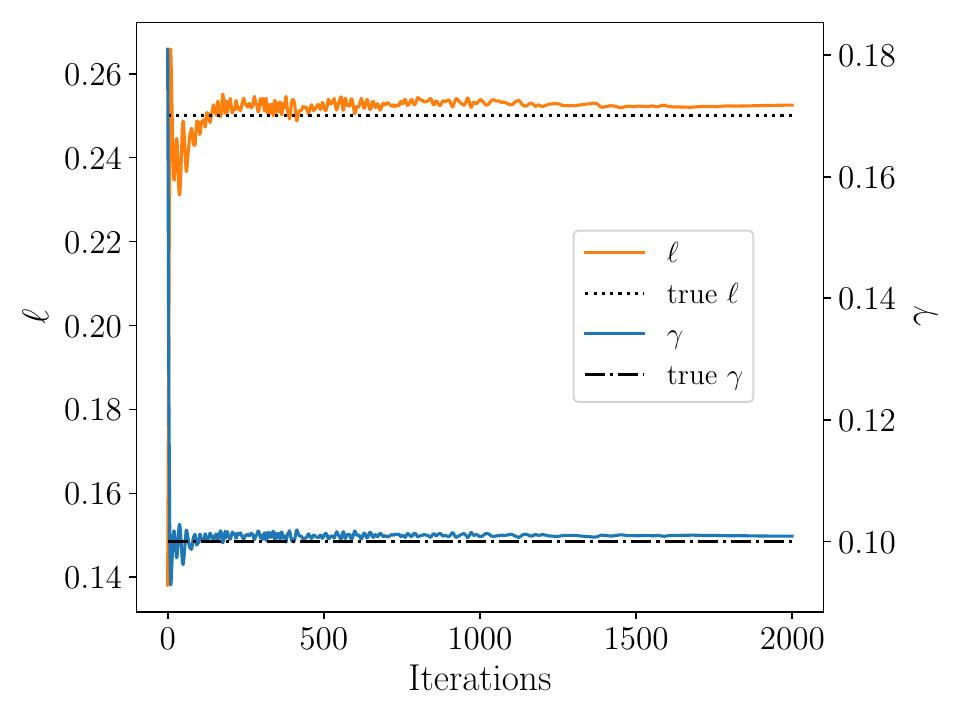}}
    \subfloat[Convergence $\alpha$]{\includegraphics[width=0.32\linewidth]{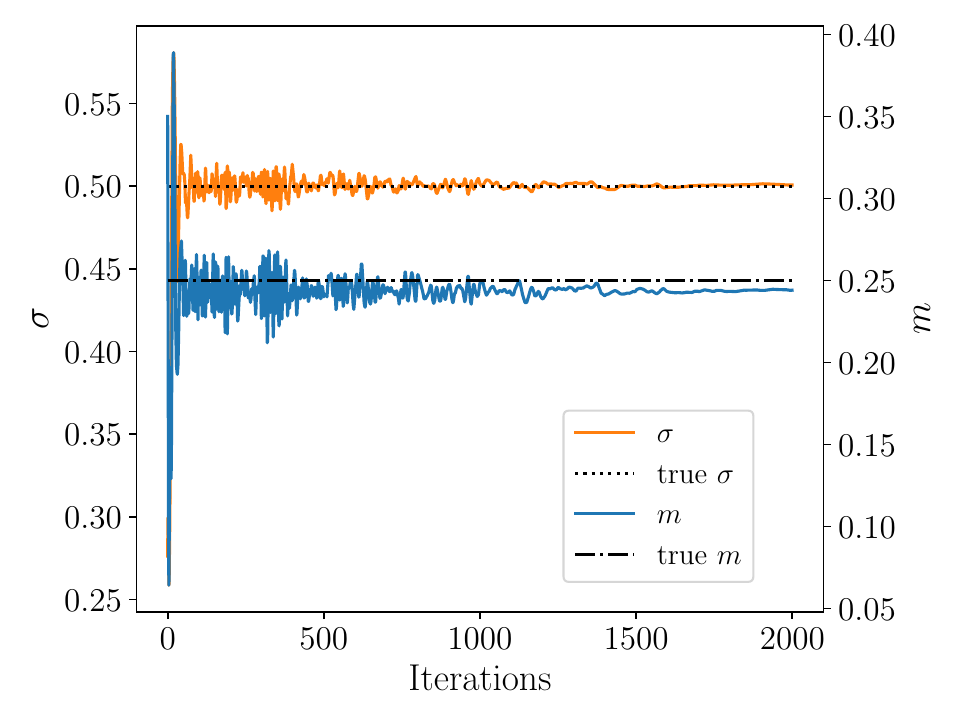}}
    \caption{Deconvolution and distributional inversion for the 1D porous medium flow problem with Wittle-Mat\'ern process noise and $10^4$ data vectors $y$. (a) shows the loss function, (b) the convergence of $\beta$, (c) the convergence of $\alpha$.\label{fig:Darcy1D_ConvergenceKL}}
\end{figure}

\subsubsection{Combined learning}\label{ssec:4_2_2}
We conclude the numerical examples with the porous medium flow problem by solving a deconvolution and distributional inversion problem where we jointly learn the noise amplitude $\gamma$, noise lengthscale $\ell$, and the permeability field log-normal distribution parameters $m$, $\sigma$. 

\noindent\textbf{Data Specifics.} The parameters $\beta=(\gamma, \ell, \upsilon^\dagger)=(0.1, 0.25,0.5)$, with $\upsilon^\dagger$ fixed. We set the number of data and number of samples used in the loss to $N=10^4$ 
and we generate data with $z^{(n)}\sim \mu(\alpha^\dagger)$ as in Subsection~\ref{sssec_4_1_1}. 

\noindent\textbf{Learning Specifics.} We set the number of samples $N_s$ used in the loss to also be $10^4$. Having previously established the superiority of the cut-gradient form, we focus solely on using~\eqref{eq:A1} for inference in this subsection.

\noindent\textbf{Results.}
We show the convergence plot over iterations in Figure~\ref{fig:Darcy1D_ConvergenceKL}. We use the Adam optimizer~\cite{kingma2014adam} an initial learning rate of $10^{-1}$ decayed by half 10 times over the 2k iterations. 
The relative errors on the last 100 iterations are: 0.21\% for $\sigma$, 2.25\% for $m$, 0.88\% for $\gamma$, and 0.99\% for $\ell$.

\section{Elastodynamics}\label{sec:elastodyn}

In relation to Contribution~\ref{con:C5}, we investigate the effectiveness of the proposed methodology when applied to a model of two-dimensional
elastodynamics, subject to Rayleigh damping~\cite{zhang2021novel, bleyer2024comet}.
Let $u=(u_1, u_2)^\top$ denote the two dimensional displacement field and
\begin{align}
    &\varepsilon = \frac{1}{2}\bigl(\nabla u + (\nabla u )^\top \bigr),
\end{align}
the induced strain. Balance of linear momentum gives the following closed model for the evolution of $u=u(x,\kt):$\footnote{We note the gradient flow pseudo-time is $t$, the discretized algorithm pseudo-time is $\st$, and the physical systems real time is $\kt$}
\begin{subequations}\label{eq:damped_elastodyn_sf}
    \begin{align}
    &\rho \ddot{u}  = \nabla \cdot \bsigma -\rho R_m\dot{u} - f;\quad x, \kt\in D\times [0,T],\\
    &\bsigma = \mathsf{C} : (\varepsilon+R_k\dot{\varepsilon});
\end{align}
\end{subequations}
here $\rho$ is the density, $\bsigma$ the Cauchy stress tensor, $\mathsf{C}$ the fourth order stiffness tensor, $f$ the body force, $R_m, R_k$ the Rayleigh coefficients associated to the mass and stiffness respectively. We consider the equations on a two dimensional physical domain $D=[0,l]\times[-h/2,h/2]$, $x=(x_1, x_2)\in D$. The equations are subject to the boundary and initial conditions
\begin{subequations}
\begin{align}
    &\bsigma((x_1, -h/2),\kt)\cdot \hat{n} = \bsigma((x_1, h/2),\kt)\cdot \hat{n} = 0,\\
    &u((0,x_2),\kt) = u((l, x_2),\kt) = 0,\\
    &u(x, 0) = 0,\\
    &\dot{u}(x,0) = 0.
\end{align}
\end{subequations}
Thus there is initially no displacement, the body is at rest, the body is clamped at the edges in the $x_1$ direction and free at the edges in the $x_2$ direction. The body force $f$ sets the beam in motion and will be specified in each numerical example subsection. For an isotropic inhomogeneous (piece-wise constant or weakly inhomogeneous) material we take
\begin{align}
    \bsigma(u) = \lambda_e(x) \mathrm{tr}(\varepsilon)\I + 2\mu_e(x)\varepsilon,
\end{align}
where $\lambda_e(x), \mu_e(x)$ are the two lam\'e parameters, which we specify in terms of the Young's modulus $z(x)$ and the Poisson ratio, $\nu_e$, in two spatial dimensions as~\cite{salenccon2012handbook, slaughter2012linearized}
 \begin{align}
     \lambda_e(x) = \frac{z(x)\nu_e}{(1+\nu_e)(1-\nu_e)}, \quad\quad
     \mu_e(x) = \frac{z(x)}{2(1+\nu_e)}.
 \end{align}
 We model the random spatially varying Young's modulus $z(x)\sim\mu$ as a shifted lognormal field
\begin{subequations}
\label{eq:youngsrandomfield}
\begin{align}
    z((x_1,x_2)) &= E_e(\exp(a(x_1)) + b),\\
     a(x_1)&\sim\cN(0, C^l_{\gamma_\mu, \ell_\mu, \upsilon_\mu}).
\end{align}    
\end{subequations}
 Here $E_e$ is a base Young's modulus, and $b$ is a shift factor ensuring the sampled spatially varying Young's modulus is not so small as to incur numerical instability in the elastodynamics simulation -- and it is set to $10^{-1}$. The covariance $C^l_{\gamma_\mu, \ell_\mu, \upsilon_\mu}$ is defined as in~\eqref{eq:wm_covariance} with domain-adapted eigenfunctions $\cos(j\pi x/l)$ replacing each
 $\psi_j(x):=\cos(j\pi x)$ found there.
 We highlight that the material properties vary only in the $x_1$ direction and
 not the $x_2$ direction.

We may formulate the PDE weakly.
To be specific we consider the following setting: $z \in L^\infty(D; \bR_{> b E_e})$, $u\in C([0,T]; H^1(D;\bR^2))$, $f\in L^\infty([0,T]; L^2(D; \bR^2))$.\footnote{Spaces for $\dot{u}, \ddot{u}$ are the same as for $u$.}
The solution operator maps a 2D Young's modulus field to a 2D time varying field, however the Young's only varies in one dimension and we only expect small variation of the solution field in the $x_2$ direction. Hence, we specify $F^\dagger$ appearing in~\eqref{eq:data_gen} as a map from a 1D function to the horizontal mid-plain of the time-varying solution. This has the form $F^\dagger:L^\infty([0,l]; \bR_{> b E_e})\mapsto C([0,T]; H^1_0([0,l];\bR^2))$. 
In Figure~\ref{fig:beamdefo} we show the beam at the peak of its dynamic deformation. 
\begin{figure}[t]
    \centering
    \includegraphics[width=1\linewidth]{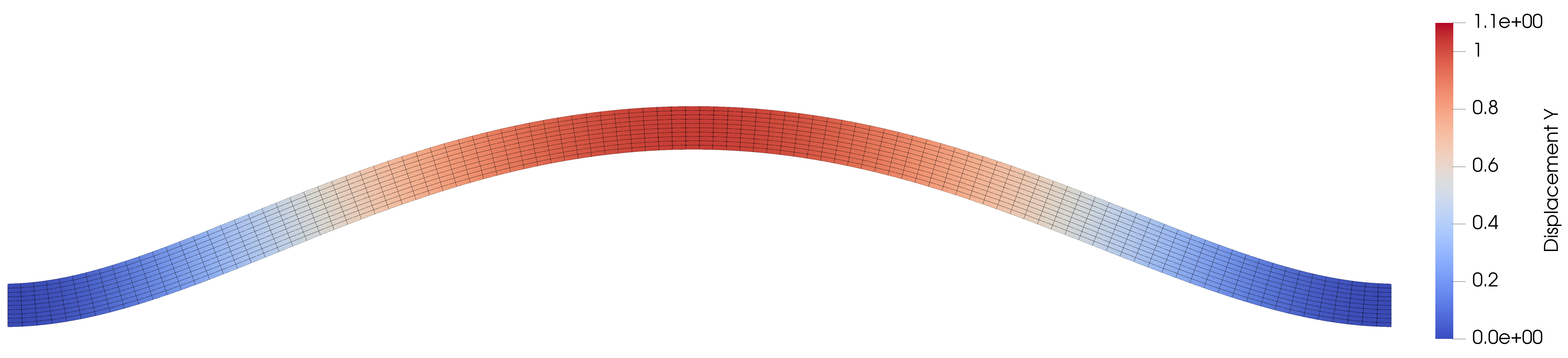}
    \caption{Snapshot of displacement of a sampled elastodynamic model with $100\times 10$ FE elements along the length and height of the domain.}
    \label{fig:beamdefo}
\end{figure}

For the surrogate modelling task to be compatible with a standard neural operator setup we replicate the Young's modulus across time as $z(x,\kt)=z(x), \forall \kt\in [0, T]$ with $z(x)$ defined in~\eqref{eq:youngsrandomfield}. Thus, we can use a 2D FNO~\cite{li2020fourier}, $F^\phi:C([0,T]; L^\infty([0,l]; \bR_{> b E_e})\mapsto C([0,T]; H^1_0([0,l];\bR))$ and map the 1D Young's modulus function to the vertical displacement of the beam along its mid-height across time.
We learn this surrogate model while concurrently learning $(\alpha, \Gamma)$ using~\eqref{eq:learn_phi} and Algorithm~\ref{alg:algorithm2}.

In the following subsections, three observation and inference setups are considered.
\begin{itemize}
    \item The first of these, Subsection~\ref{ssec:5_1}, considers the case where we measure accelerations at various spatial locations along the beam at very high frequnecy in time -- we may view the data as continuous time; the signal is corrupted with uncorrelated Gaussian noise; from this corrupted data we infer the noise standard deviation $\gamma_\eta$, the Young's modulus random field amplitude $\gamma_\mu$, and the Young's modulus random field lengthscale $\ell_\mu$.
    \item  The second example, Subsection~\ref{ssec:5_2}, considers the measurement of displacements sparsely in time and in space; the data generating noise is correlated in time and space with a Whittle-Mat\'ern field; we learn a misspecified noise model to learn the standard deviation of a Gaussian with a scaled diagonal noise $\gamma_\eta$, along with the same $\alpha$ parameters $\gamma_\mu, \ell_\mu$ as in Subsection~\ref{ssec:5_1}.
    \item  The last elastodynamic example, Subsection~\ref{ssec:5_3}, considers displacement measurements collected densely in space and in time; the noise is correlated in space and time as in Subsection~\ref{ssec:5_2}; we recover both the amplitude $\gamma_\eta$, and the lengthscale $\ell_\eta$ of the correlated noise random field, along with the same Young's modulus random field information as in Subsection~\ref{ssec:5_1}. 
\end{itemize}

We now detail the setup information common to all three elastodynamic examples, further information specific to each setup is explained in the individual example subsections.

\noindent\textbf{Data.}
We use $N=10^4$ data realizations.
We set in all three examples, $\alpha=(\gamma_\mu^\dagger, \ell_\mu^\dagger, \upsilon_\mu^\dagger) = (0.25$, $1$, $2$) in~\eqref{eq:youngsrandomfield}; $\upsilon^\dagger_\mu$ is kept fixed at its true value and not learned in the inference scheme. We collect data from time $\kt=0$ to $\kt=T=1$.

\noindent\textbf{Learning.}
We use $N_s=500$ samples per evaluation of~\eqref{eq:O1b}.
We set the regularization in~\eqref{eq:O1b} to have the form $h(\alpha) = 1/(2\cdot10^2)\|\log(\ell_\mu, \gamma_\mu)-\log(1, 2)\|^2$.
We learn the surrogate model FNO with $N_F=100$ batch-size to learn $\phi^\star_\st$ in~\eqref{eq:learn_phi} with a pretraining number of iterations of $\sT_\mathrm{pre.} = 10^3$ with $N_\mathrm{pre.}=100$, and $\sT_\mathrm{inner} = 10$ step on $\phi$ minimizing $\eqref{eq:learn_phi}$ per step on $(\alpha, \Gamma)$. We terminate the training data-pairs acquisition after $\sT_a = 2500$ steps. On $(\alpha, \Gamma)$ we use a decaying learning rate initialized at $10^{-1}$ halved 8 times over training; for $\phi$ we use a decaying learning rate initialized at $10^{-3}$ decayed by half 8 times; both with the Adam optimizer~\cite{kingma2014adam}. 

\noindent\textbf{Solver.}
For all three examples in this section we specify the number of elements in the FE model to be $100\times10$ in $x_1,x_2$ repectively and the nuerical time-integration step is $\Delta_\kt=0.01$. We also set $\rho=7.8\cdot 10^{-3}\, U_mU_d^{-3}$, $\nu_e=0.3$, $E_e=10^4 \,U_f U_d^{-2}$, where $U_m$ are units of mass, $U_d$ are units of distance, $U_f=U_m U_d U_\kt^{-2}$ are units of force, $U_\kt$ are units of time. The $N_\beta, N_\gamma$, parameters for the Newmark scheme are $1/4, 1/2$, and we set $R_m, R_k$,  to $10^{-1}, 5\cdot10^{-3}$, respectively.
For computational efficiency, we do not update $\bP^\st_{z,u}$ every step (Algorithm~\ref{alg:algorithm2} lines 8-11) -- but withhold computing new training pairs $u, z$ until a sufficient batch is collected in order to take advantage of parallelization capabilities of modern multi-core hardware;  this batch-size is set to 60. We use FEniCSx code~\cite{baratta2023dolfinx, alnaes2014unified,bleyer2024comet, langtangen2017solving} for the solver which uses an implicit Newmark-$\beta$ time integrator~\cite{newmark1959method}. 

\subsection{Observing Acceleration}
\label{ssec:5_1}
\begin{figure}[t]
    \centering
    \subfloat[Solution and observations]{\includegraphics[width=0.32\linewidth]{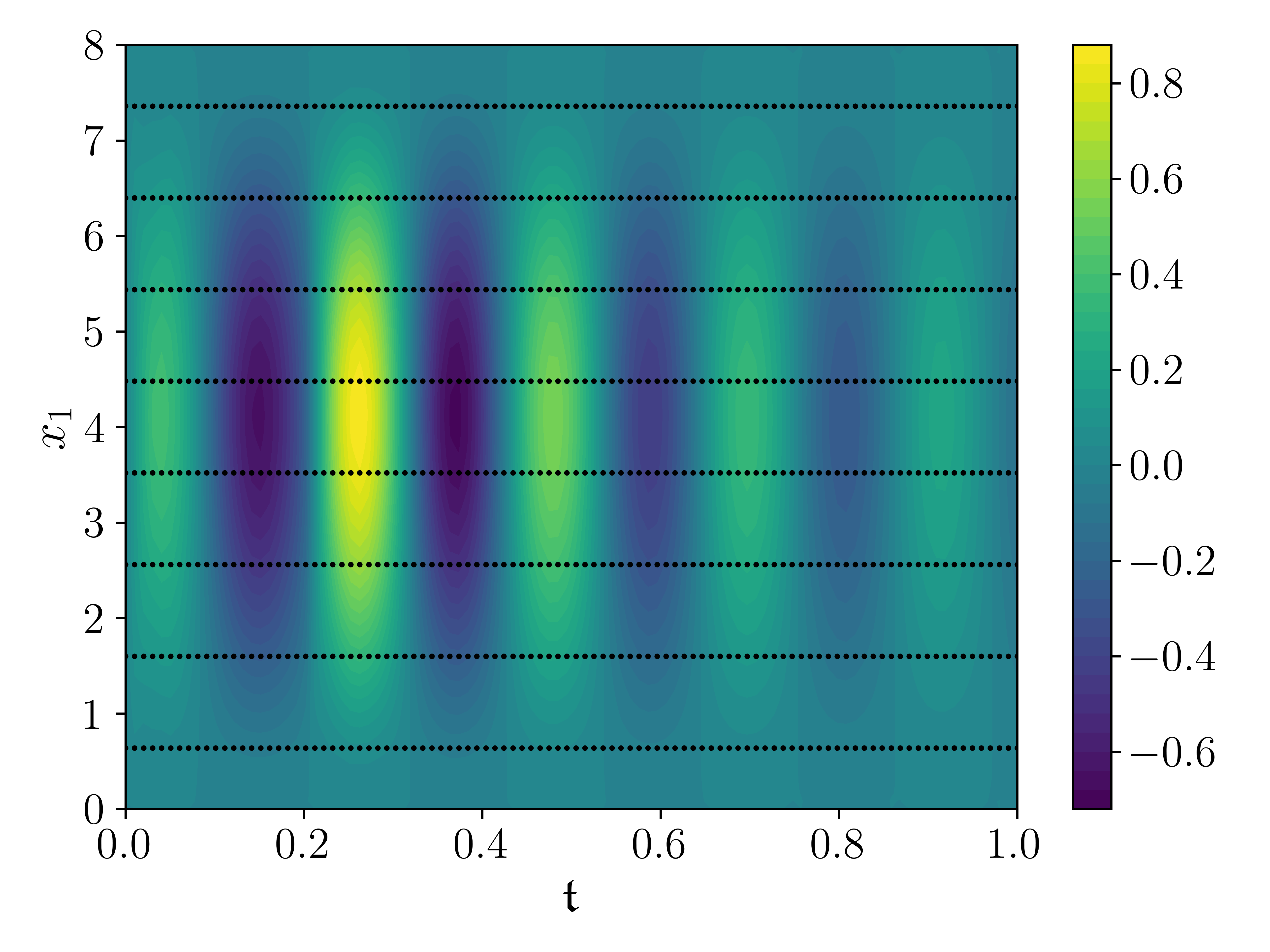}\label{fig:elasto:1:U0}}
    \subfloat[data at locations]{\includegraphics[width=0.32\linewidth]{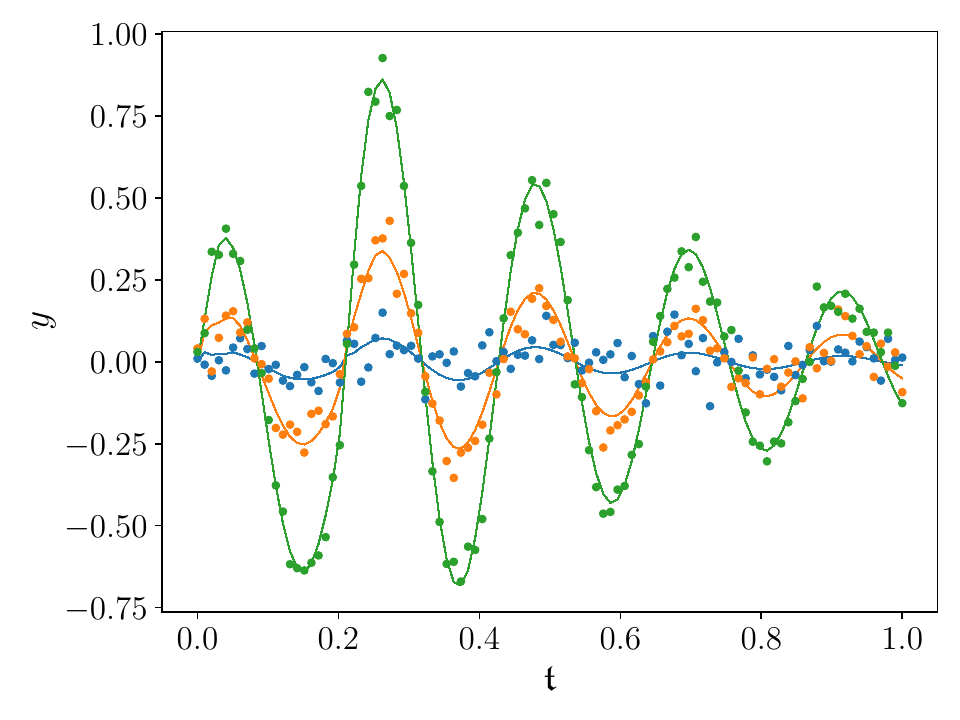}\label{fig:elasto:1:obsT}}
     \subfloat[data at times]{\includegraphics[width=0.32\linewidth]{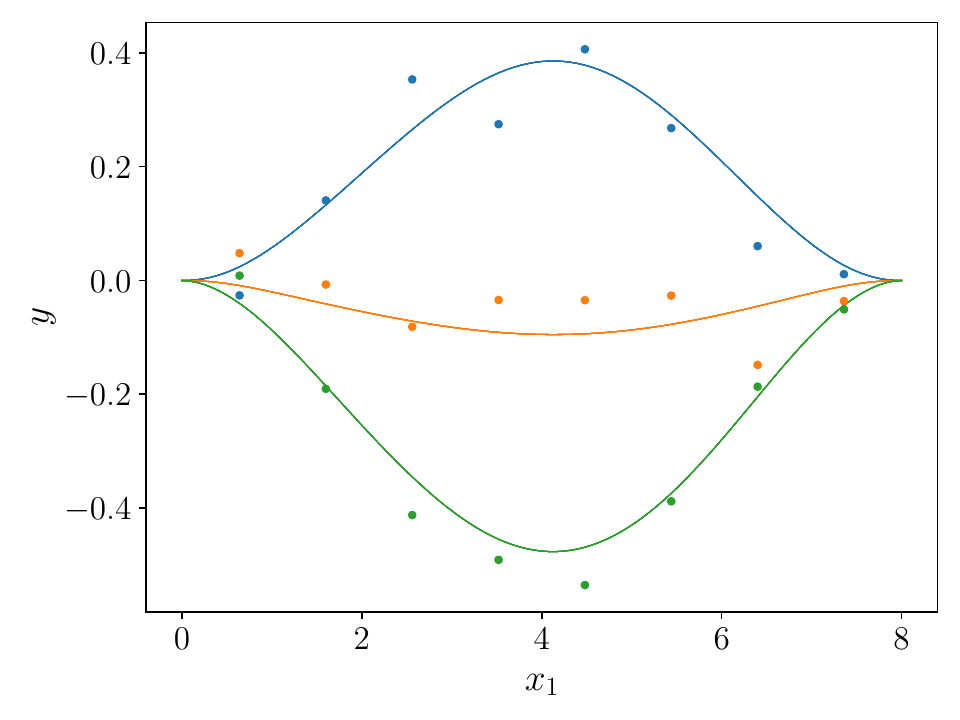}\label{fig:elasto:1:obsX}}
    \caption{(a) Sample solution with dotted observation locations. (b) Solution with continuous time observation lines. (c) Solution and observational data across beam at different time instances.}
\end{figure}
\begin{figure}[t]
    \centering
    \subfloat[Data Loss]{\includegraphics[width=0.48\linewidth]{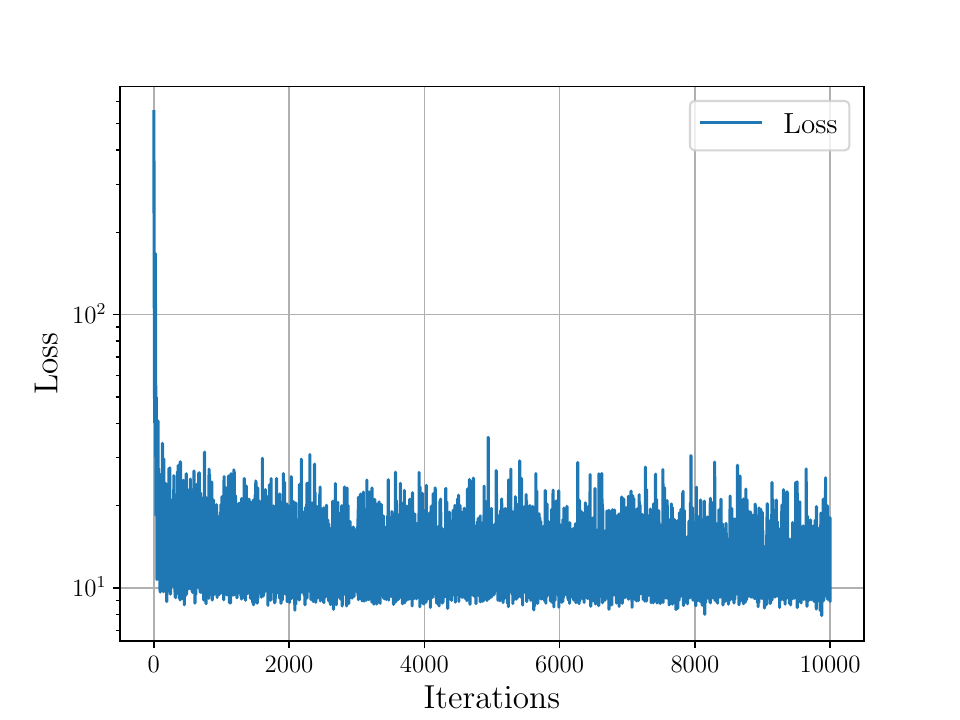}}
    \subfloat[Surrogate Loss]{\includegraphics[width=0.48\linewidth]{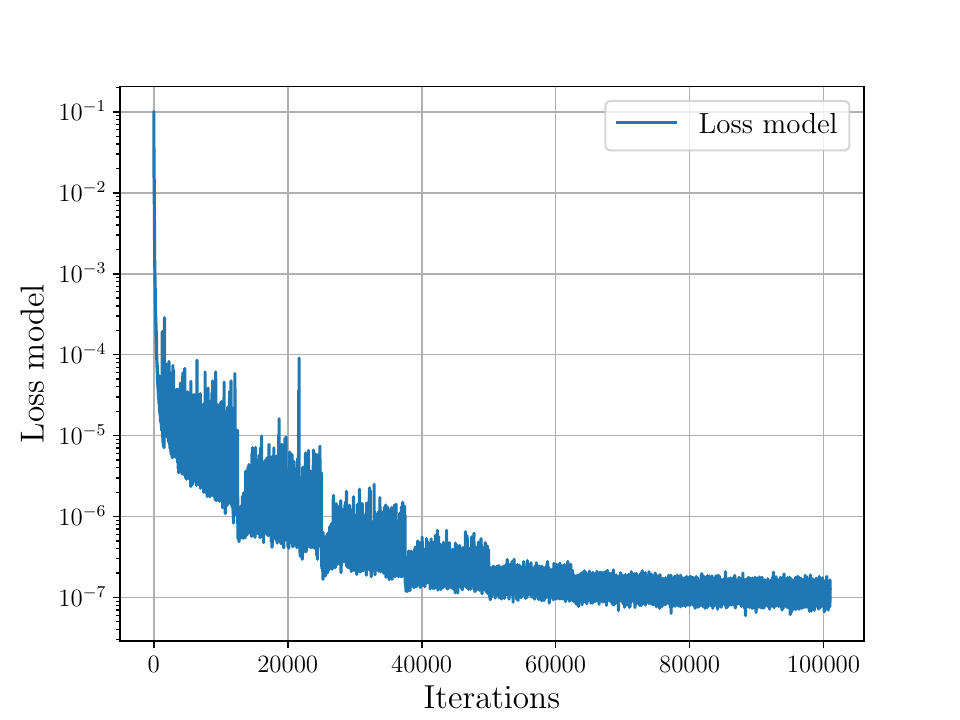}}
    \caption{Data-fit Loss and Surrogate model supervised loss on empirical measure. Only the first 2500 iterations collect PDE solver solutions into $\bP_{z,u}^\st$~\eqref{eq:cummul_emp_dataset}.}
    \label{fig:elasto:1:losses}
\end{figure}
\begin{figure}[t]
    \centering
    \subfloat[$\gamma_\mu$-convergence]{\includegraphics[width=0.32\linewidth]{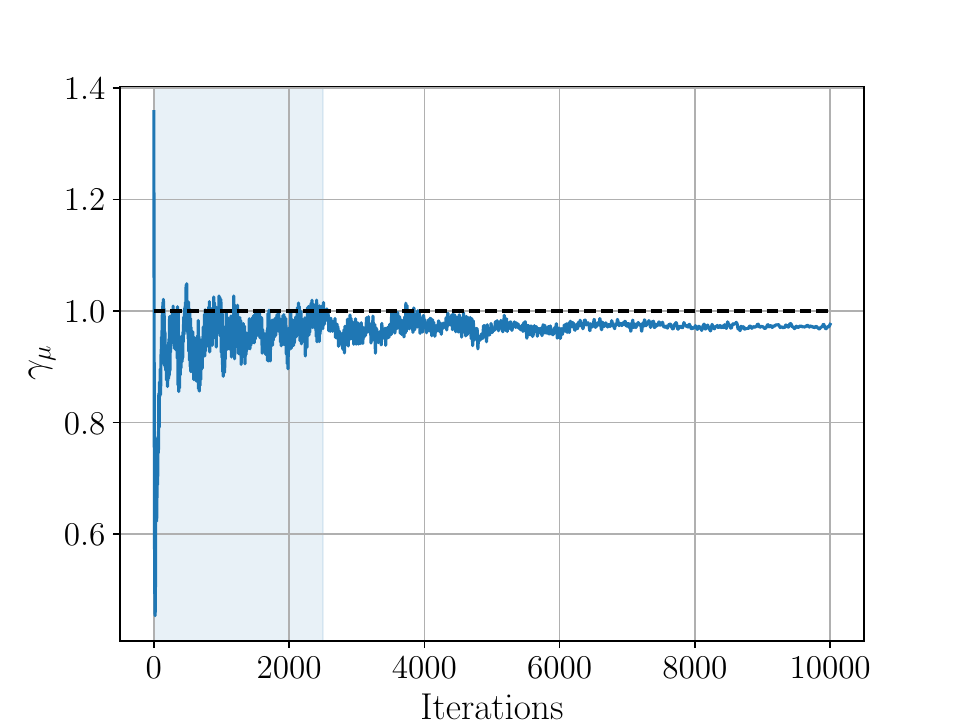}}
    \subfloat[$\ell_\mu$-convergence]{\includegraphics[width=0.32\linewidth]{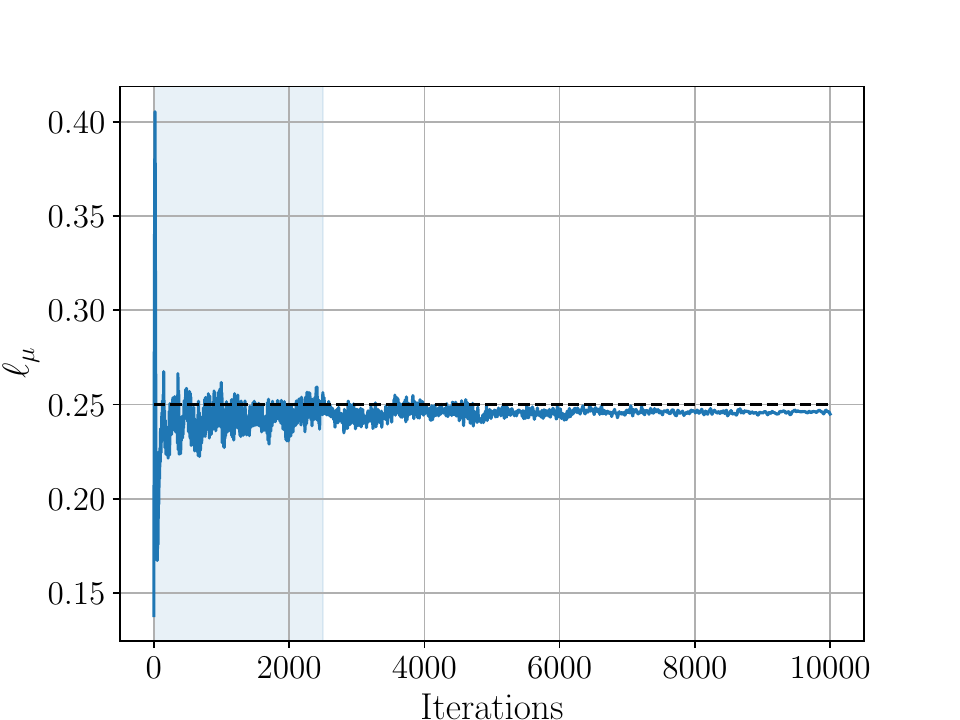}}
    \subfloat[$\gamma_\eta$-convergence]{\includegraphics[width=0.32\linewidth]{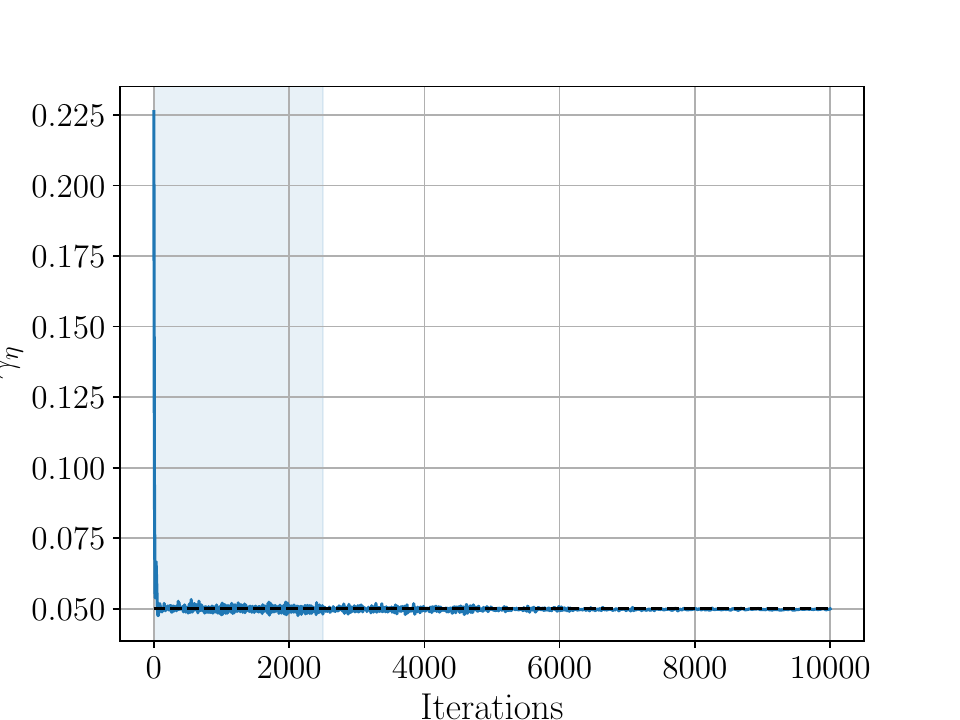}}
    \caption{Convergence of the material properties' distribution for (a) material random field amplitude $\gamma_\mu$ (b)  lengthscale $\ell_\mu$, and (c) the noise standard-deviation $\gamma_\eta^2$; learned parameter (solid line) along with true value (black dashed line).}
    \label{fig:elasto:1:convergence}
\end{figure}
In this example we measure beam accelerations sparsely in space and densely in time. 

\noindent\textbf{Data Specifics.}
We generate the data with uncorrelated noise with $\Gamma^\dagger = {\gamma_\eta^2}^\dagger\I$ which we learn as $\Gamma(\beta)=\gamma_\eta^2\I$ (where $\beta=(\gamma)$), and the observation operator $g$ collects data in space at 8 locations and continuously in time (100 time steps). 
Accelerations are scaled by a factor of $10^{-3}$.
The standard deviation of the noise is 0.05 (corresponding to a continuous-time white noise process  with standard deviation of 0.005 ($0.05/\sqrt{\Delta_\kt}$ with $\Delta_\kt=0.01\,U_\kt$)). 
In Figure~\ref{fig:elasto:1:U0} we show a sampled spatiotemporal solution field with the observation locations given as black dots. 
In Figures~\ref{fig:elasto:1:obsT} and~\ref{fig:elasto:1:obsX} we show three cross-sections of the spacetime solutions at differente spatial locations and time instances, along with the observations corresponding to each solution instance. The uncorrelated nature of the noise is visible.

\noindent\textbf{Learning Specifics.}
The regularizer over $\beta$ in~\eqref{eq:O1b} is $r(\beta)=1/(2\cdot10^2)\|\log(\gamma_\eta)-\log(1)\|^2$.

\noindent\textbf{Solver Specifics.}
The vector valued forcing function setting the beam in motion for this example is constant in space and quadratic in time with the form 
\begin{align}
    f(x, \kt)\begin{cases}
        (0, 5-500(\kt-0.1)^2),\quad & 0\leq \kt \leq\kt_f\\
        0, \quad & \kt>\kt_f
    \end{cases}
\end{align} 
where $\kt_f=0.2\,U_\kt$ is the cutoff time of the forcing, thereafter the beam is in free vibration; the forcing is positive and reaches a peak of $5\,U_f$.

\noindent\textbf{Results.}
In Figure~\ref{fig:elasto:1:losses}, we show the regularized data-fit~\eqref{eq:O1b} and in the surrogate model loss~\eqref{eq:learn_phi}. Note the ten-fold increase in steps on the surrogate model loss due to $\sT_\mathrm{inner}=10.$ We can also see the drops in the loss, most visible in the surrogate loss (b), where the learning rate is decayed by half.
In Figure~\ref{fig:elasto:1:convergence} we show the convergence of the parameters $\alpha=(\gamma_\mu, \ell_\mu)$ and $\beta=(\gamma_\eta)$. The data-acquisition phase $0\leq\st<\sT_{a}$ is shown as a shaded light blue region.
The relative errors averaging over the last 100 iterations are: 2.80\% for $\gamma_\mu$, 1.56\% for $\ell_\mu$, and 0.36\% for $\gamma_\eta$.

\subsection{Observing Displacements Sparsely}
\label{ssec:5_2}

In this example, the set of observations for each beam is the displacement measured at sparse spatial and temporal locations. The true noise distribution is correlated but is learned a misspecified uncorrelated approximation. 
\begin{figure}[t]
    \centering
    \subfloat[Solution and observations]{\includegraphics[width=0.32\linewidth]{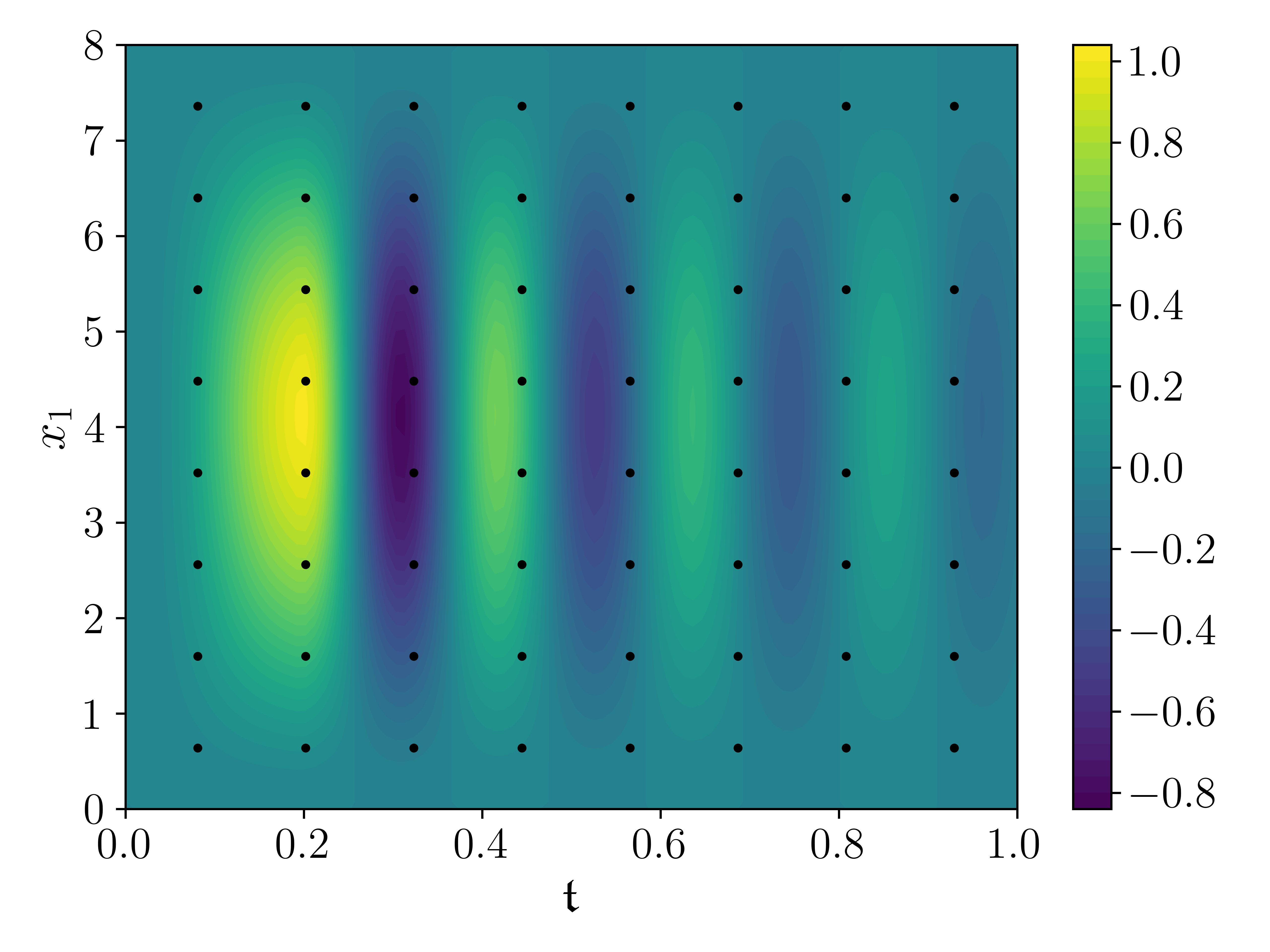}\label{fig:elasto:2:U0}}
    \subfloat[data at locations]{\includegraphics[width=0.32\linewidth]{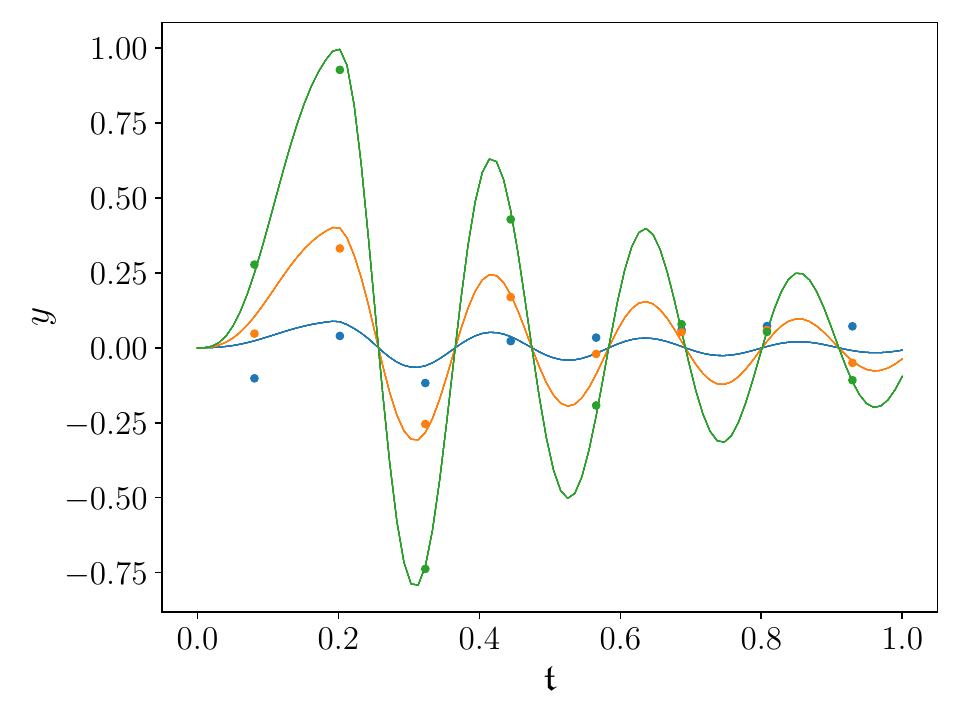}\label{fig:elasto:2:obsT}}
     \subfloat[data at times]{\includegraphics[width=0.32\linewidth]{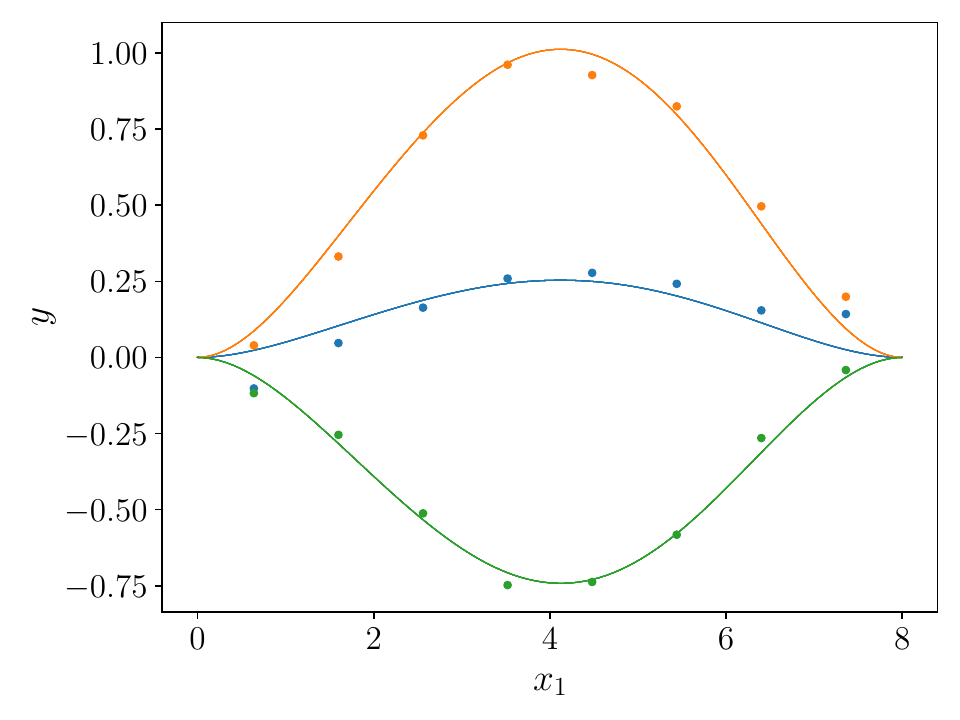}\label{fig:elasto:2:obsX}}
    \caption{(a) Sample solution with dotted observation locations. (b) Solution and observational data across time at different locations. (c) Solution and observational data across beam at different time instances.}
\end{figure}
\begin{figure}[t]
    \centering
    \subfloat[Data Loss]{\includegraphics[width=0.48\linewidth]{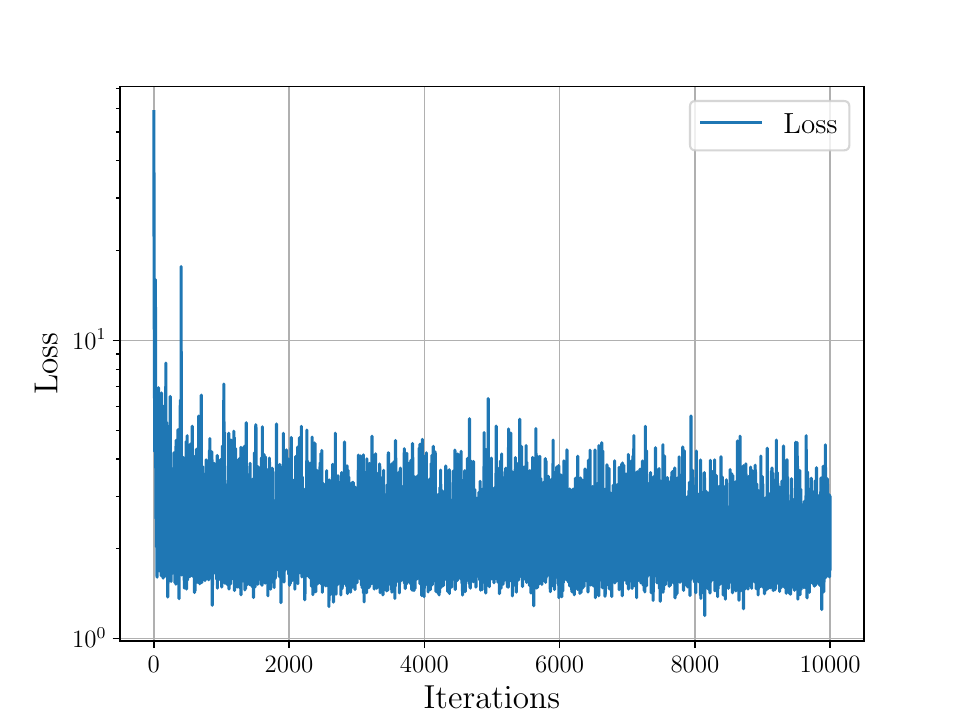}}
    \subfloat[Surrogate Loss]{\includegraphics[width=0.48\linewidth]{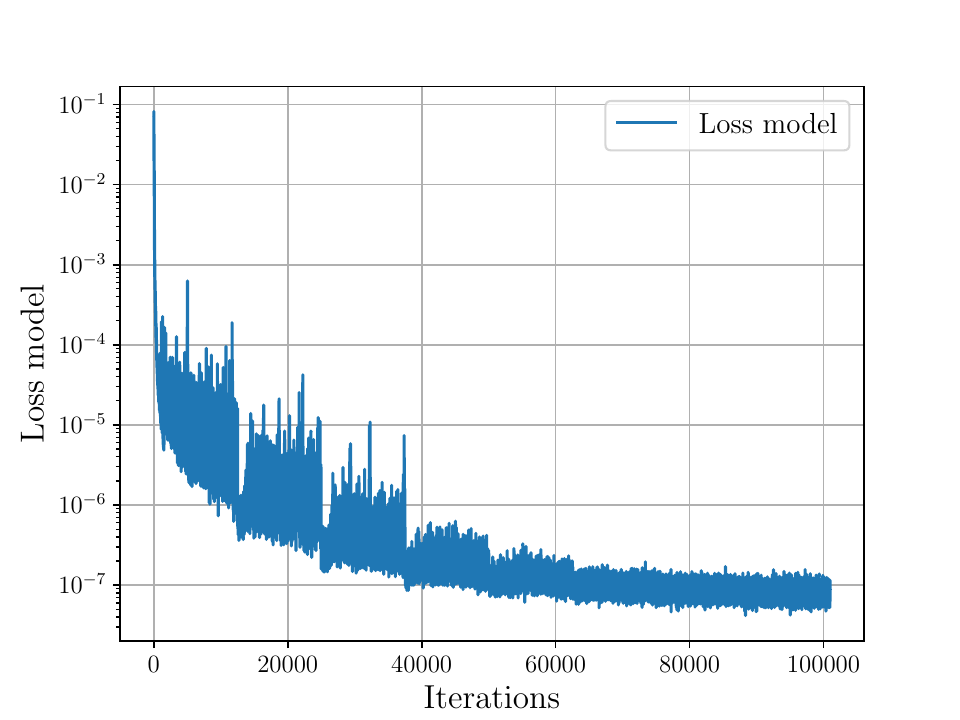}}
    \caption{Data-fit Loss and Surrogate model supervised loss on empirical measure.}
    \label{fig:elasto:2:losses}
\end{figure}
\begin{figure}[t]
    \centering
    \subfloat[$\gamma_\mu$-convergence]{\includegraphics[width=0.32\linewidth]{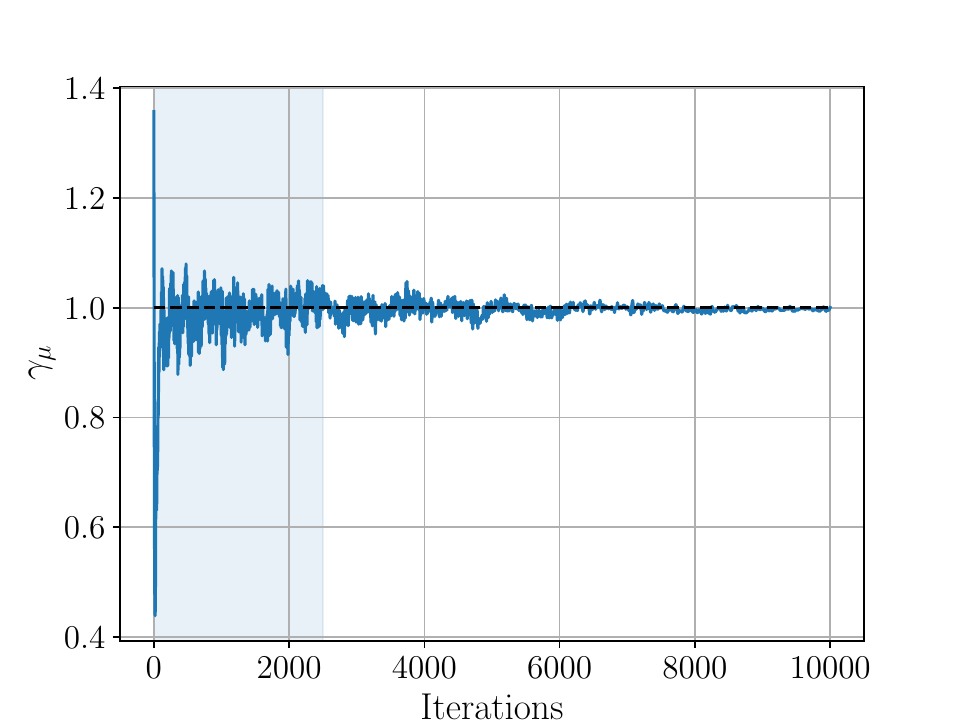}}
    \subfloat[$\ell_\mu$-convergence]{\includegraphics[width=0.32\linewidth]{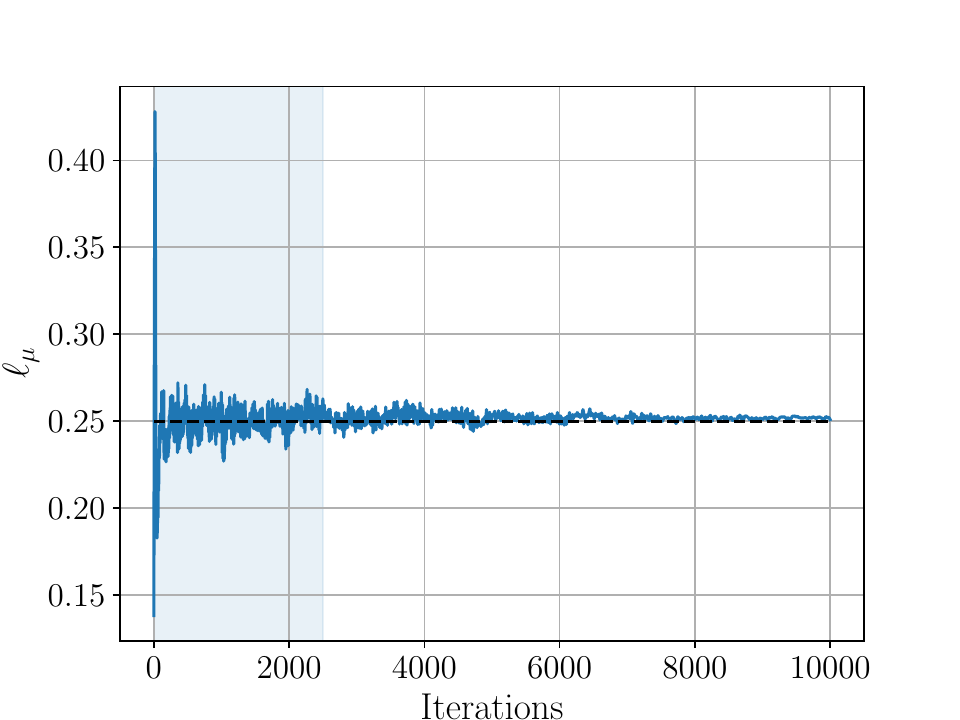}}
    \subfloat[$\gamma_\eta$-convergence]{\includegraphics[width=0.32\linewidth]{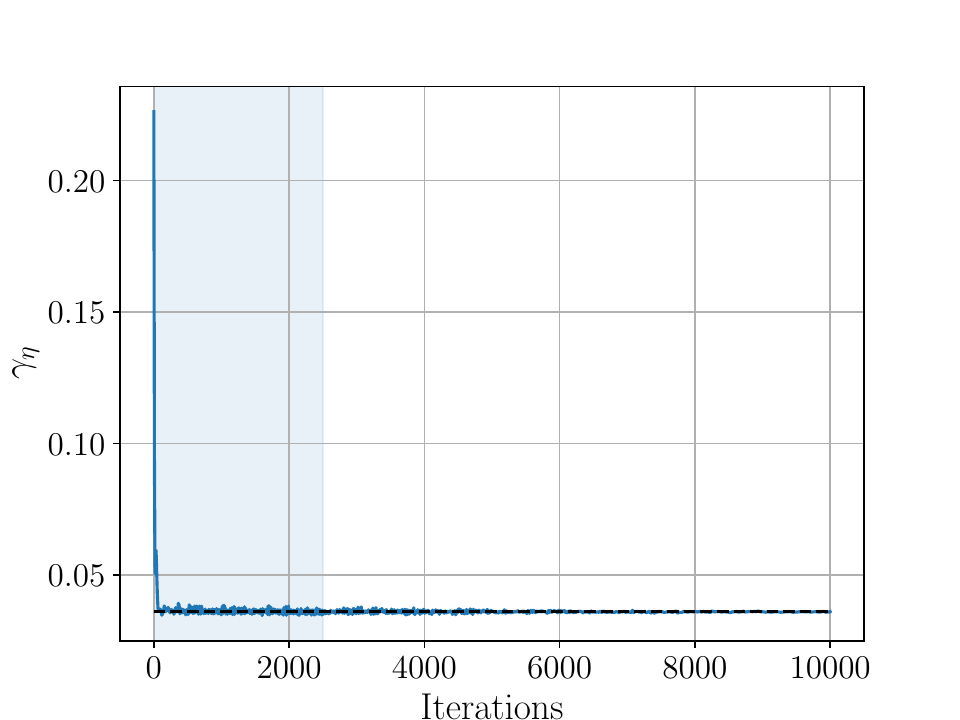}}
    \caption{Convergence of the material properties' distribution (a,b) $\alpha = (\gamma_\mu, \ell_\mu)$ (blue line), against the ground truth (black dashed line) and (c) noise standard-deviation $\gamma_\eta^2$ where the dashed line for (c) is the empirical marginal standard-deviation of the true random noise field.}
    \label{fig:elasto:2:convergence}
\end{figure}

\noindent\textbf{Data Specifics.}
We set the ground truth noise model to be correlated  $\eta^\dagger:=\cN(0, \Gamma(\beta^\dagger))$ with $\beta^\dagger=(\gamma_\eta^\dagger, \ell_\eta^\dagger, \upsilon^\dagger_\eta)$ with a Whittle-Mat\'ern field similar to~\eqref{eq:wm_kl_expansion}, but we now define the expansion for a two-dimensional field over $(x_1, \kt)$
\begin{subequations}\label{eq:wm_kl_expansion_2D_ST}
\begin{align}
    \xi^{(n)} &= \sum_{j=1}^{J+1}\sum_{k=1}^{K+1} \sqrt{\sigma\ell^d(\ell^2(j^2+k^2)\pi^2+1)^{-\upsilon-d/2}}\psi_{j,k}(x_1, \kt) \xi^{(n)}_{1,j,k},\quad \xi^{(n)}_{1,j,k}\sim \cN(0,1);\\
    \psi_{j,k}(x_1, \kt) &= \cos(j\pi x_1/l)\cos(k\pi \kt).
\end{align}
\end{subequations}
We set the true values $(\gamma_\eta^\dagger, \ell_\mu^\dagger,\upsilon^\dagger_\eta)=(0.1, 0.25, 0.5)$. We observe displacements on an $8\times8$ spatiotemporal grid.
In Figures~\ref{fig:elasto:2:obsT} and~\ref{fig:elasto:2:obsX} we show three cross-sections of the spacetime solutions at different spatial locations, at different times, along with the observations corresponding to each solution cross-section. 
Figure~\ref{fig:elasto:2:U0} shows an example solution field with the sparse observation locations.

\noindent\textbf{Learning Specifics.}
We do not learn the full correlated covariance matrix $\Gamma^\dagger$ in this example, rather we learn a misspecified noise model of the form $\Gamma(\beta)=\gamma^2_\eta\I$. We aim for the recovered quantity to correspond to the standard deviation of the spatiotemporaly marginalized signal in the data, it has a value of 0.036. This is the ground truth value against which we compare our recovered $\gamma_\eta$. The regularizer $r(\beta)$ used in~\eqref{eq:O1b} is $1/(2\cdot10^2)\|\log(\gamma_\eta)-\log(0.5)\|^2$.

\noindent\textbf{Solver Specifics.}
The body force setting the beam in motion is
\begin{align}\label{eq:elastodyn_bcic}
    f(x, \kt) =
    \begin{cases}
        (0, 25)^\top\kt,&\quad 0\leq\kt\leq \kt_f,\\
        (0, 0)^\top , &\quad \kt > \kt_f,
    \end{cases}
\end{align}
where $\kt_f=0.2\,U_\kt$ is the end-time for the action of the force -- after which the beam is in free vibration.  The forcing is positive and reaches a peak of $5\,U_f$. 

\noindent\textbf{Results.} 
In Figure~\ref{fig:elasto:2:losses}, we show the regularized data-fit~\eqref{eq:O1b} and the surrogate model loss~\eqref{eq:learn_phi}.
In Figure~\ref{fig:elasto:2:convergence} we show the convergence of the parameters $\alpha=(\gamma_\mu, \ell_\mu)$ and $\beta=(\gamma_\eta)$. The data-acquisition phase $0\leq\st<\sT_{a}$ is shown as a shaded light blue region.
The relative errors averaging over the last 100 iterations are: 0.33\% for $\gamma_\mu$, 0.72\% for $\ell_\mu$. The relative difference between $\gamma_\eta$ and the marginal amplitude of the signal is 0.20\%.

\subsection{Observing Displacements Densely}
\label{ssec:5_3}
\begin{figure}[t]
    \centering
     \subfloat[Noise random field 1]{\includegraphics[width=0.35\linewidth]{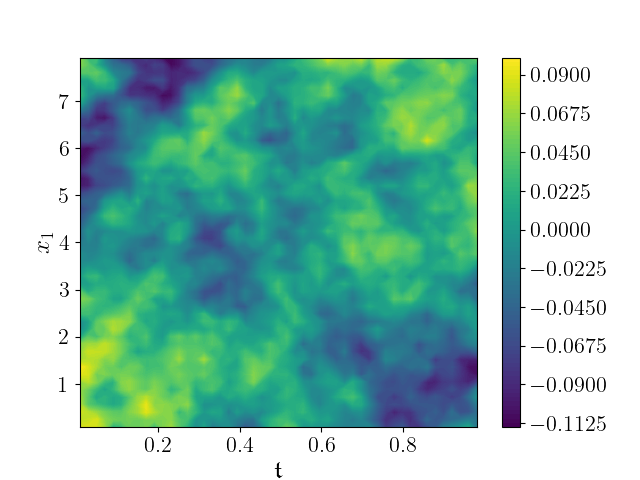}\label{fig:elasto:3:elastodyn_randfield_noise}}
    \subfloat[Data at locations]{\includegraphics[width=0.32\linewidth]{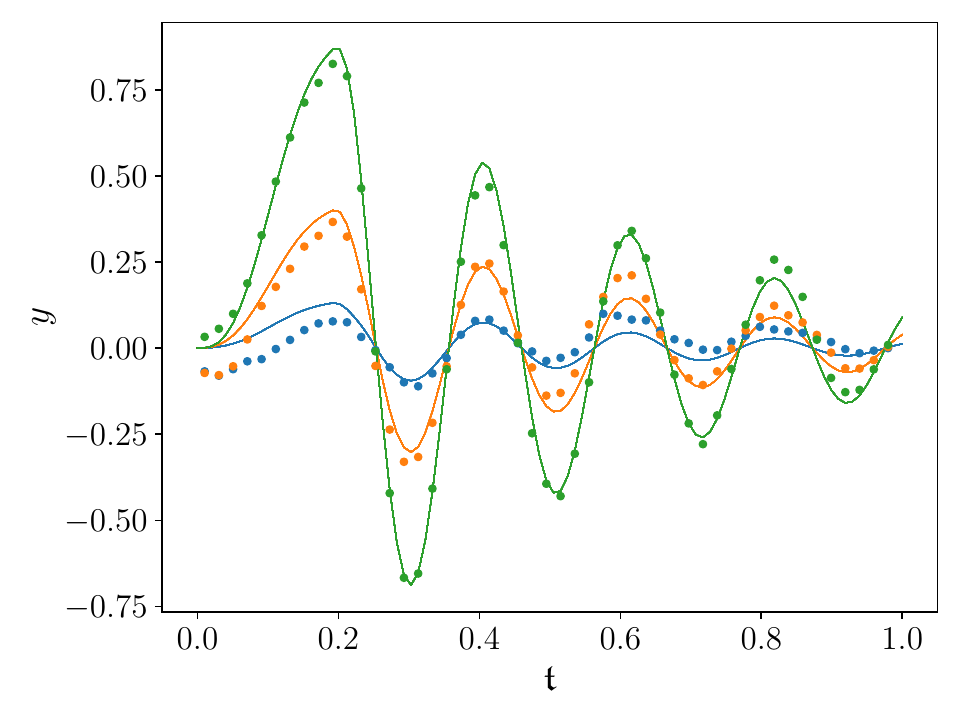} \label{fig:elasto:3:obsYT}}
    \subfloat[data  at times]{\includegraphics[width=0.32\linewidth]{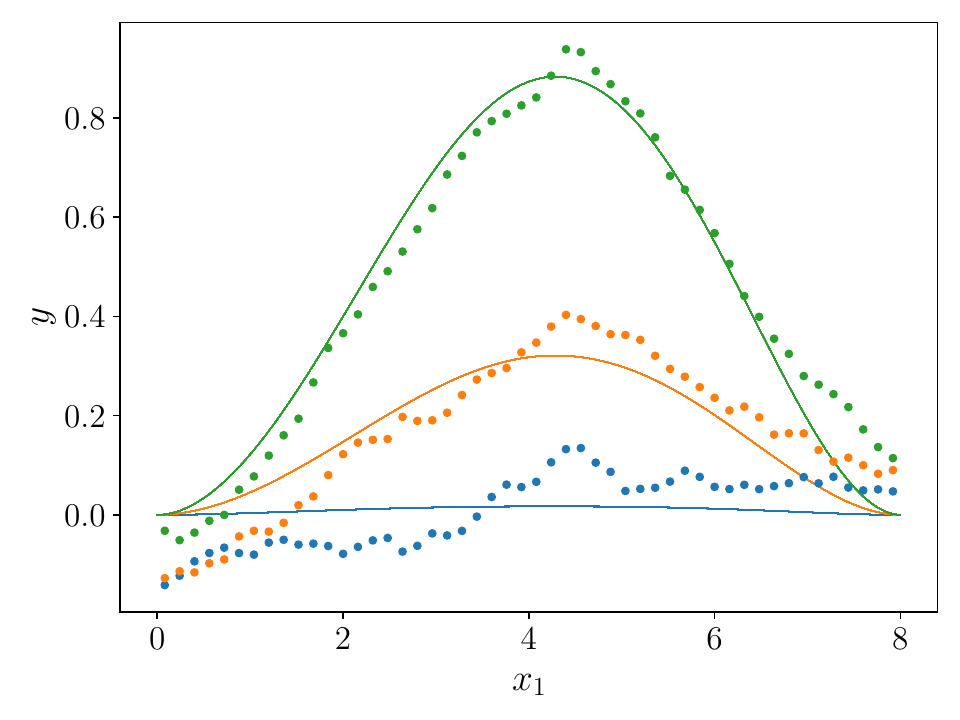}\label{fig:elasto:3:obsYX}}
    \caption{(a) Sample noise random field. (b) Solution and observational data across time at different locations. (c) Solution and observational data across beam at different time instances.}
\end{figure}
\begin{figure}[t]
    \centering
    \subfloat[Data Loss]{\includegraphics[width=0.48\linewidth]{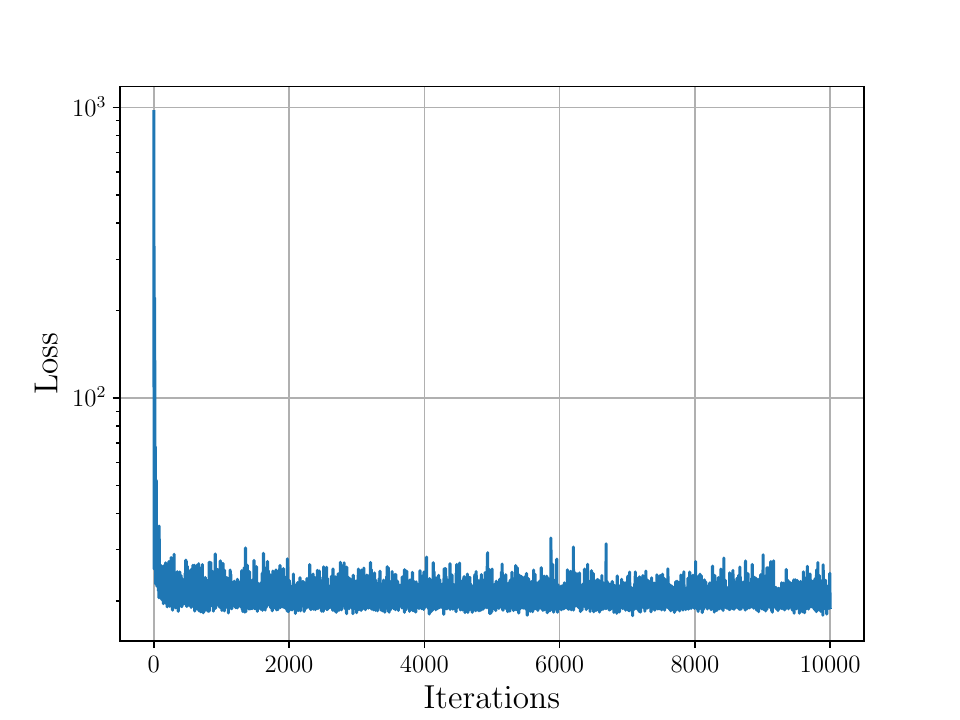}}
    \subfloat[Surrogate Loss]{\includegraphics[width=0.48\linewidth]{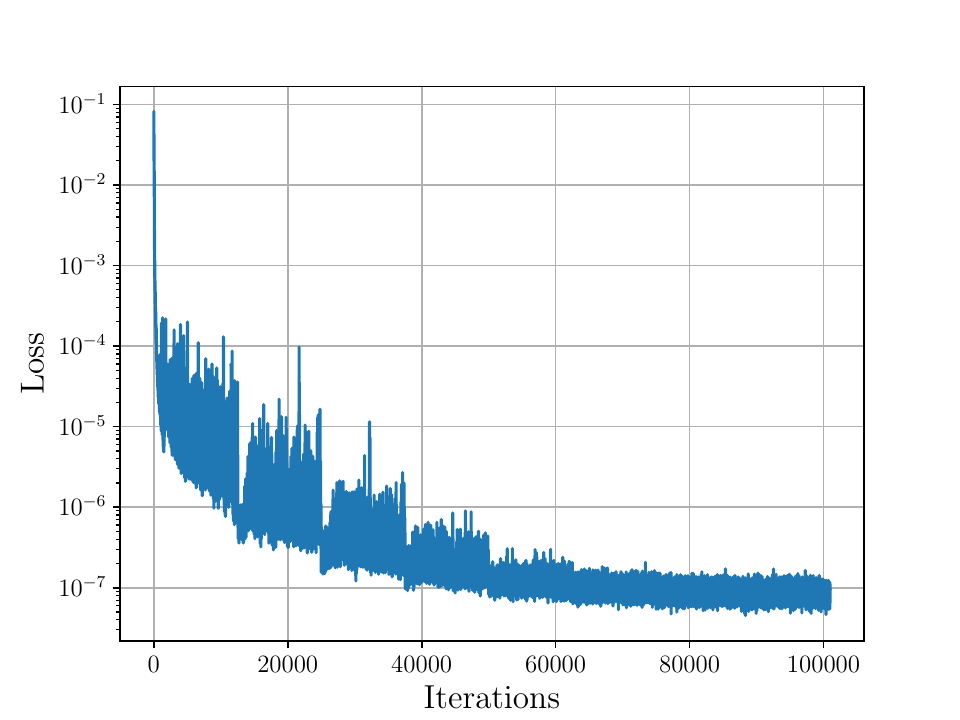}}
    \caption{Data-fit Loss and Surrogate model supervised loss on empirical measure.}
    \label{fig:elasto:3:losses}
\end{figure}
\begin{figure}[t]
    \centering
    \subfloat[$\gamma_\mu$-convergence]{\includegraphics[width=0.48\linewidth]{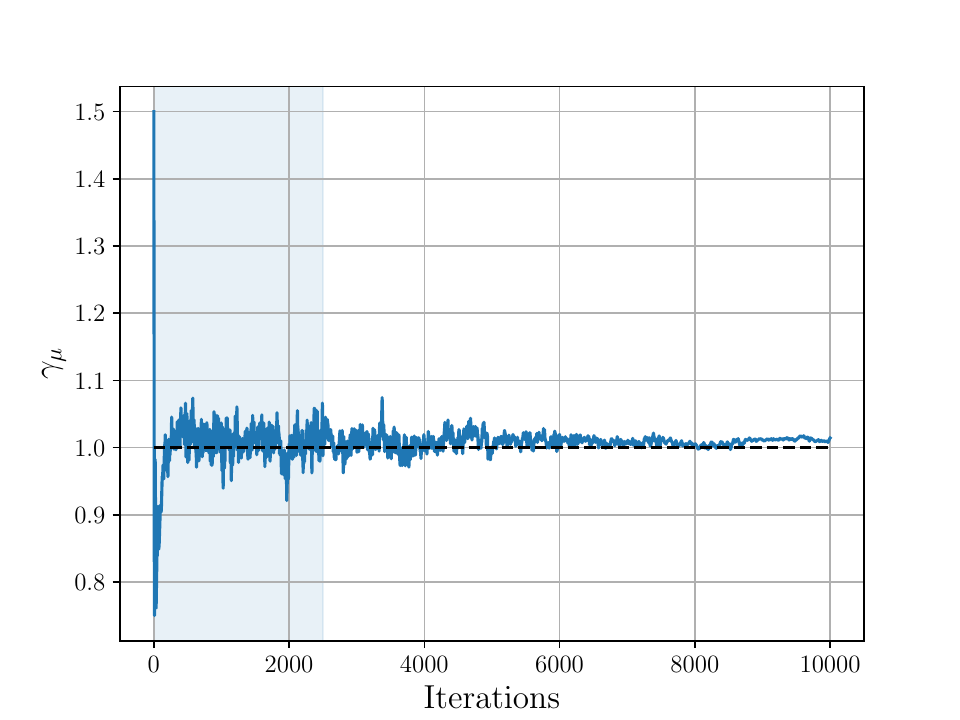}}
    \subfloat[$\ell_\mu$-convergence]{\includegraphics[width=0.48\linewidth]{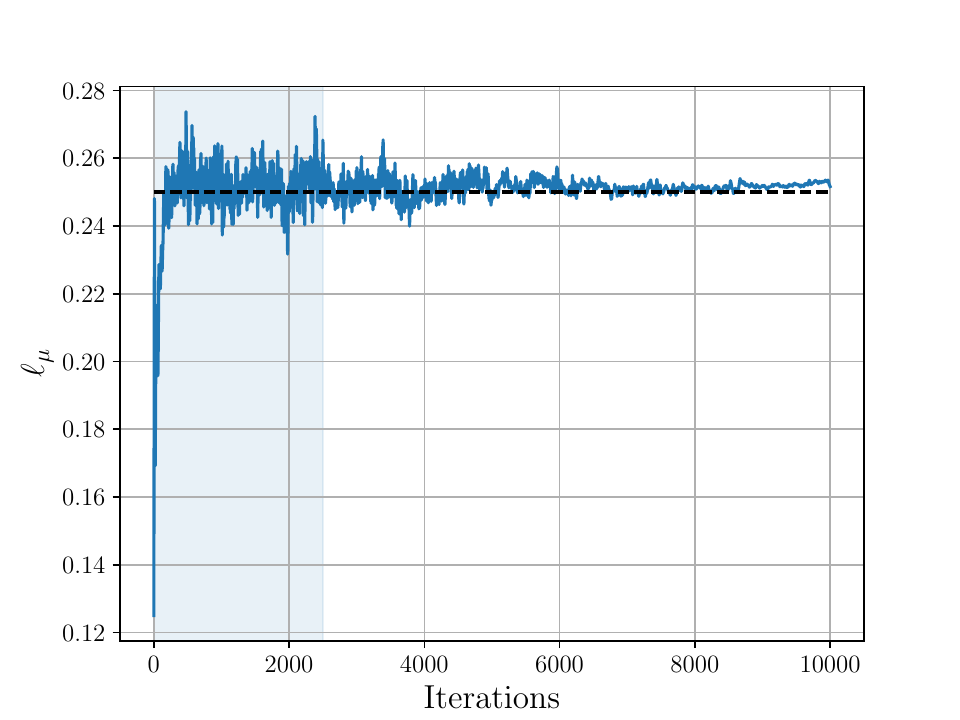}}\\
    \subfloat[$\gamma_\eta$-convergence]{\includegraphics[width=0.48\linewidth]{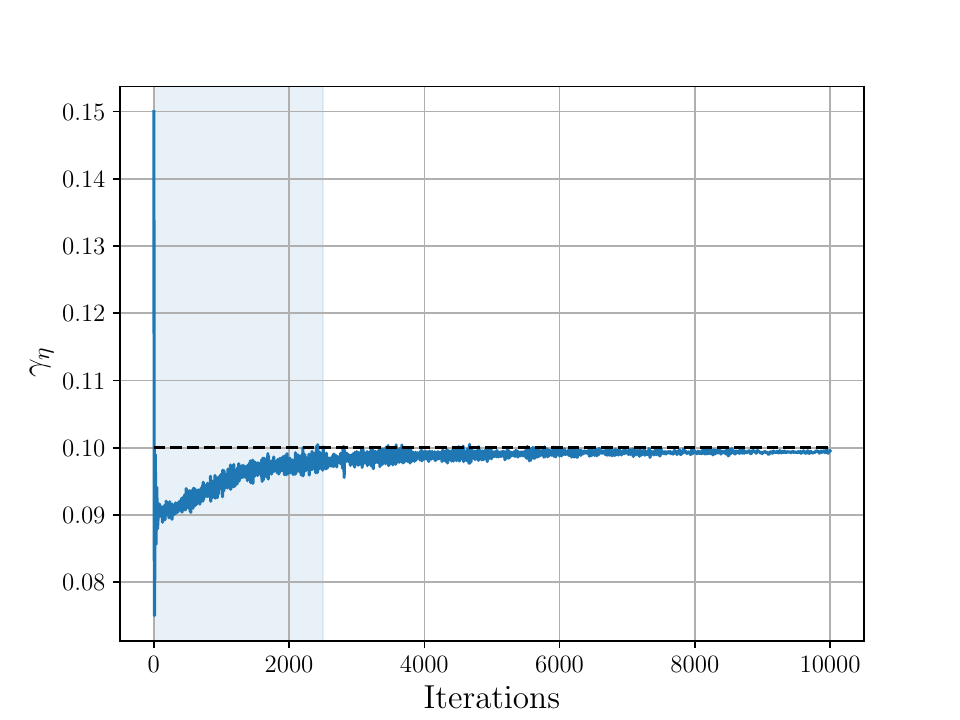}}
    \subfloat[$\ell_\eta$-convergence]{\includegraphics[width=0.48\linewidth]{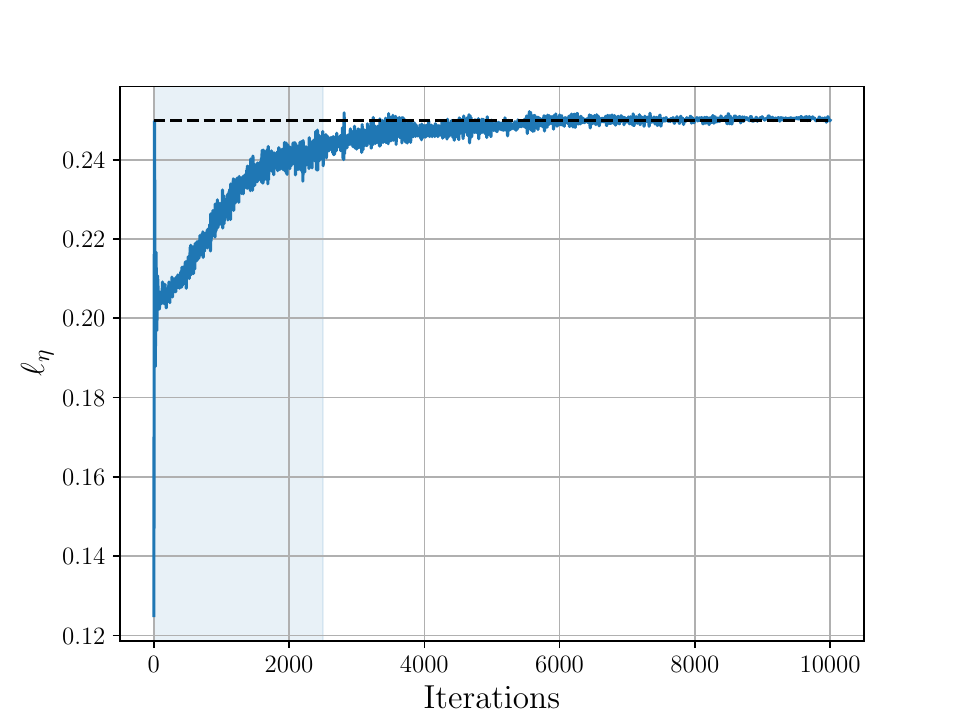}}
    \caption{Convergence of material properties measure $\alpha = (\gamma_\mu, \ell_\mu)$ and noise measure covariance $\beta= (\gamma_\eta, \ell_\eta)$. The inferred parameters are the solid blue lines, and the dashed black line is the ground truth.}
    \label{fig:elasto:3:convergence}
\end{figure}
We now look at an elastodynamic problem where noise is correlated in space and time, and we observe every other spatiotemporal solution node. 

\noindent\textbf{Data Specifics.}
In this example, we observe $50\times 49$ data-points per solution.
The noise is generated from the same Whittle-Mat\'ern field~\eqref{eq:wm_kl_expansion_2D_ST} as in Subsection~\ref{ssec:5_2} with parameters $\beta= (\gamma_\eta^\dagger, \ell_\mu^\dagger,\upsilon^\dagger_\eta)=(0.1, 0.25, 0.5)$; $\upsilon^\dagger_\eta=0.5$ is kept fixed. 
Figure~\ref{fig:elasto:3:elastodyn_randfield_noise} shows a random field generated from the noise distribution $\eta^\dagger$.
Figure~\ref{fig:elasto:3:obsYT} and~\ref{fig:elasto:3:obsYX} show the solutions and corresponding observations at various time and space cross-sections. We can see the correlation structure in the noise corruption.

\noindent\textbf{Learning Specifics.}
The $\Gamma$ regularizer in~\eqref{eq:O1b} is $r(\beta) = 1/(2\cdot10^2)\|\log(\ell_\eta, \gamma_\eta)-\log(0.1, 2.)\|^2$.

\noindent\textbf{Solver Specifics.} The body force is the same as in~\ref{ssec:5_2}.

\noindent\textbf{Results.}
Figure~\ref{fig:elasto:3:losses} shows the loss function for the data fit and the FNO loss.
In Figure~\ref{fig:elasto:3:convergence} we show the convergence of the $(\alpha, \beta)$ parameters.
The relative errors averaging over the last 100 iterations are: 0.99\% for $\gamma_\mu$, 1.19\% for $\ell_\mu$, 0.55\% for $\gamma_\eta$, and 0.07\% for $\ell_\eta$. 
We note that fewer observations can be used to successfully recover the $(\alpha, \beta)$ parameters, the current setup exemplifies the method's ability to deal with high-dimensional observation vectors ($y\in\Rb^{2450}$).

\section{Time-Averaged Data}\label{sec:l96}
In this section, we adapt the proposed methodology to study the learning of
parameters in chaotic dynamical systems using time averaged statistics of trajectories, thereby addressing Contribution~\ref{con:C6}. In this setting
we rely on the central limit theorem to define a noise error covariance, arising
from the use of finite time averaging of ergodic systems. This covariance is
not explicitly known and so learning it is of paramount importance.

Consider time-averaged data obtained, after burn-in time $\mathfrak{T_0}$, as 
\begin{subequations}
\begin{align}
    y^{(n)} &= \cG_\tau(z^{(n)}; s_0^{(n)}) = \frac{1}{\tau}\int_{\mathfrak{T_0}}^{\mathfrak{T_0}+\tau}\varphi(s^{(n)}(\kt;z^{(n)}))\md \kt,\\
    z^{(n)}&\sim\mu^\dagger,\\
    s^{(n)}_0&\sim \bP^\dagger_{s_0},
\end{align}
\end{subequations}
where $s_0^{(n)}\in\Rb^{d_s}$ denotes the initial state at time 0; $s^{(n)}(\kt; z^{(n)})\in C([0,T];\Rb^{d_s})$ is the $n^\text{th}$ system's trajectory; $\varphi:\Rb^{d_s}\rightarrow \Rb^{d_y}\times [0,T]$ is a measurement operator collecting features of $s^{(n)}$. The map $\cG_\tau$ returns time-averaged statistics of trajectories, which for $\varphi$ assembling polynomial features of $s$, returns time-averaged moments of $s^{(n)}(\kt;z^{(n)})$.

We make the standing assumption that the dynamical system generating system
trajectories is ergodic.
\begin{assumption}\label{assum:2}
We assume that time averages over
time-interval of length $\tau$ converge as $\tau \to \infty$ to a value independent
of initialization; and we assume that a central limit theorem (CLT), with randomness defined through drawing initial condition from the invariant measure of the dynamics, controls the error
at finite time. Specifically the CLT gives us
\begin{align}\label{eq:clt}
    \cG_\tau(z; s_0)& = \cG_\infty(z) + \frac{1}{\sqrt{\tau}}\cN(0, \Gamma).
\end{align}
Thus randomization of the initial condition leads to zero-mean Gaussian deviations in the time-averaged statistics of the system. 
\end{assumption}
The map $\cG_\tau:\cZ\times\cU\mapsto \bR^{d_y}$ may not be differentiable w.r.t to the model parameters in $\cZ$. We would like to construct a differentiable and initial-condition-independent surrogate model, $\cG^\phi:\cZ\rightarrow \bR^{d_y}$,
approximating $\cG_\infty$. Under Assumption~\ref{assum:2}, a suitable learning objective is then
\begin{align}\label{eq:ideal_loss_CLT_surrogate}
    \bE_{s_0\sim\bP_{s_0}}[\|\cG_\tau(z; s_0) - \cG^\phi(z)\|^2]
\end{align}
viewed as a function of $\cG^\phi$ (non-parametric) or of $\phi$ entering 
$\cG^\phi$ (parametric):
\begin{lemma}
    Under Assumption \ref{assum:2}, a minimizer of \eqref{eq:ideal_loss_CLT_surrogate} is attained at $\cG^\phi(z) = \cG_\infty(z)$.
\end{lemma}
\begin{proof}
    First differentiating w.r.t $\phi$ at a minimizer, we replace $\cG_\tau$ in expectation with $\cG_\infty$ and drop the initial condition, $s_0$, via \eqref{eq:clt} as
    \begin{align*}
         \nabla_\phi \bE_{s_0\sim\bP_{s_0}}[\|\cG_\tau(z; s_0) - \cG^\phi(z)\|^2] \nonumber
         &= -2\bE_{s_0 \sim\bP_{s_0}}[(\cG_\tau(z; s_0) - \cG^\phi(z) )^\top\nabla_\phi\cG^\phi(z)]\nonumber\\
         &=-2\bE_{\xi\sim \cN(0,\Gamma)}[(\cG_\infty(z) + \xi - \cG^\phi(z) )^\top\nabla_\phi\cG^\phi(z)].\nonumber
         \end{align*}
Setting this equal to zero to obtain a stationary point, and noting that $\Eb_\xi[\xi]=0$, we obtain
    \begin{align}
        -2(\cG_\infty(z)-\cG^\phi(z))^\top\nabla_\phi\cG^\phi(z) = 0.\nonumber
    \end{align}
    Hence $\cG^\phi(z) = \cG_\infty(z)$ is a stationary point of of~\eqref{eq:ideal_loss_CLT_surrogate}. At  $\cG^\phi(z) = \cG_\infty(z)$
    \begin{align*}
        \nabla^2_\phi \bE_{s_0\sim\bP_{s_0}}[\|\cG_\tau(z; s_0) - \cG^\phi(z)\|^2] &= \nabla_\phi( -2(\cG_\infty(z)-\cG^\phi(z))^\top\nabla_\phi\cG^\phi(z) )\nonumber\\  
        &= 2\nabla_\phi\cG^\phi(z)^\top\nabla_\phi\cG^\phi(z) - 2(\cG_\infty(z)-\cG^\phi(z))\nabla_\phi^2\cG^\phi(z)\nonumber\\
        &= 2\nabla_\phi\cG^\phi(z)^\top\nabla_\phi\cG^\phi(z)\succeq0.
    \end{align*}
     Hence $\cG^\phi(z) = \cG_\infty(z)$ is a minimizer of~\eqref{eq:ideal_loss_CLT_surrogate}.
\end{proof}
\begin{remark}\label{rem:gamma_indpt}
    Note that $\Gamma$ depends on $z$ through the dynamical system of interest. In this work, we make the assumption $\Gamma(z)\approx\Gamma$. Empirically, the proposed algorithm tends to recover $\Gamma^\star\approx\Gamma(z^\star)$ where $z^\star$ is the parameter with highest density under $\mu^\dagger$.
    \diabox0{cblack}
\end{remark}
\noindent
To obtain an implementable loss function to learn $\cG^\phi$ for various $z$ values we randomize objective function~\eqref{eq:ideal_loss_CLT_surrogate} over $z, y$ pairs from $\bP^\st_{z, y}$ in~\eqref{eq:cummul_emp_dataset} where $y^{(\st)} = \cG_\tau(z^{(\st)}; s_0^{(\st)})$ with $s_0\sim\bP_{s_0}$, a guess of $\bP_{s_0}^\dagger$. We then obtain the surrogate parameters 
\begin{align}
    \label{eq:Gphi_loss}
    \phi^\star_\st = \argmin\;\Eb_{z,y\sim \bP^\st_{z,y}}[\|y - \cG^\phi(z)\|^2].
\end{align}

We now turn our attention to the deconvolution and populational inversion framework. To formulate the problem at hand in a form amenable to~\eqref{eq:O1b}
we make a series of approximations in the following steps:
\begin{align}
    \label{expr:E1}\tag{{E1}}  &(\cG_\tau(z; s_0))_\#(\bP_z \otimes \bP_{s_0})\nonumber;\\
    \label{expr:E2}\tag{{E2}}   &(\cG_\infty(z) + \Gamma(z)^{\frac{1}{2}}\xi)_\#(\bP_z \otimes \cN(\xi; 0,\I))\nonumber ;\\
    \label{expr:E3}\tag{{E3}}    &\cN(0, \Gamma) * (\cG_\infty(z))_\#\bP_z\nonumber .
\end{align}
Expression~\eqref{expr:E1} describes the exact data-generating measure for $\bP_z, \bP_{s_0}$. Going from~\eqref{expr:E1} to~\eqref{expr:E2} involves Assumption~\ref{assum:2}. Going from~\eqref{expr:E2} to~\eqref{expr:E3} uses Remark~\ref{rem:gamma_indpt}, $\Gamma(z)\approx\Gamma$, thereby allowing us to factorize $\Gamma$ and learn it as a single matrix. 
Replacing $\bP_z$ in~\eqref{expr:E3} with $\mu(\alpha)$ and $\cG_\infty$ with $\cG^{\phi^\star_\st}$, with ${\phi^\star_\st}$ concurrently learned with~\eqref{eq:Gphi_loss} as per Algorithm~\ref{alg:algorithm2}, we can write the objective
\begin{align}
    \sL(\alpha, \Gamma;\Gamma') &= \dd_{{\D(\Gamma')}}\left(\nu,  \cN(0, \Gamma) * (\cG^{\phi^\star_\st}(z))_\#\mu(\alpha)\right) + h(\alpha) + r(\Gamma).
\end{align}
Every update on $\Gamma$ we set the precondition matrix $\D(\Gamma')$ to the diagonal matrix, $\D$, of the LDL decomposition, $L_d\D L_d^\top=\Gamma'$. As, algorithmically, $\Gamma$ is expressed in terms of it's Cholesky factor $L$, $\D^{-1/2}$ is efficiently computed as the reciprocal of the diagonal of $L$. We have found this to empirically be a good preconditioner. The term $r(\Gamma)$ will be specified in Section~\ref{sssec:cov_reg}. We again use $\dd_\B$ to be $\sSW^2_{2, \B}$ as defined by~\eqref{eq:weighted_sw}.

\subsection{Notes on Regularization}\label{ssec:LorenzReg}
Special care is taken in time-averaged inference with regards to regularization in terms of covariances and surrogate models.
\subsubsection{Covariance Regularization}\label{sssec:cov_reg}
In both experiments that follow we attempt to learn $\Gamma$ as an approximation of $\Gamma^\dagger(z)$. We find the true empirical covariance to be very poorly conditioned -- leading to learning difficulties and numerical instabilities. An effective solution implemented in both experiments is to regularize the condition number of the estimated $\Gamma$~\cite{won2013condition, dey1985estimation, james1992estimation}. We do this by using the covariance regularizer in~\eqref{eq:O1b} to be $r(\Gamma) = \epsilon\kappa_2(\Gamma)$ where $\kappa_p(\Gamma) := \|\Gamma\|_p\|\Gamma^{-1}\|_p$ is the condition number of $\Gamma$ which we compute in the 2-norm by the ratio of largest and smallest eigenvalues of $\Gamma$, a more stable computation than that which arises from the general $p$-norm definition. We find this value for $\epsilon$ minimally interferes with inference whilst ensuring stability of the learning of $\Gamma$. For consistency, in both examples we fix $\epsilon=10^{-5}$. This regularizer mirrors using a regularizer on $\beta$ in the previous sections. We note, for future work, condition numbers under other norms could be used, such as $p=\infty$ or $p=1$.

\subsubsection{Neural Network Surrogate Model Regularization}\label{sssec:nn_reg}
In both the following examples, we make use of a Lipschitz constrained MLP for $\cG_\tau$ to enforce smoothness in the predictions~\cite{gouk2021regularisation}. This has been found to be useful due to the high stochasticity of the data-pairs  as we randomize over initial conditions. This is similar to using a long lengthscale in GP regression to recover the mean of the outputs. Each layer of a 5-layer, Gelu activation, 100 neuron network is constrained to have a Lipschitz number under 10 in the $p=\infty$ operator norm. 

\subsection{Lorenz 96 Single-Scale}
\begin{figure}[t]
    \centering
    \subfloat[Full 100 time units]{\includegraphics[width=0.44\linewidth]{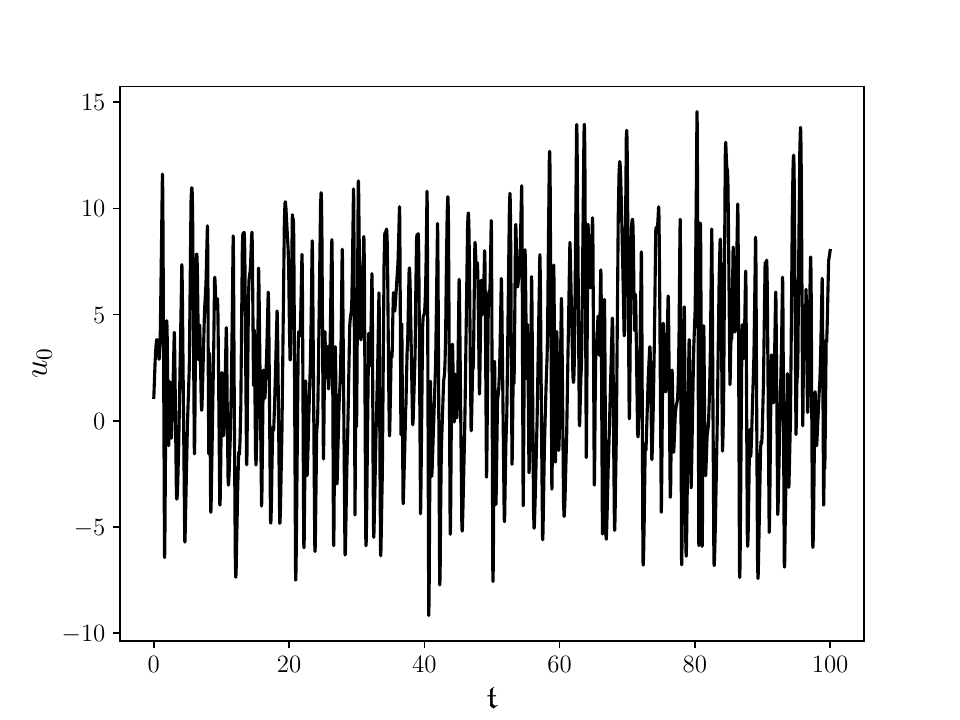}}
    \subfloat[First 20 time units]{\includegraphics[width=0.44\linewidth]{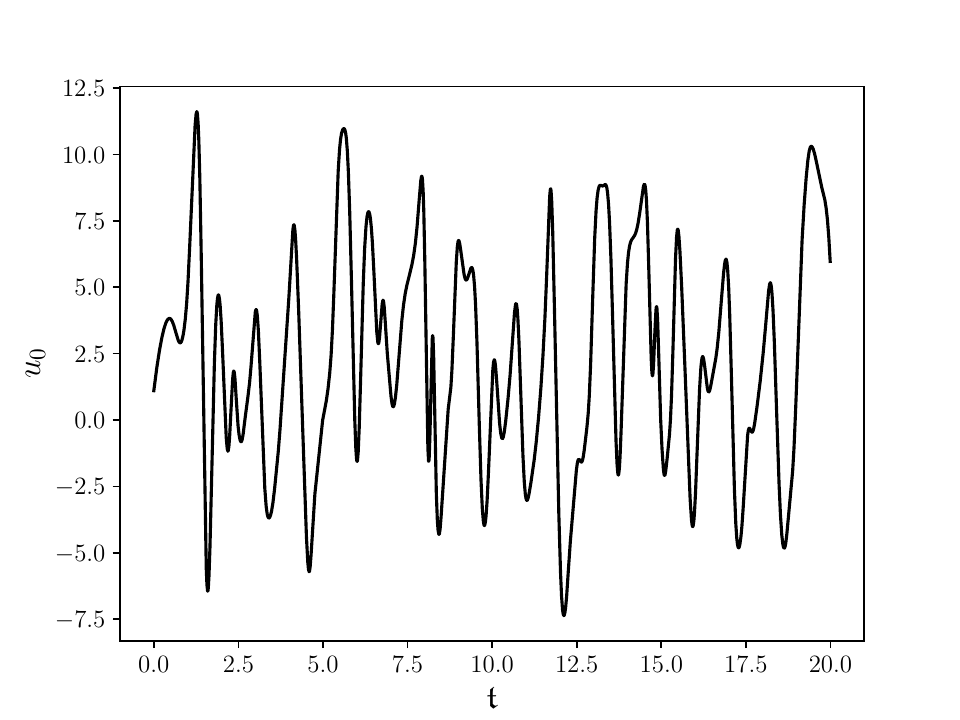}}
    \caption{The first dimension of the Lorenz 96 single scale model for (a) 100 time unites, and (b) the first 20 time units.}
    \label{fig:trajectoryLorenz96ss}
\end{figure}
\begin{figure}[t]
    \centering
    \subfloat[Data Loss]{\includegraphics[width=0.32\linewidth]{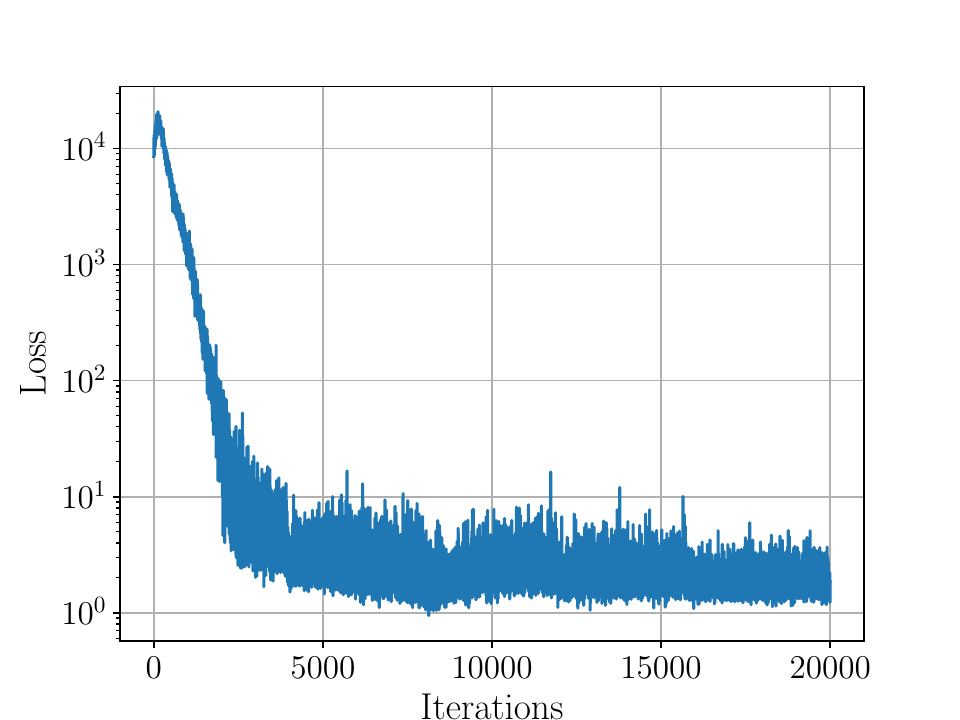}}
    \subfloat[Model Loss]{\includegraphics[width=0.32\linewidth]{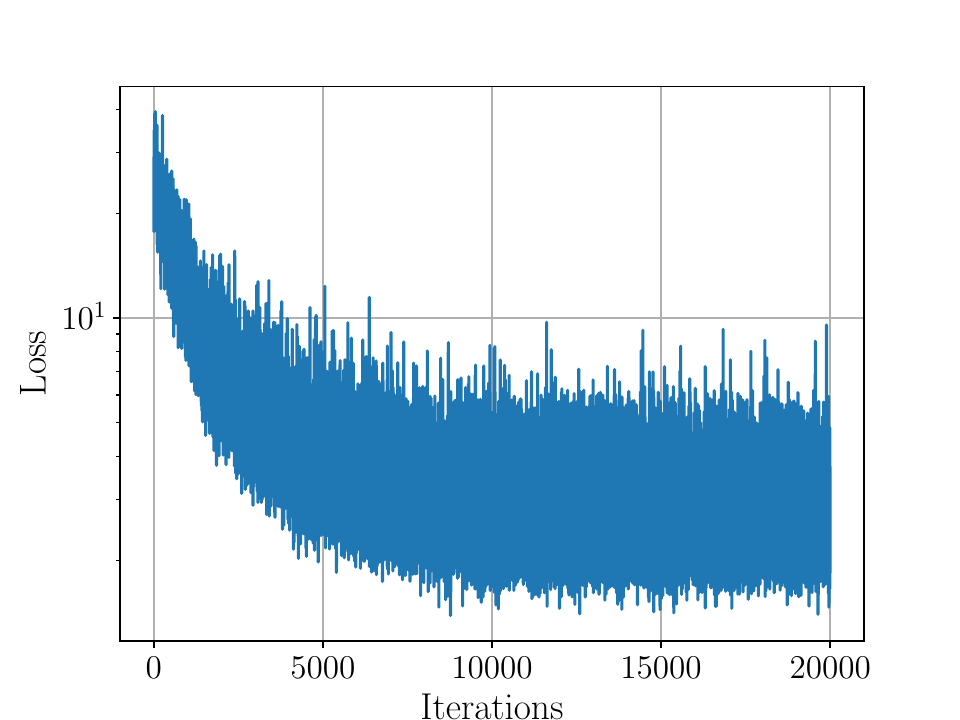}}
    \subfloat[$\log \kappa_2(\Gamma)$]{\includegraphics[width=0.32\linewidth]{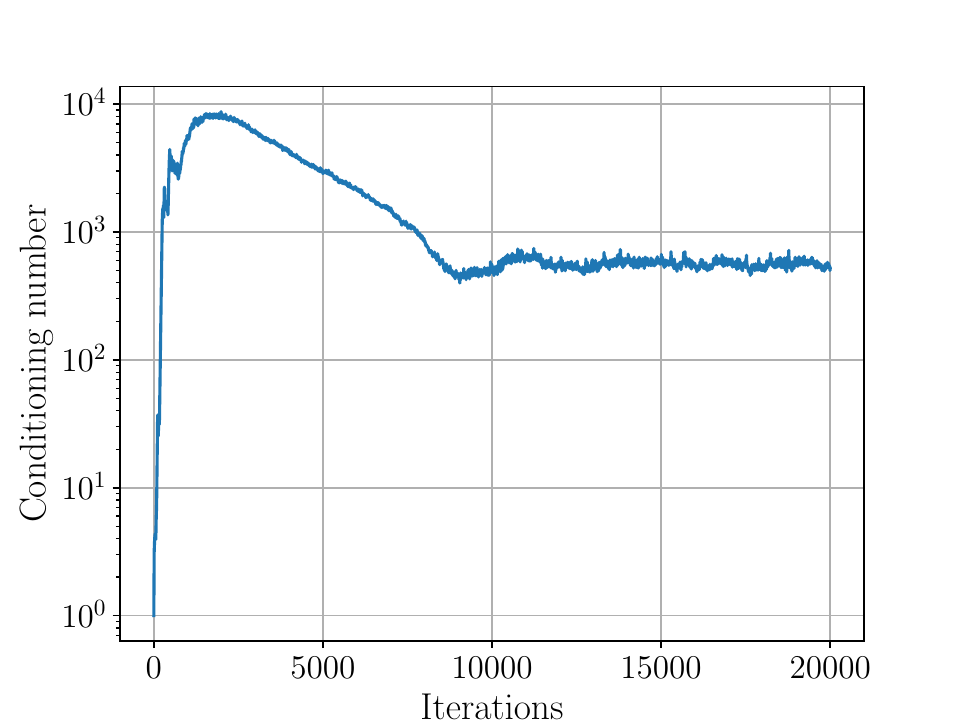}}
    \caption{(a) Data Loss function, (b) model $\cG^\phi$ loss function. (b) the evolution of the condition number of the estimated covariance.}
    \label{fig:lossLorenz96ss}
\end{figure}
\begin{figure}[t]
    \centering
    {\includegraphics[width=0.5\linewidth]{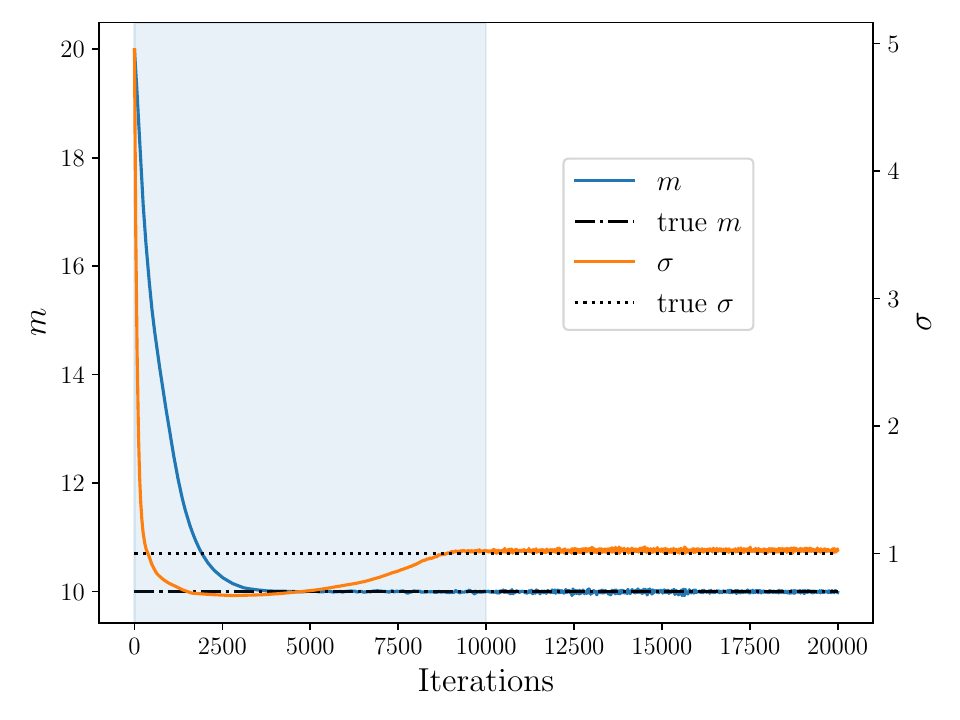}}
    \caption{Convergence of parameters $\alpha$}
    \label{fig:ConvergenceLorenz96ss}
\end{figure}
\begin{figure}[t]
    \centering
    \subfloat[$\Gamma^\star$]{\includegraphics[width=0.5\linewidth]{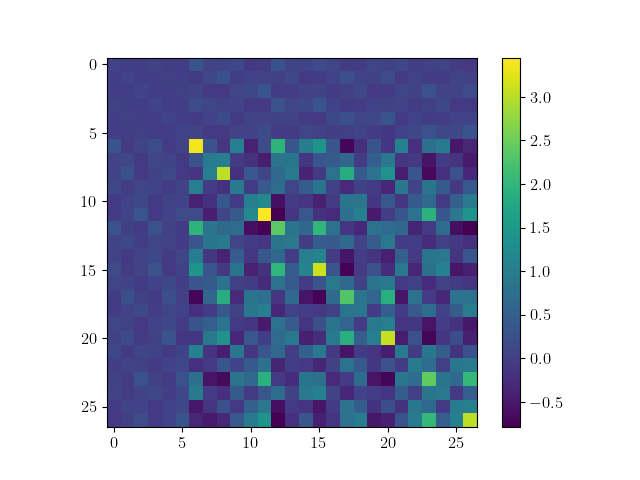}}
    \hspace{-3em}
    \subfloat[  $\mathrm{Cov}_{s_0\sim\bP^\dagger_{s_0}} \mathrm{[}\cG_\tau(m^\dagger; s_0)\mathrm{]}$ ]{\includegraphics[width=0.5\linewidth]{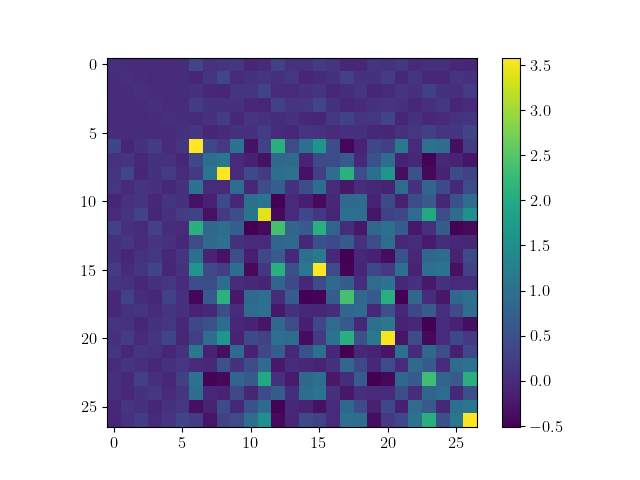}}
    \caption{(a) The learned $\Gamma^\star$. (b) Empirical covariance of the feature function for random initializations of trajectories for $m=10$ (the mode of $\mu^\dagger$).}
    \label{fig:CovsLorenz96ss}
\end{figure}
We first assess the proposed deconvolution and populational time-averaged inversion methodology on the single-scale Lorenz 96~\cite{lorenz1996predictability} model,
used as a simplified representation of chaotic atmospheric dynamics:
\begin{subequations}
    \begin{align}
    \dot{u}_k &= u_{k-1}(u_{k+1} - u_{k-2}) - u_k + \F,\\
    u_{k+K}&=u_k,\quad k=1 \hdots K
\end{align}
\end{subequations}
where $s(t)=u_{1:k}(t)$  and the $\mu(\alpha)$ distribution is over the parameter $z=(\F)$. 

\noindent\textbf{Data Specifics.}
The feature function 
\begin{align}
    \varphi(s) &=\left(\{ u_k\}_{k=1}^K , \{u_ku_j | k,j\in\mathbb{N}, 1\leq k \leq K, 1\leq j\leq k \}\right),
\end{align}
collects per-dimension means and covariances. We set $K=6$, hence $\varphi:\Rb^{6}\mapsto\Rb^{27}$; we fix $c=10$,  $\tau=100$ with 20 units of time burn-in for $\cG_\tau$. 
and set $\tau=100$ with 20 units of time burn-in for $\cG_\tau$.
The true parameter distribution is  $\F^{(n)}\sim\cN(z; m^\dagger, (\sigma^\dagger)^2)$ with $m^\dagger=10$, $\sigma^\dagger=1$ with $\alpha^\dagger=(m^\dagger, \sigma^\dagger)$. 
The data is a collection of $10^4$ trajectories with initial state distribution $\Pb^\dagger_{s_0} := \cN(0, 10^2\I)$.
In Figure~\ref{fig:trajectoryLorenz96ss} we  plot the first dimension of a sample trajectory. 

\noindent\textbf{Learning Specifics.}
The covariance matrix $\Gamma\in\Rb^{27\times 27}$ is parameterized as a lower triangular matrix with 378 learnable parameters constrained to be positive on the diagonal.  
The  trajectories in the data set are generated with the initial state distribution $\Pb^\dagger_{s_0} := \cN(0, 10^2\I)$ which during inference is replaced with the miss-specified initial state measure $\bP_{s_0} = \cN(0, 8^2\I)$ to demonstrate the ``washing out'' of the initial conditions, and hence does need not be inferred.
The differentiable initial-condition-independent surrogate model $\cG^\phi$ is parameterized as a fully connected Lipschitz constrained neural network (Section~\ref{sssec:nn_reg}) and the covariance regularizer is specified in~\ref{ssec:LorenzReg}.
The regularizer $h(m, \sigma)  = {1}/({2\cdot 5^2})\|m-8\|^2 + {1}/({2 \cdot 2^2})\|\log(\sigma) - \log(0.5)\|^2$. 
We use a learning batch size of 60 for $\cG^{\phi^\star_\st}$. We use a learning rate of $10^{-2}$ on $(\alpha, \Gamma)$ and a 5-times decayed-by-half Adam learning rate initialized at $10^{-3}$ for learning $\phi$.
We use a pretraining number of iterations of $\sT_\mathrm{pre.} = 10^3$ with $N_\mathrm{pre.}=60$, and $\sT_\mathrm{inner} = 20$ step on $\phi$ minimizing $\eqref{eq:learn_phi}$ per step on $(\alpha, \Gamma)$. We terminate the training data-pairs acquisition after $\sT_a = 10$k steps.

\noindent\textbf{Solver Specifics.}
We use a time step of $10^{-2}U_\kt$ with a fourth order Runge-Kutta scheme.

\noindent\textbf{Results.}
Figure~\ref{fig:lossLorenz96ss} shows the data fit loss, surrogate model loss and the condition number of $\Gamma$. Figure~\ref{fig:ConvergenceLorenz96ss} shows the convergence of the $\alpha$ parameters. In Figure~\ref{fig:CovsLorenz96ss} we show the learned covariance matrix and the covariance matrix of the observations at the mode of the true parameter distribution for randomized initial conditions from $\bP^\dagger_{s_0}$. 
The average relative errors from the last 1000 iterations are: 
0.01\% for $m$, 2.64\% 
for $\sigma$. The relative difference between $\Gamma^\star$ and $\mathrm{Cov}_{s_0\sim\bP^\dagger_{s_0}}[\cG_\tau(m^\dagger;s_0)]$ is of 18.22\% in the Frobenius norm.

\subsection{Lorenz 96 Multi-Scale}
\begin{figure}[t]
    \centering
    \subfloat[Full 100 time units]{\includegraphics[width=0.44\linewidth]{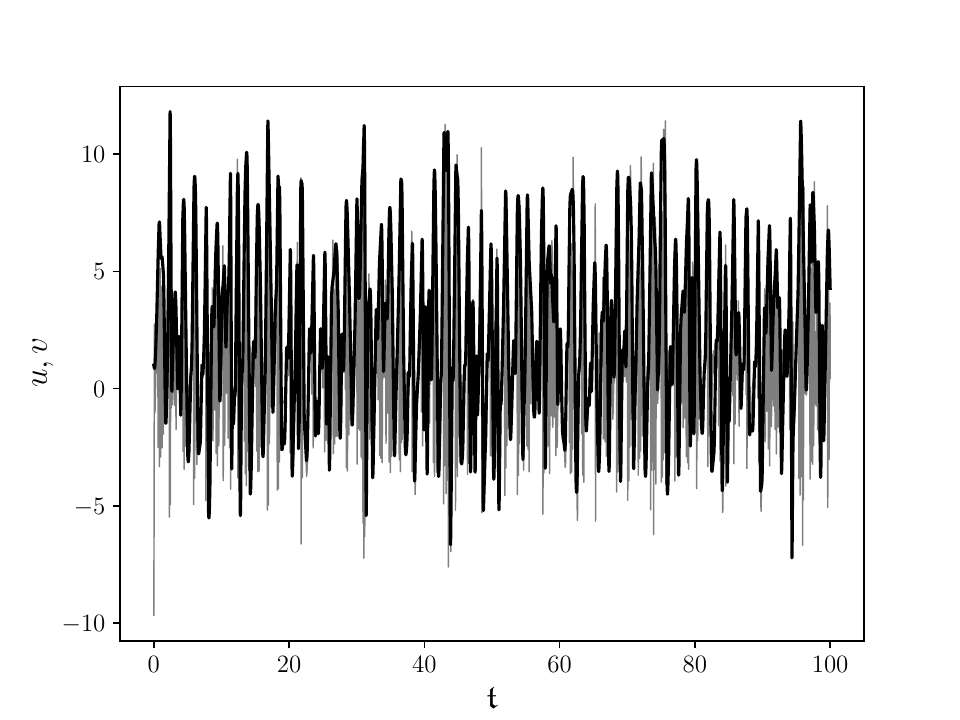}}
    \subfloat[First 20 time units]{\includegraphics[width=0.44\linewidth]{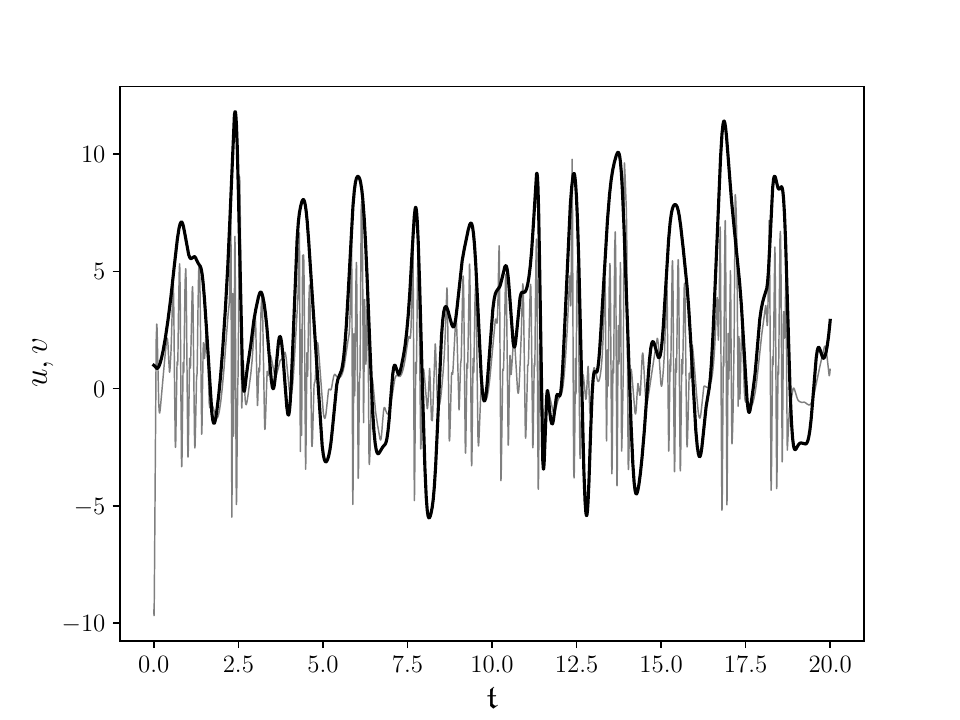}}
    \caption{100 (a) and 20  (b) time unites of the first dimension of Lorenz 96 multi-scale showing both the slow and fast scale trajectories of $u_k, v_{k,j}$ for $k=j=1$.}
    \label{fig:trajectoryLorenz96ms}
\end{figure}
\begin{figure}[t]
    \centering
    \subfloat[Data Loss]{\includegraphics[width=0.32\linewidth]{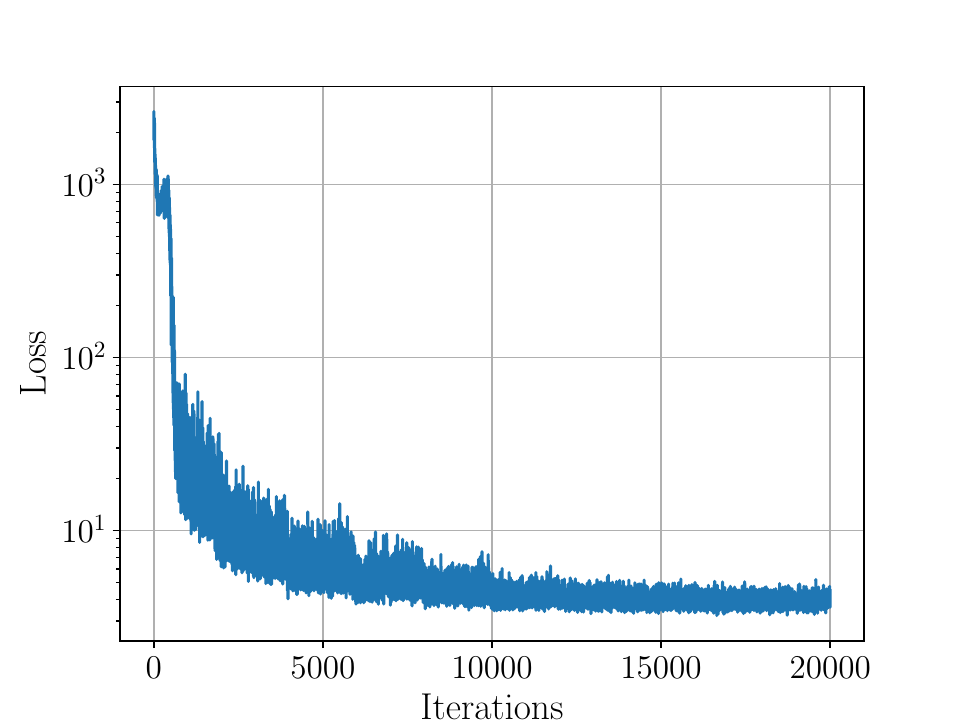}}
     \subfloat[Model Loss]{\includegraphics[width=0.32\linewidth]{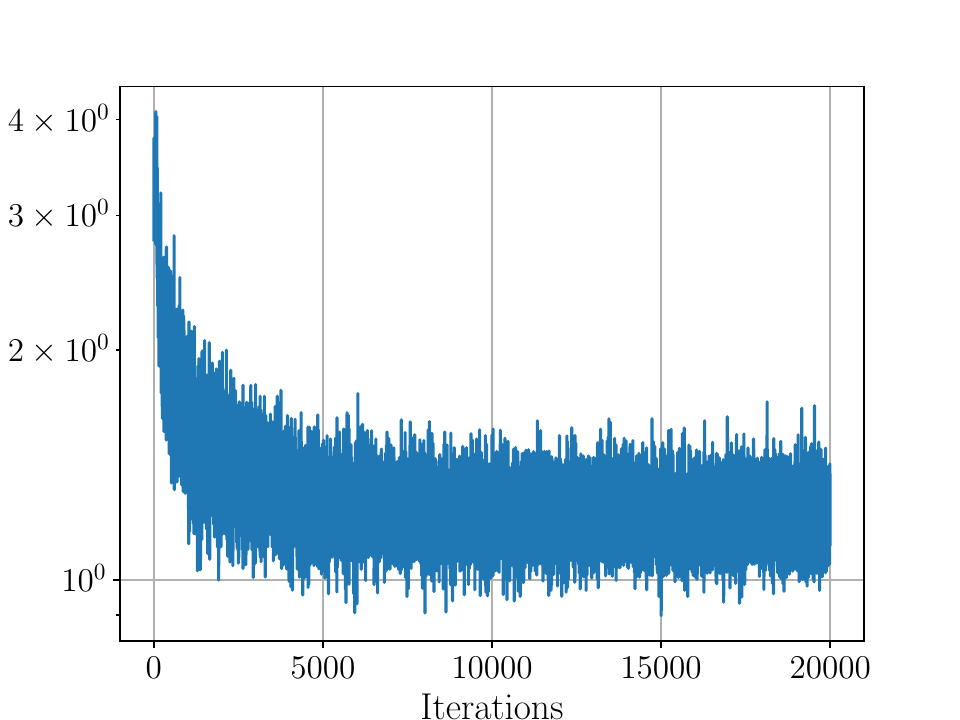}}
    \subfloat[$\log \kappa_2(\Gamma)$]{\includegraphics[width=0.32\linewidth]{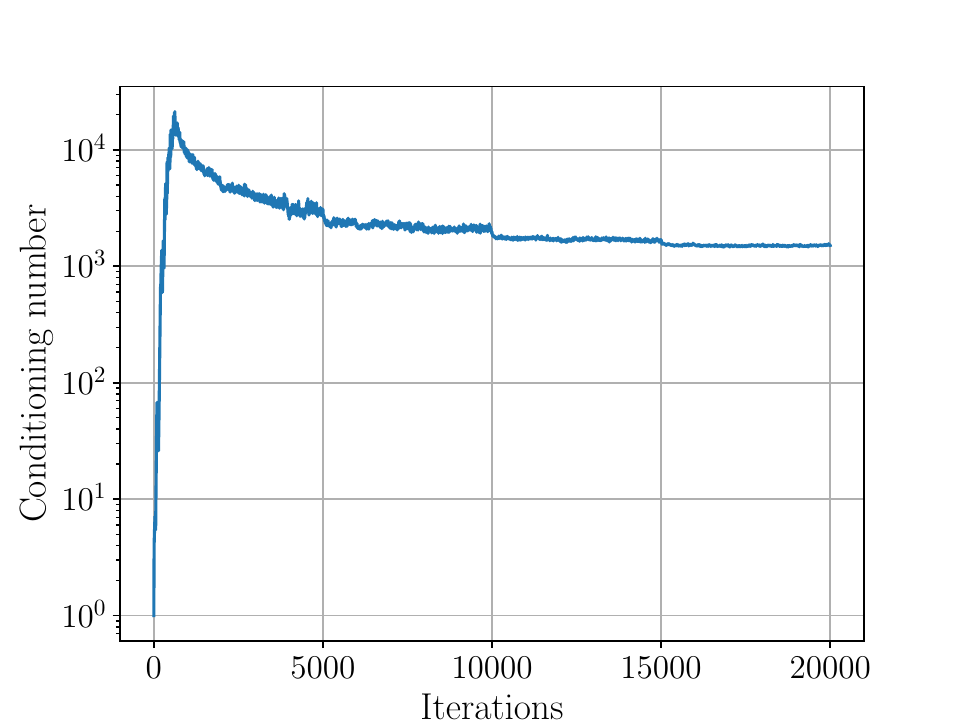}}
    \caption{(a) Data Loss function and the model $\cG^\phi$ loss function. (b) the evolution of the condition number of the estimated covariance.}
    \label{fig:lossLorenz96ms}
\end{figure}
\begin{figure}[t]
    \centering
    \subfloat[$\F$ v $\h$]{\includegraphics[width=0.32\linewidth]{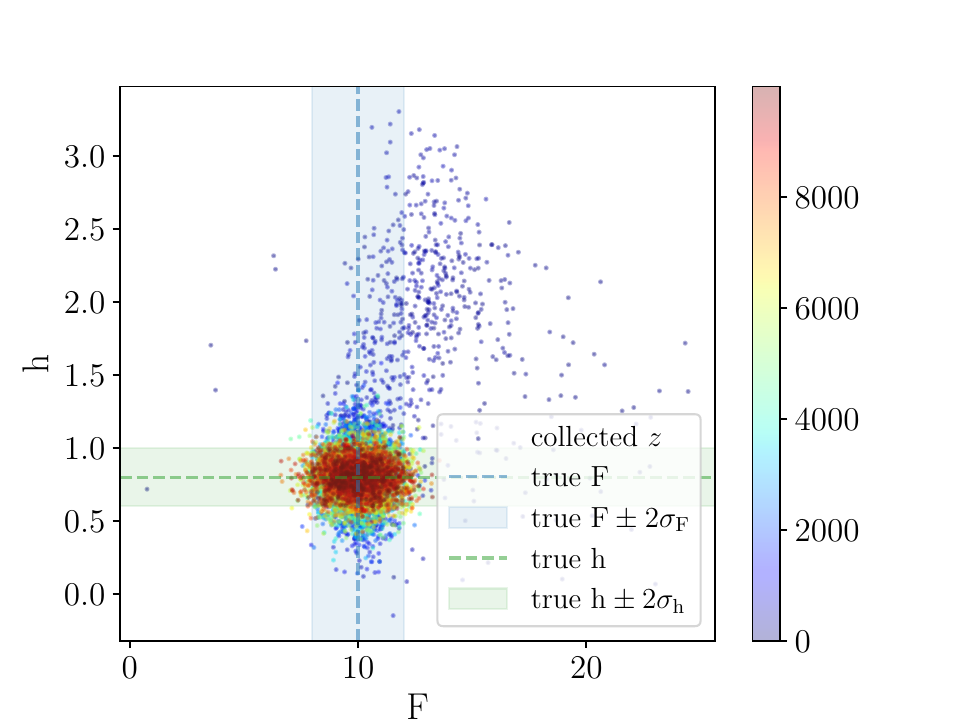}}
    \subfloat[$\h$ v $\b$]{\includegraphics[width=0.32\linewidth]{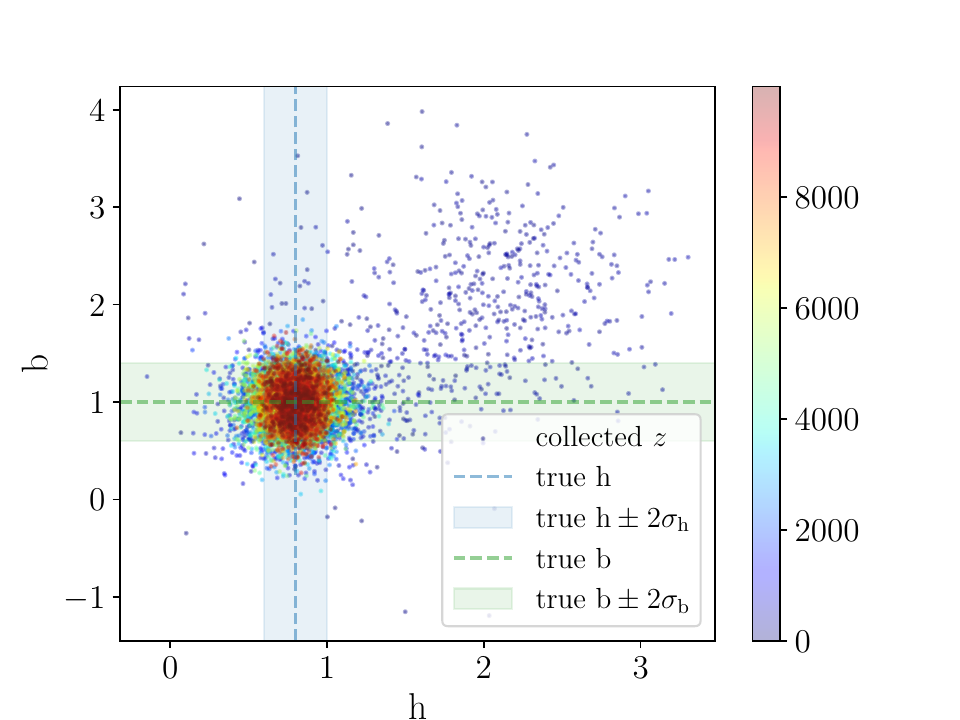}}
    \subfloat[$\b$ v $\F$]{\includegraphics[width=0.32\linewidth]{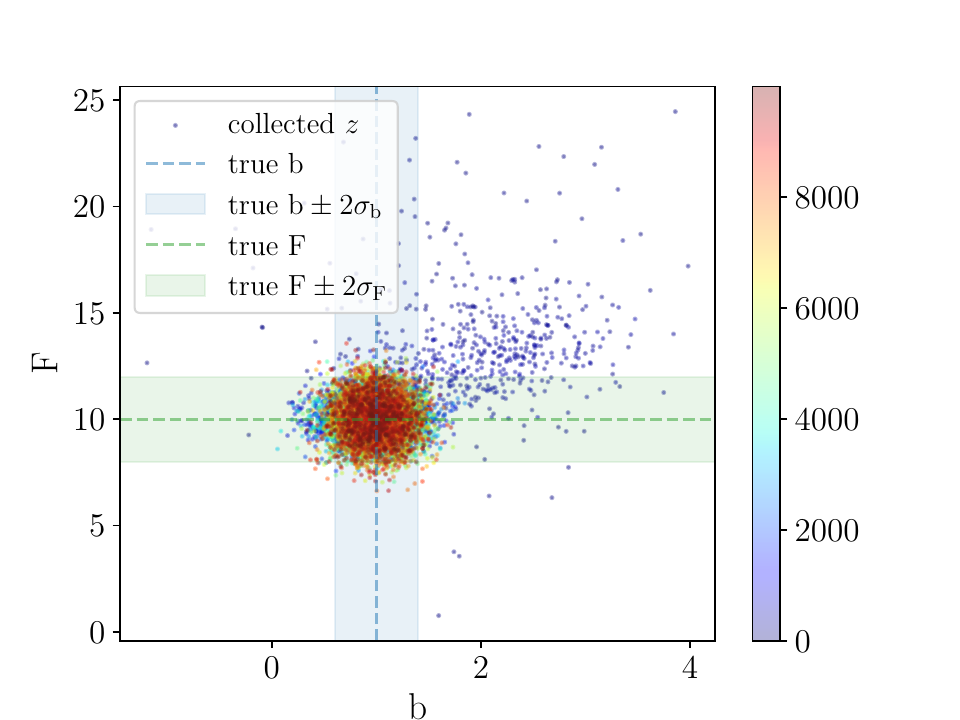}}
    \caption{The data sample colour indicates the sampling index $\st$ of $ z^{(\st)} =(F^{(\st)}, h^{(\st)}, b^{(\st)})\sim\mu(\alpha_\st)$ as it goes into $\bP^\st_{y,z}$. Hence, as the colorbs progress upwards on the color scale, the more recent the selected sample $z^{(\st)}$.}
    \label{fig:data_aquisition_l96ms}
\end{figure}
\begin{figure}[t]
    \centering
    \subfloat[Convergence $\F$]{\includegraphics[width=0.32\linewidth]{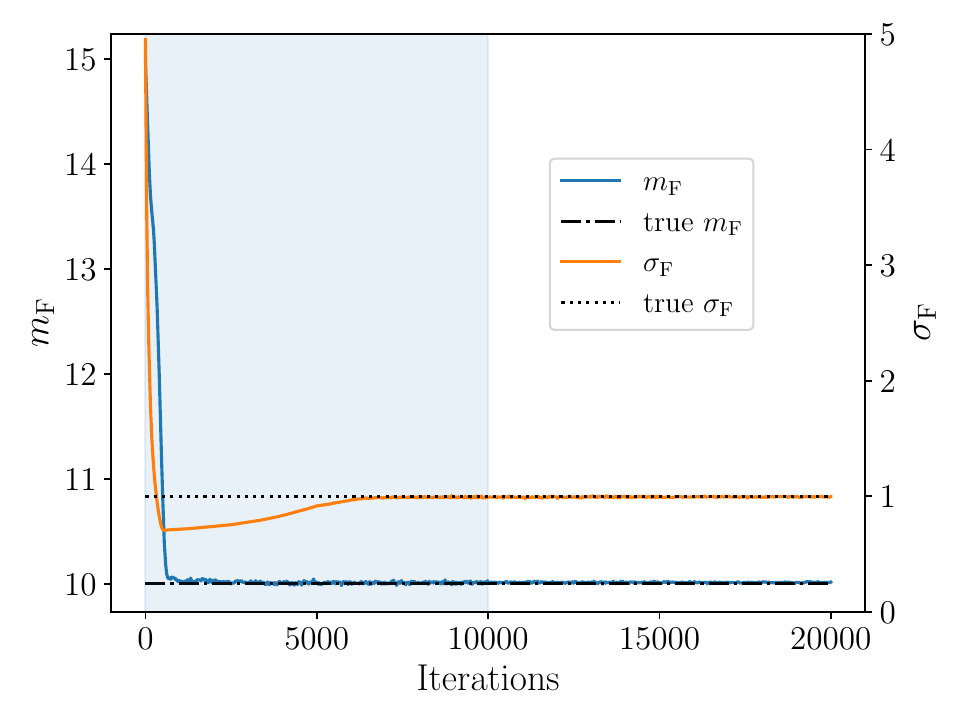}}
    \subfloat[Convergence $\h$]{\includegraphics[width=0.32\linewidth]{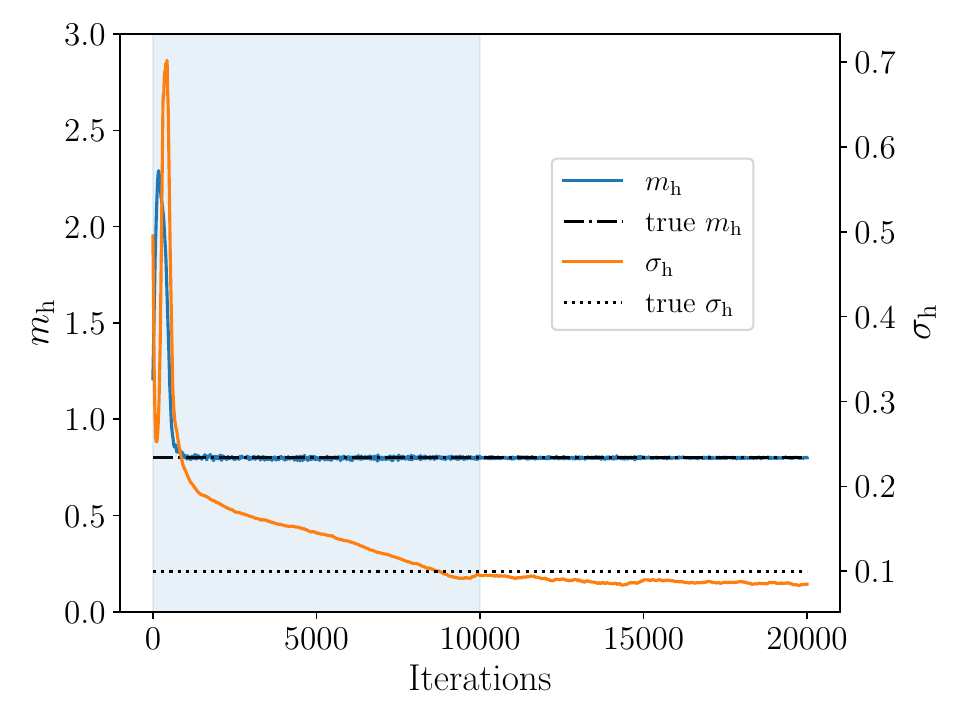}}
    \subfloat[Convergence $\b$]{\includegraphics[width=0.32\linewidth]{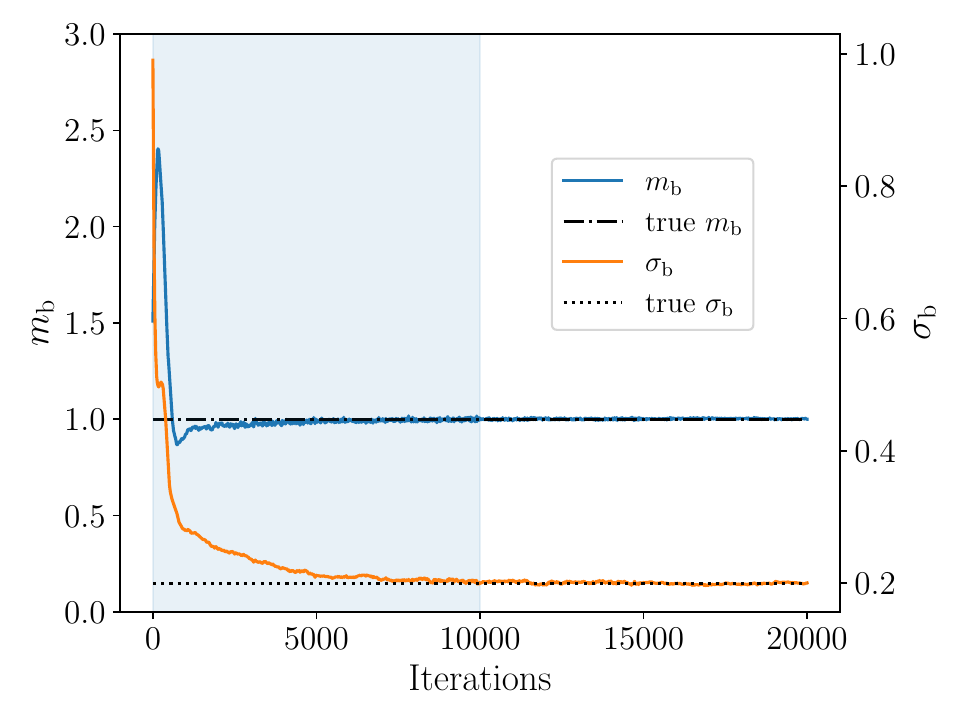}}
    \caption{Convergence of parameters $\alpha$. The blue shaded area denotes the training pairs acquisition after which we stop adding training pairs to $\bP^T_{z,y}$.}
    \label{fig:ConvergenceLorenz96ms}
\end{figure}
\begin{figure}[t]
    \centering
    \subfloat[$\Gamma^\star$]{\includegraphics[width=0.5\linewidth]{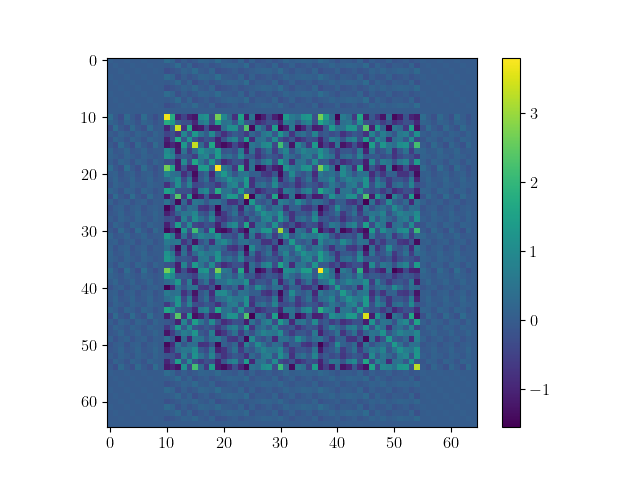}}
    \hspace{-3em}
    \subfloat[ $\mathrm{Cov}_{s_0\sim\bP^\dagger_{s_0}} \mathrm{[}\cG_\tau(m^\dagger; s_0)\mathrm{]}$]{\includegraphics[width=0.5\linewidth]{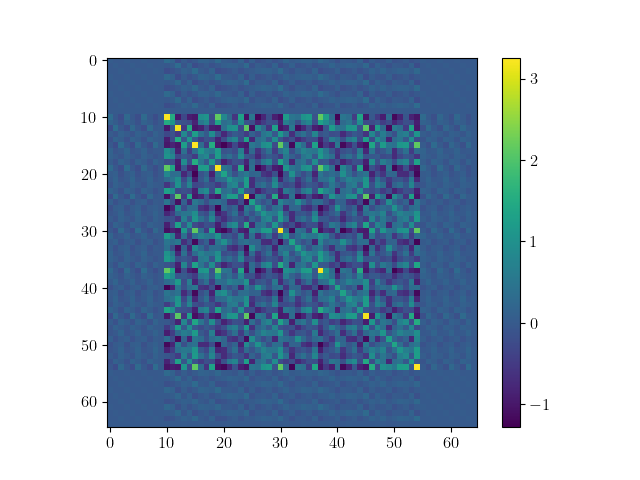}}
    \caption{(a) The learned $\Gamma^\star$. (b) Empirical covariance of the feature function for random initializations of trajectories for $z=m^\dagger=(m_\F, m_\h, m_\b)^\dagger$.}
    \label{fig:CovsLorenz96ms}
\end{figure}
Building on the Lorenz 96 single-scale model, we now turn our attention to the more elaborate Lorenz 96 multi-scale model~\cite{lorenz1996predictability}
\begin{subequations}
\begin{align}
    \dot{u}_k &= -u_{k-1}(u_{k-2} - u_{k+1}) - u_k + \F - \h \bar{v}_k,\\
    \frac{1}{c}\dot{v}_{k,l} &= -  v_{k, l+1}(v_{k, l+2} - v_{k,l-1}) - v_{k,l} + \b u_k,\\
    u_{k+K}&=u_k, \;   v_{l+L}=v_l, \quad k=1 \hdots K,\; l=1\dots L,
\end{align}
\end{subequations}
where $\bar{v}_k=\frac{1}{L}\sum_{l=1}^L v_{k,l}$ and we have periodic boundary conditions on $u$ and $v$ variables~\cite{schneider2017earth, cleary2021calibrate, fatkullin2004computational}. For $c>1$, this model combines chaotic dynamics at two different time scales\footnote{The model dynamics are chaotic for the chosen distribution of data-generating parameters at both scales.}; the slow variables are $u_{1:K}$, each of  which is associated to as set $v_{k,1:L}$ of fast variables; thereby reflecting the multi-scale nature of atmospheric dynamics. 

\noindent\textbf{Data Specifics.}
We specify the observed variables and feature function as
\begin{subequations}
\begin{align}
    w &= \left( u_1, \hdots, u_K, \frac{1}{L}\sum_{l=1}^L \bar{v}_l\right),\\
    \varphi(s) &= \left(\{ w_i\}_{i=1}^{K+1} , \{w_iw_j | i,j\in\mathbb{N}, 1\leq i\leq K+1, 1\leq j\leq i \}\right).
\end{align}
\end{subequations}
We set $K=9$, $L=10$, hence $\varphi:\Rb^{99}\mapsto\Rb^{65}$ ; we fix $c=10$,  $\tau=100$ with 20 units of time burn-in for $\cG_\tau$.
We specify the data generating $z^{(n)}=(\F^{(n)}, \h^{(n)}, \b^{(n)})$ to be sampled from
$\mu(\alpha^\dagger):=\cN(z; m^\dagger, \diag\,(\sigma^\dagger)^2)$ 
where $m^\dagger=(m_\F, m_\h, m_\b)^\dagger=(10, 0.8, 1),$ 
and $\sigma^\dagger=(\sigma_\F, \sigma_\h, \sigma_\b)^\dagger=(1, 0.1, 0.2)$. 
We collect $10^4$ trajectories with the true data generating initial measure $\cN(0, 5^2\I)$. In Figure~\ref{fig:trajectoryLorenz96ms} we show the first dimension of the system, for the full 100 time units (a), for the first 20 time units also showing the first fast variable $v_{0,0}$ (b).

\noindent\textbf{Learning Specifics.} The covariance matrix $\Gamma\in\Rb^{65\times 65}$ is parameterized as a lower triangular matrix with 2145 learnable parameters constrained to be positive on the diagonal.
The true data generating initial measure is $\cN(0, 5^2\I)$, and in the inference scheme we misspecify the initial state measure $\bP_{s_0}$ to be $\cN(0, 8^2\I)$.
The surrogate model $\cG^\phi$ and Covariance regularizer are specified in Subsection~\ref{ssec:LorenzReg}.
The regularizer $h(\alpha)$ $= 1/(2\cdot 5^2)\|(m_\F, m_\h, m_\b) - (8, 2, 2)\|^2$ + $1/(2\cdot 2^2)\|(\sigma_\F, \sigma_\h, \sigma_\b) - \log(0.5, 0.5, 0.5)\|^2$.
We use a learning batch size of 100 for $\cG^{\phi^\star_\st}$. We use a learning rate of $10^{-2}$ on $(\alpha, \Gamma)$ with Adam decayed by half four times over training; and a 10-times decayed-by-half Adam learning rate initialized at $10^{-3}$ for learning $\phi$.
We use a pretraining number of iterations of $\sT_\mathrm{pre.} = 10^3$ with $N_\mathrm{pre.}=100$, and $\sT_\mathrm{inner} = 20$ step on $\phi$ minimizing $\eqref{eq:learn_phi}$ per step on $(\alpha, \Gamma)$. We terminate the training data-pairs acquisition after $\sT_a = 10$k steps.

\noindent\textbf{Solver Specifics.}  We time integrate the system with a fourth order Runge-Kutta scheme with time step of $10^{-3}U_\kt$.

\noindent\textbf{Results.}
Figure~\ref{fig:lossLorenz96ms} shows the data fit loss, surrogate model loss and the condition number of $\Gamma$. Figure~\ref{fig:ConvergenceLorenz96ms} shows the convergence of $\alpha={(m_\F, m_\h, m_\b, \sigma_\F, \sigma_\h, \sigma_\b)}$. In Figure~\ref{fig:CovsLorenz96ms} we show: (a), the learned covariance matrix; (b), the covariance matrix of the observations at the mode of the true parameter distribution for randomized initial conditions from $\bP^\dagger_{s_0}$. In Figure~\ref{fig:data_aquisition_l96ms} we show the active learning collection of points for $z^{(\st)}$ going into $\bP^\st_{y,z}$ as the learning iterations progress, the colour bar indicates the gradient-step number $\st$ at which the sample was taken. We can see from this that the learning efforts are eventually concentrated on regions of high density under $\mu^\dagger$.
The average relative errors on the last 1000 iterations are: 
0.13\% for $\F_m$, 0.46\% for $\F_\sigma$, 
0.03\% for $\h_m$, 14.92\% for $\h_\sigma$, 
0.13\% for $\b_m$, 0.48\% for $\b_\sigma$, 
The relative difference between $\Gamma^\star$ and $\mathrm{Cov}_{s_0\sim\bP^\dagger_{s_0}}[\cG_\tau(m^\dagger;s_0)]$ is of 17.23\% in the Frobenius norm.

\section{Conclusions and Future Work}\label{sec:conclusion}
In this work, we propose a flexible and broadly applicable inference scheme for performing simultaneous deconvolution and populational inversion with data originating from collections of physical systems. We propose a loss function for performing this inference along with an update scheme robust to empiricalization. We also propose a scheme to concurrently learn a surrogate model which concentrates the learning efforts to regions of high probability under the estimated distribution over the parameters. The proposed methodology is then tested on a variety of numerical examples, including the Darcy model for porous medium flow, and damped elastodynamics. Furthermore, we show how the methodology can be adapted to problems in chaotic dynamical systems with time-averaged observation models and learn both the distribution of the parameters of the systems and the covariance of CLT approximations of the error arising from finite-time-averaging; such data acquisition models arise naturally when calibrating GCMs.
There are a number of natural directions for future  methodological work: expanding the types of problems this methodology can be applied to, such as stochastic differential equations; making use of different parametric approximations of the distribution of interest;
exploring inference under stronger model misspecification; performing inference with more general forms of unknown noise distributions,   beyond Gaussian noise, and including multiplicative or nonlinear noise models. Future theoretical questions suggested by our numerics include: proving theoretical guarantees under finite sample size approximation; and providing theory
characterizing parameter identifiability of deconvolution in populational inverse problems.

\section*{Acknowledgements}
The authors would like to thank Dr. Charikleia Stoura for discussions on the topic of structural dynamics, which helped improve this paper.

\bibliographystyle{elsarticle-num} 
\bibliography{biblio}

@article{sacks1989statistical,
  title={Statistical science},
  author={Sacks, Jerome and Welch, WJ and Mitchell, TJ and Wynn, HP},
  journal={Design and Analysis of Computer Experiments},
  volume={4},
  number={4},
  pages={409--423},
  year={1989}
}

@article{szekely2003statistics,
  title={E-statistics: The energy of statistical samples},
  author={Sz{\'e}kely, G{\'a}bor J},
  journal={Bowling Green State University, Department of Mathematics and Statistics Technical Report},
  volume={3},
  number={05},
  pages={1--18},
  year={2003}
}

@book{engl1996regularization,
  title={Regularization of inverse problems},
  author={Engl, Heinz Werner and Hanke, Martin and Neubauer, Andreas},
  volume={375},
  year={1996},
  publisher={Springer Science \& Business Media}
}

@book{kaipio2006statistical,
  title={Statistical and computational inverse problems},
  author={Kaipio, Jari and Somersalo, Erkki},
  volume={160},
  year={2006},
  publisher={Springer Science \& Business Media}
}

@article{akyildiz2024efficient,
  title={Efficient Prior Calibration From Indirect Data},
  author={Akyildiz, O Deniz and Girolami, Mark and Stuart, Andrew M and Vadeboncoeur, Arnaud},
  journal={SIAM SISC, to appear; arXiv preprint arXiv:2405.17955},
  year={2024}
}

@article{gassiat2022deconvolution,
  title={Deconvolution with unknown noise distribution is possible for multivariate signals},
  author={Gassiat, Elisabeth and Le Corff, Sylvain and Leh{\'e}ricy, Luc},
  journal={The Annals of Statistics},
  volume={50},
  number={1},
  pages={303--323},
  year={2022},
  publisher={Institute of Mathematical Statistics}
}

@article{halva2021disentangling,
  title={Disentangling identifiable features from noisy data with structured nonlinear ICA},
  author={H{\"a}lv{\"a}, Hermanni and Le Corff, Sylvain and Leh{\'e}ricy, Luc and So, Jonathan and Zhu, Yongjie and Gassiat, Elisabeth and Hyvarinen, Aapo},
  journal={Advances in Neural Information Processing Systems},
  volume={34},
  pages={1624--1633},
  year={2021}
}

@article{won2013condition,
  title={Condition-number-regularized covariance estimation},
  author={Won, Joong-Ho and Lim, Johan and Kim, Seung-Jean and Rajaratnam, Bala},
  journal={Journal of the Royal Statistical Society Series B: Statistical Methodology},
  volume={75},
  number={3},
  pages={427--450},
  year={2013},
  publisher={Oxford University Press}
}

@article{dey1985estimation,
  title={Estimation of a covariance matrix under Stein's loss},
  author={Dey, Dipak K and Srinivasan, C},
  journal={The Annals of Statistics},
  pages={1581--1591},
  year={1985},
  publisher={JSTOR}
}

@incollection{james1992estimation,
  title={Estimation with quadratic loss},
  author={James, William and Stein, Charles},
  booktitle={Breakthroughs in statistics: Foundations and basic theory},
  pages={443--460},
  year={1992},
  publisher={Springer}
}

@article{schneider2017earth,
  title={Earth system modeling 2.0: A blueprint for models that learn from observations and targeted high-resolution simulations},
  author={Schneider, Tapio and Lan, Shiwei and Stuart, Andrew and Teixeira, Joao},
  journal={Geophysical Research Letters},
  volume={44},
  number={24},
  pages={12--396},
  year={2017},
  publisher={Wiley Online Library}
}

@article{cleary2021calibrate,
  title={Calibrate, emulate, sample},
  author={Cleary, Emmet and Garbuno-Inigo, Alfredo and Lan, Shiwei and Schneider, Tapio and Stuart, Andrew M},
  journal={Journal of Computational Physics},
  volume={424},
  pages={109716},
  year={2021},
  publisher={Elsevier}
}

@inproceedings{lorenz1996predictability,
  title={Predictability: A problem partly solved},
  author={Lorenz, Edward N},
  booktitle={Proc. Seminar on predictability},
  volume={1},
  number={1},
  year={1996},
  organization={Reading}
}

@software{bleyer2024comet,
  author       = {Bleyer, Jeremy},
  title        = {{Numerical tours of Computational Mechanics with 
                   FEniCSx}},
  month        = jan,
  year         = 2024,
  publisher    = {Zenodo},
  version      = {v0.1},
  doit         = {10.5281/zenodo.10470942},
}

@article{baratta2023dolfinx,
  title={DOLFINx: the next generation FEniCS problem solving environment},
  author={Baratta, Igor A and Dean, Joseph P and Dokken, J{\o}rgen S and Habera, Michal and HALE, Jack and Richardson, Chris N and Rognes, Marie E and Scroggs, Matthew W and Sime, Nathan and Wells, Garth N},
  year={2023}
}

@article{alnaes2014unified,
  title={Unified form language: A domain-specific language for weak formulations of partial differential equations},
  author={Aln{\ae}s, Martin S and Logg, Anders and {\O}lgaard, Kristian B and Rognes, Marie E and Wells, Garth N},
  journal={ACM Transactions on Mathematical Software (TOMS)},
  volume={40},
  number={2},
  pages={1--37},
  year={2014},
  publisher={ACM New York, NY, USA}
}

@book{langtangen2017solving,
  title={Solving PDEs in python: the FEniCS tutorial I},
  author={Langtangen, Hans Petter and Logg, Anders},
  year={2017},
  publisher={Springer Nature}
}

@book{salenccon2012handbook,
  title={Handbook of continuum mechanics: General concepts thermoelasticity},
  author={Salen{\c{c}}on, Jean},
  year={2012},
  publisher={Springer Science \& Business Media}
}

@book{slaughter2012linearized,
  title={The linearized theory of elasticity},
  author={Slaughter, William S},
  year={2012},
  publisher={Springer Science \& Business Media}
}

@article{fatkullin2004computational,
  title={A computational strategy for multiscale systems with applications to {Lorenz} 96 model},
  author={Fatkullin, Ibrahim and Vanden-Eijnden, Eric},
  journal={Journal of Computational Physics},
  volume={200},
  number={2},
  pages={605--638},
  year={2004},
  publisher={Elsevier}
}

@article{gouk2021regularisation,
  title={Regularisation of neural networks by enforcing {Lipschitz} continuity},
  author={Gouk, Henry and Frank, Eibe and Pfahringer, Bernhard and Cree, Michael J},
  journal={Machine Learning},
  volume={110},
  pages={393--416},
  year={2021},
  publisher={Springer}
}

@article{zhang2021novel,
  title={A novel Rayleigh-type viscoelastic perfectly-matched-layer for wave propagation analysis: formulation, implementation and application},
  author={Zhang, W and Taciroglu, E},
  journal={Computer Methods in Applied Mechanics and Engineering},
  volume={383},
  pages={113913},
  year={2021},
  publisher={Elsevier}
}

@article{stuart2010inverse,
  title={Inverse problems: a {{Bayesian}} perspective},
  author={Stuart, Andrew M},
  journal={Acta numerica},
  volume={19},
  pages={451--559},
  year={2010},
  publisher={Cambridge University Press}
}

@book{aster2018parameter,
  title={Parameter estimation and inverse problems},
  author={Aster, Richard C and Borchers, Brian and Thurber, Clifford H},
  year={2018},
  publisher={Elsevier}
}

@book{groetsch1993inverse,
  title={Inverse problems in the mathematical sciences},
  author={Groetsch, Charles W and Groetsch, CW},
  volume={52},
  year={1993},
  publisher={Springer}
}

@book{darcy1856fontaines,
  title={Les fontaines publiques de la ville de Dijon: exposition et application des principes {\`a} suivre et des formules {\`a} employer dans les questions de distribution d'eau},
  author={Darcy, Henry},
  volume={1},
  year={1856},
  publisher={Victor dalmont}
}

@article{bui2008model,
  title={Model reduction for large-scale systems with high-dimensional parametric input space},
  author={Bui-Thanh, Tan and Willcox, Karen and Ghattas, Omar},
  journal={SIAM Journal on Scientific Computing},
  volume={30},
  number={6},
  pages={3270--3288},
  year={2008},
  publisher={SIAM}
}

@article{bonneel2015sliced,
  title={Sliced and radon wasserstein barycenters of measures},
  author={Bonneel, Nicolas and Rabin, Julien and Peyr{\'e}, Gabriel and Pfister, Hanspeter},
  journal={Journal of Mathematical Imaging and Vision},
  volume={51},
  pages={22--45},
  year={2015},
  publisher={Springer}
}

@article{gretton2012kernel,
  title={A kernel two-sample test},
  author={Gretton, Arthur and Borgwardt, Karsten M and Rasch, Malte J and Sch{\"o}lkopf, Bernhard and Smola, Alexander},
  journal={The Journal of Machine Learning Research},
  volume={13},
  number={1},
  pages={723--773},
  year={2012},
  publisher={JMLR. org}
}

@article{szekely2013energy,
  title={Energy statistics: A class of statistics based on distances},
  author={Sz{\'e}kely, G{\'a}bor J and Rizzo, Maria L},
  journal={Journal of statistical planning and inference},
  volume={143},
  number={8},
  pages={1249--1272},
  year={2013},
  publisher={Elsevier}
}

@article{rousseau2024wasserstein,
  title={Wasserstein convergence in {{Bayesian}} and frequentist deconvolution models},
  author={Rousseau, Judith and Scricciolo, Catia},
  journal={The Annals of Statistics},
  volume={52},
  number={4},
  pages={1691--1715},
  year={2024},
  publisher={Institute of Mathematical Statistics}
}

@article{capitao2024deconvolution,
  title={Deconvolution of repeated measurements corrupted by unknown noise},
  author={Capitao-Miniconi, J{\'e}r{\'e}mie and Gassiat, {\'E}lisabeth and Leh{\'e}ricy, Luc},
  journal={arXiv preprint arXiv:2409.02014},
  year={2024}
}

@article{vadeboncoeur2023fully,
  title={Fully probabilistic deep models for forward and inverse problems in parametric PDEs},
  author={Vadeboncoeur, Arnaud and Akyildiz, {\"O}mer Deniz and Kazlauskaite, Ieva and Girolami, Mark and Cirak, Fehmi},
  journal={Journal of Computational Physics},
  volume={491},
  pages={112369},
  year={2023},
  publisher={Elsevier}
}

@article{rixner2021probabilistic,
  title={A probabilistic generative model for semi-supervised training of coarse-grained surrogates and enforcing physical constraints through virtual observables},
  author={Rixner, Maximilian and Koutsourelakis, Phaedon-Stelios},
  journal={Journal of Computational Physics},
  volume={434},
  pages={110218},
  year={2021},
  publisher={Elsevier}
}

@article{zhu2019physics,
  title={Physics-constrained deep learning for high-dimensional surrogate modeling and uncertainty quantification without labeled data},
  author={Zhu, Yinhao and Zabaras, Nicholas and Koutsourelakis, Phaedon-Stelios and Perdikaris, Paris},
  journal={Journal of Computational Physics},
  volume={394},
  pages={56--81},
  year={2019},
  publisher={Elsevier}
}

@article{zhang2024bilo,
  title={Bilo: Bilevel local operator learning for pde inverse problems},
  author={Zhang, Ray Zirui and Xie, Xiaohui and Lowengrub, John S},
  journal={arXiv preprint arXiv:2404.17789},
  year={2024}
}

@article{alexander2009deconvolution,
  title={Deconvolution problems in nonparametric statistics},
  author={Alexander, Meister},
  journal={Lecture Notes in Statistics, Springer, Berlin, Heidelberg},
  pages={5--138},
  year={2009}
}

@article{meister2007deconvolving,
  title={Deconvolving compactly supported densities},
  author={Meister, Alexander},
  journal={Mathematical Methods of Statistics},
  volume={16},
  pages={63--76},
  year={2007},
  publisher={Springer}
}

@inproceedings{moulines1997maximum,
  title={Maximum likelihood for blind separation and deconvolution of noisy signals using mixture models},
  author={Moulines, Eric and Cardoso, J-F and Gassiat, Elisabeth},
  booktitle={1997 ieee international conference on acoustics, speech, and signal processing},
  volume={5},
  pages={3617--3620},
  year={1997},
  organization={IEEE}
}

@article{attias1998blind,
  title={Blind source separation and deconvolution: the dynamic component analysis algorithm},
  author={Attias, Hagai and Schreiner, Christoph E.},
  journal={Neural computation},
  volume={10},
  number={6},
  pages={1373--1424},
  year={1998},
  publisher={MIT Press}
}

@article{capitao2023deconvolution,
  title={Deconvolution of spherical data corrupted with unknown noise},
  author={Capitao-Miniconi, J{\'e}r{\'e}mie and Gassiat, {\'E}lisabeth},
  journal={Electronic Journal of Statistics},
  volume={17},
  number={1},
  pages={607--649},
  year={2023},
  publisher={The Institute of Mathematical Statistics and the Bernoulli Society}
}

@article{oja2000independent,
  title={Independent component analysis: algorithms and applications},
  author={Oja, Erkki and Hyvarinen, A},
  journal={Neural networks},
  volume={13},
  number={4-5},
  pages={411--430},
  year={2000}
}

@article{hyvarinen1997fast,
  title={A fast fixed-point algorithm for independent component analysis},
  author={Hyv{\"a}rinen, Aapo and Oja, Erkki},
  journal={Neural computation},
  volume={9},
  number={7},
  pages={1483--1492},
  year={1997},
  publisher={MIT Press One Rogers Street, Cambridge, MA 02142-1209, USA journals-info}
}

@article{hyvarinen2002independent,
  title={Independent component analysis},
  author={Hyvarinen, Aapo and Karhunen, Juha and Oja, Erkki},
  journal={Studies in informatics and control},
  volume={11},
  number={2},
  pages={205--207},
  year={2002},
  publisher={INFORMATICS AND CONTROL PUBLICATIONS}
}

@article{lee1999independent,
  title={Independent component analysis using an extended infomax algorithm for mixed subgaussian and supergaussian sources},
  author={Lee, Te-Won and Girolami, Mark and Sejnowski, Terrence J},
  journal={Neural computation},
  volume={11},
  number={2},
  pages={417--441},
  year={1999},
  publisher={MIT Press}
}

@article{10.1214/08-AOS652,
author = {Jan Johannes},
title = {{Deconvolution with unknown error distribution}},
volume = {37},
journal = {The Annals of Statistics},
number = {5A},
publisher = {Institute of Mathematical Statistics},
pages = {2301 -- 2323},
year = {2009},
doit = {10.1214/08-AOS652},
}

@article{comte2010pointwise,
  title={Pointwise deconvolution with unknown error distribution},
  author={Comte, Fabienne and Lacour, Claire},
  journal={Comptes Rendus. Math{\'e}matique},
  volume={348},
  number={5-6},
  pages={323--326},
  year={2010}
}

@article{li2024stochastic,
  title={Stochastic Inverse Problem: stability, regularization and {Wasserstein} gradient flow},
  author={Li, Qin and Oprea, Maria and Wang, Li and Yang, Yunan},
  journal={arXiv preprint arXiv:2410.00229},
  year={2024}
}

@article{bernton2019parameter,
  title={On parameter estimation with the Wasserstein distance},
  author={Bernton, Espen and Jacob, Pierre E and Gerber, Mathieu and Robert, Christian P},
  journal={Information and Inference: A Journal of the IMA},
  volume={8},
  number={4},
  pages={657--676},
  year={2019},
  publisher={Oxford University Press}
}

@article{bassetti2006minimum,
  title={On minimum Kantorovich distance estimators},
  author={Bassetti, Federico and Bodini, Antonella and Regazzini, Eugenio},
  journal={Statistics \& probability letters},
  volume={76},
  number={12},
  pages={1298--1302},
  year={2006},
  publisher={Elsevier}
}

@article{belili1999estimation,
  title={Estimation bas{\'e}e sur la fonctionnelle de Kantorovich et la distance de L{\'e}vy},
  author={Belili, Nacereddine and Bensa{\"\i}, Amin{\'e} and Heinich, Henri},
  journal={Comptes Rendus de l'Acad{\'e}mie des Sciences-Series I-Mathematics},
  volume={328},
  number={5},
  pages={423--426},
  year={1999},
  publisher={Elsevier}
}

@article{bull2021foundations,
  title={Foundations of population-based SHM, Part I: Homogeneous populations and forms},
  author={Bull, Lawrence A and Gardner, Paul A and Gosliga, Julian and Rogers, Timothy J and Dervilis, Nikolaos and Cross, Elizabeth J and Papatheou, Evangelos and Maguire, AE and Campos, Carles and Worden, Keith},
  journal={Mechanical systems and signal processing},
  volume={148},
  pages={107141},
  year={2021},
  publisher={Elsevier}
}

@book{montgomery2020introduction,
  title={Introduction to statistical quality control},
  author={Montgomery, Douglas C},
  year={2020},
  publisher={John wiley \& sons}
}

@article{butler2018combining,
  title={Combining push-forward measures and Bayes' rule to construct consistent solutions to stochastic inverse problems},
  author={Butler, T and Jakeman, J and Wildey, Tim},
  journal={SIAM Journal on Scientific Computing},
  volume={40},
  number={2},
  pages={A984--A1011},
  year={2018},
  publisher={SIAM}
}

@book{bernardo2009bayesian,
  title={{{Bayesian}} theory},
  author={Bernardo, Jos{\'e} M and Smith, Adrian FM},
  volume={405},
  year={2009},
  publisher={John Wiley \& Sons}
}

@book{gelman1995bayesian,
  title={{{Bayesian}} data analysis},
  author={Gelman, Andrew and Carlin, John B and Stern, Hal S and Rubin, Donald B},
  year={1995},
  publisher={Chapman and Hall/CRC}
}

@article{bingham2024inverse,
  title={Inverse problems for physics-based process models},
  author={Bingham, Derek and Butler, Troy and Estep, Don},
  journal={Annual Review of Statistics and Its Application},
  volume={11},
  year={2024},
  publisher={Annual Reviews}
}

@article{butler2020optimal,
  title={Optimal experimental design for prediction based on push-forward probability measures},
  author={Butler, T and Jakeman, John D and Wildey, Tim},
  journal={Journal of Computational Physics},
  volume={416},
  pages={109518},
  year={2020},
  publisher={Elsevier}
}

@article{bhattacharya2021model,
  title={Model reduction and neural networks for parametric PDEs},
  author={Bhattacharya, Kaushik and Hosseini, Bamdad and Kovachki, Nikola B and Stuart, Andrew M},
  journal={The SMAI journal of computational mathematics},
  volume={7},
  pages={121--157},
  year={2021}
}

@article{choi2021space,
  title={Space--time reduced order model for large-scale linear dynamical systems with application to Boltzmann transport problems},
  author={Choi, Youngsoo and Brown, Peter and Arrighi, William and Anderson, Robert and Huynh, Kevin},
  journal={Journal of Computational Physics},
  volume={424},
  pages={109845},
  year={2021},
  publisher={Elsevier}
}

@article{sampaio2007remarks,
  title={Remarks on the efficiency of POD for model reduction in non-linear dynamics of continuous elastic systems},
  author={Sampaio, Rubens and Soize, Christian},
  journal={International Journal for numerical methods in Engineering},
  volume={72},
  number={1},
  pages={22--45},
  year={2007},
  publisher={Wiley Online Library}
}

@article{bai2002krylov,
  title={Krylov subspace techniques for reduced-order modeling of large-scale dynamical systems},
  author={Bai, Zhaojun},
  journal={Applied numerical mathematics},
  volume={43},
  number={1-2},
  pages={9--44},
  year={2002},
  publisher={Elsevier}
}

@article{freund2000krylov,
  title={Krylov-subspace methods for reduced-order modeling in circuit simulation},
  author={Freund, Roland W},
  journal={Journal of Computational and Applied Mathematics},
  volume={123},
  number={1-2},
  pages={395--421},
  year={2000},
  publisher={Elsevier}
}

@article{fresca2022pod,
  title={POD-DL-ROM: Enhancing deep learning-based reduced order models for nonlinear parametrized PDEs by proper orthogonal decomposition},
  author={Fresca, Stefania and Manzoni, Andrea},
  journal={Computer Methods in Applied Mechanics and Engineering},
  volume={388},
  pages={114181},
  year={2022},
  publisher={Elsevier}
}

@article{chaturantabut2010nonlinear,
  title={Nonlinear model reduction via discrete empirical interpolation},
  author={Chaturantabut, Saifon and Sorensen, Danny C},
  journal={SIAM Journal on Scientific Computing},
  volume={32},
  number={5},
  pages={2737--2764},
  year={2010},
  publisher={SIAM}
}

@article{haasdonk2008reduced,
  title={Reduced basis method for finite volume approximationsof parametrizedlinear evolution equations},
  author={Haasdonk, Bernard and Ohlberger, Mario},
  journal={ESAIM: Mathematical Modelling and Numerical Analysis},
  volume={42},
  number={2},
  pages={277--302},
  year={2008},
  publisher={EDP Sciences}
}

@article{almroth1978automatic,
  title={Automatic choice of global shape functions in structural analysis},
  author={Almroth, Bo O and Stern, Perry and Brogan, Frank A},
  journal={Aiaa Journal},
  volume={16},
  number={5},
  pages={525--528},
  year={1978}
}

@article{noor1980reduced,
  title={Reduced basis technique for nonlinear analysis of structures},
  author={Noor, Ahmed K and Peters, Jeanne M},
  journal={Aiaa journal},
  volume={18},
  number={4},
  pages={455--462},
  year={1980}
}

@article{lu2021learning,
  title={Learning nonlinear operators via DeepONet based on the universal approximation theorem of operators},
  author={Lu, Lu and Jin, Pengzhan and Pang, Guofei and Zhang, Zhongqiang and Karniadakis, George Em},
  journal={Nature machine intelligence},
  volume={3},
  number={3},
  pages={218--229},
  year={2021},
  publisher={Nature Publishing Group UK London}
}

@article{kovachki2023neural,
  title={Neural operator: Learning maps between function spaces with applications to pdes},
  author={Kovachki, Nikola and Li, Zongyi and Liu, Burigede and Azizzadenesheli, Kamyar and Bhattacharya, Kaushik and Stuart, Andrew and Anandkumar, Anima},
  journal={Journal of Machine Learning Research},
  volume={24},
  number={89},
  pages={1--97},
  year={2023}
}

@article{li2020fourier,
  title={Fourier neural operator for parametric partial differential equations},
  author={Li, Zongyi and Kovachki, Nikola and Azizzadenesheli, Kamyar and Liu, Burigede and Bhattacharya, Kaushik and Stuart, Andrew and Anandkumar, Anima},
  journal={arXiv preprint arXiv:2010.08895},
  year={2020}
}

@book{villani2008optimal,
  title={Optimal transport: old and new},
  author={Villani, C{\'e}dric and others},
  volume={338},
  year={2008},
  publisher={Springer}
}

@book{santambrogio2015optimal,
  title={Optimal transport for applied mathematicians},
  author={Santambrogio, Filippo},
  volume={87},
  year={2015},
  publisher={Springer}
}

@article{kennedy2001bayesian,
  title={{{Bayesian}} calibration of computer models},
  author={Kennedy, Marc C and O'Hagan, Anthony},
  journal={Journal of the Royal Statistical Society: Series B (Statistical Methodology)},
  volume={63},
  number={3},
  pages={425--464},
  year={2001},
  publisher={Wiley Online Library}
}

@article{kennedy2000predicting,
  title={Predicting the output from a complex computer code when fast approximations are available},
  author={Kennedy, Marc C and O'Hagan, Anthony},
  journal={Biometrika},
  volume={87},
  number={1},
  pages={1--13},
  year={2000},
  publisher={Oxford University Press}
}

@article{krige1951statistical,
  title={A statistical approach to some basic mine valuation problems on the Witwatersrand},
  author={Krige, Daniel G},
  journal={Journal of the Southern African Institute of Mining and Metallurgy},
  volume={52},
  number={6},
  pages={119--139},
  year={1951},
  publisher={Southern African Institute of Mining and Metallurgy}
}

@article{yan2021stein,
  title={Stein variational gradient descent with local approximations},
  author={Yan, Liang and Zhou, Tao},
  journal={Computer Methods in Applied Mechanics and Engineering},
  volume={386},
  pages={114087},
  year={2021},
  publisher={Elsevier}
}

@article{dutta2021greedy,
  title={A greedy non-intrusive reduced order model for shallow water equations},
  author={Dutta, Sourav and Farthing, Matthew W and Perracchione, Emma and Savant, Gaurav and Putti, Mario},
  journal={Journal of Computational Physics},
  volume={439},
  pages={110378},
  year={2021},
  publisher={Elsevier}
}

@article{li2023surrogate,
  title={Surrogate modeling for {{Bayesian}} inverse problems based on physics-informed neural networks},
  author={Li, Yongchao and Wang, Yanyan and Yan, Liang},
  journal={Journal of Computational Physics},
  volume={475},
  pages={111841},
  year={2023},
  publisher={Elsevier}
}

@article{dunlop2017hierarchical,
  title={Hierarchical {{Bayesian}} level set inversion},
  author={Dunlop, Matthew M and Iglesias, Marco A and Stuart, Andrew M},
  journal={Statistics and Computing},
  volume={27},
  pages={1555--1584},
  year={2017},
  publisher={Springer}
}

@article{LassiRoininen2014InverseProblemsandImaging,
title = {Whittle-Matérn priors for {{Bayesian}} statistical inversion with applications in electrical impedance tomography},
journal = {Inverse Problems and Imaging},
volume = {8},
number = {2},
pages = {561-586},
year = {2014},
issn = {1930-8337},
doit = {10.3934/ipi.2014.8.561},
author = {Lassi Roininen and Janne M. J. Huttunen and Sari Lasanen}
}

@article{solin2020hilbert,
  title={Hilbert space methods for reduced-rank Gaussian process regression},
  author={Solin, Arno and S{\"a}rkk{\"a}, Simo},
  journal={Statistics and Computing},
  volume={30},
  number={2},
  pages={419--446},
  year={2020},
  publisher={Springer}
}

@article{gottlieb1985eigenvalues,
  title={Eigenvalues of the Laplacian with Neumann boundary conditions},
  author={Gottlieb, HPW},
  journal={The ANZIAM Journal},
  volume={26},
  number={3},
  pages={293--309},
  year={1985},
  publisher={Cambridge University Press}
}

@article{newmark1959method,
  title={A method of computation for structural dynamics},
  author={Newmark, Nathan M},
  journal={Journal of the engineering mechanics division},
  volume={85},
  number={3},
  pages={67--94},
  year={1959},
  publisher={American Society of Civil Engineers}
}

@article{kingma2014adam,
  title={Adam: A method for stochastic optimization},
  author={Kingma, Diederik P and Ba, Jimmy},
  journal={arXiv preprint arXiv:1412.6980},
  year={2014}
}

@article{kroese1995distributional,
  title={Distributional inference},
  author={Kroese, Albert Hendrik and Van der Meulen, EA and Poortema, Klaas and Schaafsma, W},
  journal={Statistica Neerlandica},
  volume={49},
  number={1},
  pages={63--82},
  year={1995},
  publisher={Wiley Online Library}
}

@article{robbins1964empirical,
  title={The empirical Bayes approach to statistical decision problems},
  author={Robbins, Herbert},
  journal={The Annals of Mathematical Statistics},
  volume={35},
  number={1},
  pages={1--20},
  year={1964},
  publisher={JSTOR}
}

@incollection{robbins1992empirical,
  title={An empirical Bayes approach to statistics},
  author={Robbins, Herbert E},
  booktitle={Breakthroughs in Statistics: Foundations and basic theory},
  pages={388--394},
  year={1992},
  publisher={Springer}
}

@book{abramowitz1965handbook,
  title={Handbook of mathematical functions: with formulas, graphs, and mathematical tables},
  author={Abramowitz, Milton and Stegun, Irene A},
  volume={55},
  year={1965},
  publisher={Courier Corporation}
}

@article{li2024differential,
  title={Differential Equation--Constrained Optimization with Stochasticity},
  author={Li, Qin and Wang, Li and Yang, Yunan},
  journal={SIAM/ASA Journal on Uncertainty Quantification},
  volume={12},
  number={2},
  pages={549--578},
  year={2024},
  publisher={SIAM}
}

@article{li2025inverse,
  title={Inverse Problems Over Probability Measure Space},
  author={Li, Qin and Oprea, Maria and Wang, Li and Yang, Yunan},
  journal={arXiv preprint arXiv:2504.18999},
  year={2025}
}

\end{document}